\documentclass{article}

\usepackage{preferences} 

\usepackage{CustomCommands}
\Crefname{appendix}{Appendix}{Appendices}

\usepackage[
    backend=biber,
    style=ieee,
    citestyle=numeric-comp,
    giveninits=true,
    url=false,
    eprint=false
]{biblatex}

\DeclareFieldFormat{sentencecase}{#1} 

\newcommand{\papertitle}{Learning Adaptive SED for heterogeneous load balancing}

\title{\papertitle}

\author{Sanne van Kempen, Jaron Sanders, Fiona Sloothaak, Maarten G. Wolf}

\hypersetup{
    pdfauthor={Sanne van Kempen, Jaron Sanders, Fiona Sloothaak, Maarten G.\ Wolf},
    pdftitle={\papertitle},
    pdfsubject={\papertitle},
    pdfkeywords={queueing, load balancing, machine learning},
}

\date{}

\begin{document}
\maketitle

\addcontentsline{toc}{section}{Abstract}
\begin{abstract}
	We study a two-server load balancing system with heterogeneous service rates that are \emph{a priori} unknown to the dispatcher. 
The goal is to route customers according to the Shortest--Expected--Delay (SED) policy, but this requires knowledge of the service rates. 
Empirical policies that route based on estimates perform poorly: due to estimation error, the empirical policy disagrees with the oracle on an infinite region of the state space. 
We propose an online learning algorithm that converges to SED while learning the service rates. 
The algorithm carefully balances empirical SED routing with forced exploration phases that guarantee sufficient sampling of both servers.
We prove that our algorithm achieves finite regret; this differs from classical Multi-Armed Bandit settings where regret typically grows logarithmically in time. 
Finally, numerical experiments demonstrate the performance of our algorithm and highlight the regimes in which forced exploration is especially beneficial.

\end{abstract}

\section{Introduction}%
\label{sec:introduction}

Real-world service systems such as data centers, cloud computing networks, and customer contact centers operate with heterogeneous servers.
A critical question in such a system is how to route arriving customers to servers in order to minimize customer delays.
To this end, load balancing policies have been studied in the literature~\citep{HarcholBalter2013,vanderBoor2017,Cardinaels2024}.
In practice, however, service rates may (initially) be unknown to the dispatcher, making it hard to route perfectly.
For example, service may be outsourced to third-party providers who do not share their exact service capacities~\citep{Armbrust2010}, or servers have auxiliary tasks that are subject to privacy constraints or communication delays~\citep{Choudhury2021}.
These factors lead to uncertainty in the available service rates.

This paper studies the load balancing system with two heterogeneous servers shown in \Cref{fig:loadbalancing_2_queues}.
Each server has an infinite buffer.
Arriving customers must be routed immediately upon arrival to one of the servers, but the service rates $\mu_1 > 0$ and $\mu_2 > 0$ are \emph{a priori} unknown to the dispatcher.
The objective is to minimize the delay by routing customers as efficiently as possible, despite uncertainty about the service rates.

\paragraph{The SED policy.}
Even if the service rates are known, finding a policy that minimizes the average delay in heterogeneous systems is still an open problem~\citep{Gardner2021,Bhambay2022,Hyyti2017}.
This has motivated the use of heuristic policies that are easy to implement and perform well empirically.
One such policy is \gls{SED}~\citep{Sassen1997}.
This policy routes each customer to the server with the smallest expected delay (the time between arrival and service completion).

In Markovian service systems, \gls{SED} routes customers to the server that minimizes the ratio $(q_i+1)/\mu_i$, where $q_i \geq 0$ is the number of customers at server $i$ at the time of a customer arrival.
If two ratios are equal, both servers are optimal choices.
Such ties may be resolved using a fixed tie-breaking rule or with randomized routing.
Observe that implementing \gls{SED} requires knowledge of the service rates. 

In the two-server system that we study, the decision rule is specified by the ratio $r^* = \mu_2/\mu_1$.
Indeed, we route to server one if $r^* (q_1+1)$ is less than $q_2+1$ and server two otherwise.
The line $q_2 = r^*(q_1+1)-1$ partitions the state space into two parts as shown in \Cref{fig:grid_SED}.
In states above the line, \gls{SED} routes customers to server one, and otherwise to server two.

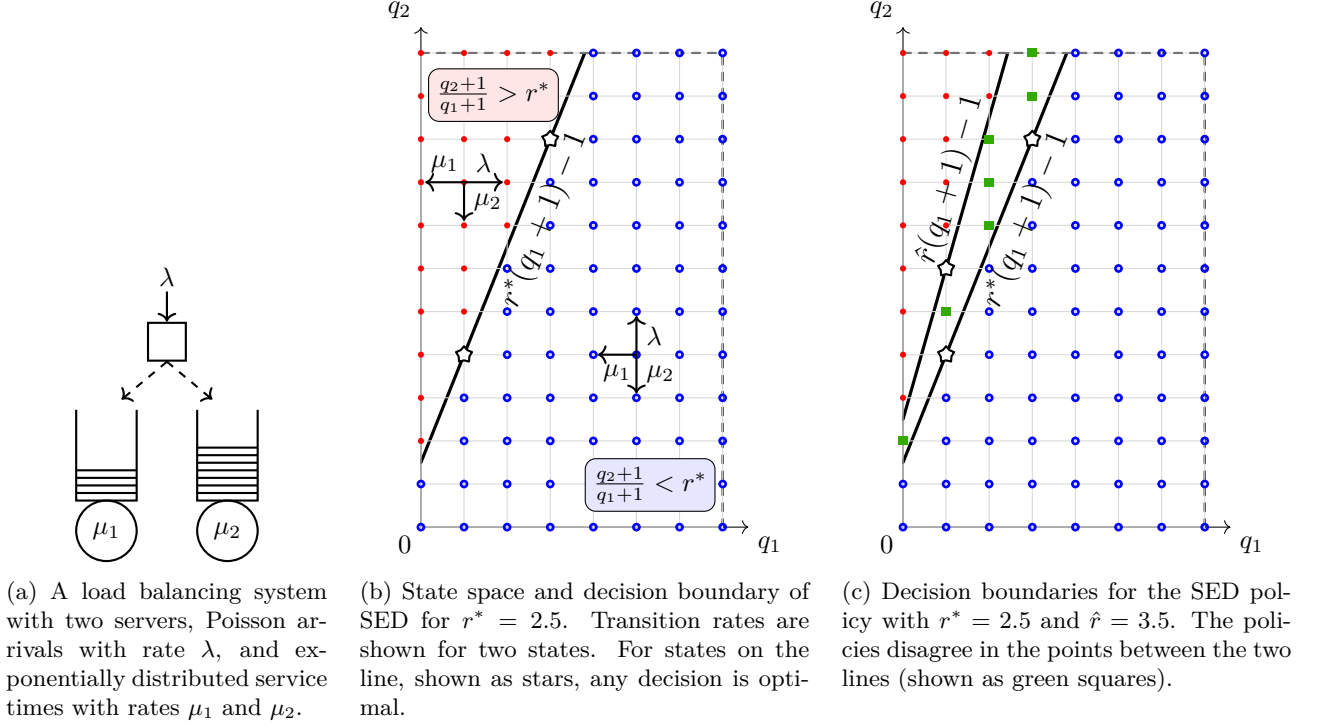
\begin{figure}[h]
	\centering
	\begin{subfigure}[t]{.25\linewidth}
		\centering
		\begin{tikzpicture}[
				scale = 1,
				transform shape,
				dispatcher/.style={rectangle, minimum size = .5cm, thick,draw},
				server/.style={circle, minimum size = .8cm, thick,draw}
			]
			\node[dispatcher] at (0,0) (dis) {};
			\draw[thick,<-,] (dis.north) --++ (0,.4) node[pos=1, yshift=2mm] {$\lambda$};
			\foreach \x/\y in {1/-.8,2/.8}{ 
					\node[server, below of = dis, yshift = -1.5cm, xshift=\y cm] (server\x) {};
					\node at (server\x) {$\mu_{\x}$};
					\node[draw = none, above of = server\x, yshift=-.6cm] (queue\x) {};
					\draw[thick] (queue\x.center) --++(-.4,0) --++(0,1.2); 
					\draw[thick] (queue\x.center) --++(.4,0) --++(0,1.2);
				}
			\foreach \x in {.1,.2,.3,.4}{ 
					\draw[thick] (queue1.center) ++(-.4,\x) --++(.8,0);
				}
			\foreach \x in {.1,.2,.3,.4,.5,.6,.7}{ 
					\draw[thick] (queue2.center) ++(-.4,\x) --++(.8,0);
				}

			\draw[thick,->,dashed] (dis.south) --++ (.6,-.5);
			\draw[thick,->,dashed] (dis.south) --++ (-.6,-.5);

		\end{tikzpicture}
		\caption{A load balancing system with two servers, Poisson arrivals with rate $\lambda$, and exponentially distributed service times with rates $\mu_1$ and $\mu_2$. }%
		\label{fig:loadbalancing_2_queues}
	\end{subfigure}
	\hfill
	\begin{subfigure}[t]{.35\linewidth}
		\centering
		\begin{tikzpicture}[scale=0.57, line cap=round, line join=round]
			\def\mone{7}
			\def\mtwo{11}
			\def\r{2.5} 

			\draw[->] (0,0) node[below left] {0} -- (\mone+0.6,0) node[below right] {$q_1$};
			\draw[->] (0,0) -- (0,\mtwo+0.6) node[above left] {$q_2$};

			\draw[dashed,thick] (\mone,0) -- (\mone,\mtwo);
			\draw[dashed,thick] (0,\mtwo) -- (\mone,\mtwo);

			\begin{scope}
				\clip (0,0) rectangle (\mone,\mtwo);
				\draw[very thick,black]
				plot[domain=0:\mone] (\x, {\r*(\x+1)-1});
			\end{scope}

			\draw[step=1,gray!30] (0,0) grid (\mone,\mtwo);

			\foreach \x in {0,...,\mone}{
					\foreach \y in {0,...,\mtwo}{
							\pgfmathsetmacro{\yline}{\r*(\x+1)-1}
							\pgfmathsetmacro{\eps}{1e-2}
							\pgfmathsetmacro{\diff}{\y - \yline}
							\ifdim \diff pt < -\eps pt
								\draw[blue, line width=1pt] (\x,\y) circle (2pt);
							\else\ifdim \diff pt > \eps pt
									\fill[red] (\x,\y) circle (2pt);
								\else
									\node[
										star,
										star points=5,
										draw=black,
										line width=0.8pt,
										fill=white,
										inner sep=1.5pt
									] at (\x,\y) {};
								\fi\fi
						}
				}

			\coordinate (P1) at (5,4);
			\draw[thick,->] (P1) -- ++(0,0.9) node[midway, right] {$\lambda$};   
			\draw[thick,->] (P1) -- ++(0,-0.9) node[midway, right] {$\mu_2$};  
			\draw[thick,->] (P1) -- ++(-0.9,0) node[midway, below] {$\mu_1$};  

			\coordinate (P2) at (1,8);
			\draw[thick,->] (P2) -- ++(0.9,0) node[midway, above] {$\lambda$};   
			\draw[thick,->] (P2) -- ++(0,-0.9) node[midway, right] {$\mu_2$};  
			\draw[thick,->] (P2) -- ++(-0.9,0) node[midway, above] {$\mu_1$};  

			\node[draw,rounded corners,fill=blue!10,anchor=north west,inner sep=3pt]
			at (3.8,1.6) {$\frac{q_2+1}{q_1+1}<r^*$}; 
			\node[draw,rounded corners,fill=red!10,anchor=south east,inner sep=3pt]
			at (3.2,9.4) {$\frac{q_2+1}{q_1+1}> r^*$}; 

			\pgfmathsetmacro{\angle}{atan(\r)}
			\node[
				rotate=\angle,
				anchor=west,
				rounded corners=1pt,
				inner sep=2pt
			]  at (2,{\r*(1.4+1)-1}) {\large $r^*(q_1+1)-1$}; 

		\end{tikzpicture}
		\caption{State space and decision boundary of \gls{SED} for $r^*=2.5$.  Transition rates are shown for two states. For states on the line, shown as stars, any decision is optimal. }%
		\label{fig:grid_SED}
	\end{subfigure}
	\hfill
	\begin{subfigure}[t]{.35\linewidth}
		\centering
		\begin{tikzpicture}[scale=0.57, line cap=round, line join=round]
			\def\mone{7}
			\def\mtwo{11}
			\def\r{2.5} 
			\def\rn{3.5} 

			\draw[->] (0,0) node[below left] {0} -- (\mone+0.6,0) node[below right] {$q_1$};
			\draw[->] (0,0) -- (0,\mtwo+0.6) node[above left] {$q_2$};

			\draw[dashed,thick] (\mone,0) -- (\mone,\mtwo);
			\draw[dashed,thick] (0,\mtwo) -- (\mone,\mtwo);

			\begin{scope}
				\clip (0,0) rectangle (\mone,\mtwo);
				\draw[very thick,black]
				plot[domain=0:\mone] (\x, {\r*(\x+1)-1});
			\end{scope}
			\begin{scope}
				\clip (0,0) rectangle (\mone,\mtwo);
				\draw[very thick,black]
				plot[domain=0:\mone] (\x, {\rn*(\x+1)-1});
			\end{scope}

			\draw[step=1,gray!30] (0,0) grid (\mone,\mtwo);

			\pgfmathsetmacro{\eps}{1e-2}

			\foreach \x in {0,...,\mone}{
					\foreach \y in {0,...,\mtwo}{
							\pgfmathsetmacro{\ylineA}{\r*(\x+1)-1}
							\pgfmathsetmacro{\ylineB}{\rn*(\x+1)-1}

							\pgfmathsetmacro{\ylow}{min(\ylineA,\ylineB)}
							\pgfmathsetmacro{\yhigh}{max(\ylineA,\ylineB)}

							\pgfmathsetmacro{\dA}{abs(\y-\ylineA)}
							\pgfmathsetmacro{\dB}{abs(\y-\ylineB)}

							\ifdim \dA pt < \eps pt
								\node[
									star,
									star points=5,
									draw=black,
									line width=0.8pt,
									fill=white,
									inner sep=1.5pt
								] at (\x,\y) {};
							\else\ifdim \dB pt < \eps pt
									\node[
										star,
										star points=5,
										draw=black,
										line width=0.8pt,
										fill=white,
										inner sep=1.5pt
									] at (\x,\y) {};

								\else\ifdim \y pt > \ylow pt
										\ifdim \y pt < \yhigh pt
											\node[text=darkgreen,fill=darkgreen,scale=0.25] at (\x,\y) {$\square$};

										\else
											\ifdim \y pt < \ylineA pt
												\draw[blue, line width=1pt] (\x,\y) circle (2pt);
											\else
												\fill[red] (\x,\y) circle (2pt);
											\fi
										\fi

									\else
										\ifdim \y pt < \ylineA pt
											\draw[blue, line width=1pt] (\x,\y) circle (2pt);
										\else
											\fill[red] (\x,\y) circle (2pt);
										\fi
									\fi\fi\fi
						}
				}

			\pgfmathsetmacro{\angle}{atan(\r)}
			\node[
				rotate=\angle,
				anchor=west,
				rounded corners=1pt,
				inner sep=2pt
			]  at (2,{\r*(1.4+1)-1}) {\large $r^*(q_1+1)-1$};

			\pgfmathsetmacro{\angle}{atan(\rn)}
			\node[
				rotate=\angle,
				anchor=west,
				rounded corners=1pt,
				inner sep=2pt
			]  at (.5,{\rn*2 - 1}) {\large $\hat{r}(q_1+1)-1$};

		\end{tikzpicture}
		\caption{Decision boundaries for the \gls{SED} policy with $r^*=2.5$ and $\hat{r}=3.5$. The policies disagree in the points between the two lines (shown as green squares).}%
		\label{fig:SED_disagree}
	\end{subfigure}
	\caption{A load balancing system with two heterogeneous servers and the \gls{SED} routing policy.}
\end{figure}

\gls{SED} achieves low average delay at high and low traffic intensities~\citep{Banawan1992} and outperforms other policies including \gls{JSQ} and \gls{JIQ}.
At medium traffic intensity, its performance is slightly worse than other  heterogeneity-aware benchmark policies since slow servers are underutilized.
Reference~\citep{Foschini1977} proves that \gls{SED} is asymptotically optimal, and that it achieves complete resource pooling in the heavy-traffic limit.
These results motivate the use of \gls{SED} as a benchmark, given that an exact delay-minimizing policy is not available.

\paragraph{Main objective and challenges.}
Our goal is to route customers according to \gls{SED} while the service rates are unknown \emph{a priori}.
We measure the performance in terms of a \emph{regret}---the cumulative number of routing decisions that differ from those prescribed by \gls{SED}.

A natural first approach is to estimate the service rates offline and use these estimates as fixed proxies for the true parameters.
However, even when the estimation error is small, such an empirical policy suffers linear regret.
To see this, consider the decision boundaries in \Cref{fig:SED_disagree}.
The \gls{SED} and empirical \gls{SED} policies disagree in the points between the two lines.
Since both decision lines are linear with different slopes $r^*$ and $\hat{r}$, the number of states contained in the wedge between the lines is infinite for any $r^*\neq \hat{r}$.
Each visit to a state in the disagreement region accumulates regret.
Hence, since the queue length process visits the disagreement wedge repeatedly, the regret accumulation is linear.
We show this effect in a numerical experiment: even with a routing error of only $1\%$, the regret of the empirical policy grows linearly, and the growth rate increases with the system load; see \Cref{fig:regret_LASED_ESED}.

This motivates the use of online learning, where parameter estimates are continuously updated as service completions are observed.
Because routing decisions must be taken while the parameters are being learned, an exploration--exploitation trade-off arises.
In the \gls{MAB} literature, a canonical setting for this trade-off, suitable exploration schemes are known to reduce regret from linear to logarithmic growth~\citep{Lai1985,Auer2002}.

In the present queueing setting, explicit exploration may appear unnecessary. 
Consider a greedy learning policy that always routes according to the empirical \gls{SED} policy using the current online estimators, without enforcing additional exploration.
If the greedy policy routes customers to both servers on a regular basis, it will obtain samples from both service time distributions, and the estimators will converge to the true service rates.
We show in the numerical section, however, that this does not hold for all regimes. 
In some settings, when the load is low and the estimation error is large, tfhe greedy policy uses only one server, while the other one is faster. 
For as long as the initial estimation error is not corrected, this greedy policy accumulates regret suboptimally.

\paragraph{LASED.}
We develop an online learning algorithm called \gls{LASED} (\Cref{alg:LASED}) that routes customers while simultaneously learning the service rates.
\Cref{thm:regret_ub} establishes that the regret incurred by this algorithm is asymptotically bounded.
This implies that the algorithm converges to \gls{SED}.

Allowing the routing policy to change after every departure would make the queueing dynamics highly non-stationary and difficult to analyze. 
We therefore use an episodic structure, and only update estimators between episodes. 
Each episode consists of an optional exploration phase followed by an exploitation phase.
The exploration phase is only triggered when the number of samples is below a certain threshold. 
In the exploitation phase, customers are routed according to the empirical \gls{SED} policy, i.e., using $\hat{r} = \hat{\mu}_2 / \hat{\mu}_1$ as an approximation to $r^*$, where $\hat{\mu}_i$ is the empirical estimator of $\mu_i$ for $i \in \{1, 2\}$.
The approximation is updated after each episode using observed service completions.
By construction, episodes start and end when the system is empty.
This ensures that departures between episodes are independent, but also implies that episode lengths are random. 

We show that, with high probability, \gls{LASED} and \gls{SED} agree in a growing region around the origin using concentration inequalities. 
We then show that the probability that the queue length process leaves this agreement region decreases over time.
To this end we use drift analysis: we show that under the stability condition $\lambda < \mu_1+\mu_2$, the process has negative drift towards the origin, hence the probability of visiting a state with large queue lengths is sufficiently small.

\subsection{Contributions}
We propose an adaptive learning algorithm that learns \gls{SED} under unknown service rates and achieves bounded regret.
The regret analysis overcomes the challenges related to the infinite-sized disagreement region and the non-stationary routing policy, where the decision boundary changes as estimators are updated.

The regret analysis relies on a decomposition.
In particular, we identify three events that can lead to suboptimal routing decisions:
(i) forced exploration events introduced by the algorithm,
(ii) estimation errors in the service rate estimators,
and
(iii) overflow events where queue lengths grow large.
To upper bound the regret, we bound the probability of each of these events.
The estimation error events are controlled via concentration inequalities of the observed service completions.
The overflow events are analyzed using a Foster--Lyapunov approach, where we show that the queue length process has a linear negative drift outside a compact set.
This allows us to control the probability of large queues.
Finally, the forced exploration events are analyzed using a coupling argument with an $M/M/1$ queue.
This coupling yields explicit bounds on both the frequency of exploration phases and their cumulative contribution to the regret.
Analyzing the episode lengths requires understanding of the busy period for which we did not find the answers in the literature outright.

As part of the regret analysis, we define a class of policies \SED{r} for real-valued $r > 0$, which can be considered as $r$-perturbations of \gls{SED}.
We prove the stability condition and analyze the busy period for \SED{r} in \Cref{sec:SED_policy}.

\subsection{Outline}

In \Cref{sec:model}, we discuss the model and introduce the \gls{SED} and \SED{r} policies.
Next, in \Cref{sec:learning_alg} we introduce our learning algorithm and formulate the regret upper bound.
In \Cref{sec:SED_policy}, we analyze the class of parameterized \gls{SED} policies, which is used in the regret analysis in \Cref{sec:proof_outline}.
\Cref{sec:literature} provides an overview of related literature.
Lastly, in \Cref{sec:num_res} we present a numerical implementation of the algorithm, and \Cref{sec:conclusion} concludes the paper. 

\section{System description}%
\label{sec:model}

\subsection{Load balancing between two queues}

We consider a load balancing queueing system with two heterogeneous servers as shown in \Cref{fig:loadbalancing_2_queues}.
Customers arrive to a dispatcher according to a Poisson process with rate $\lambda > 0$.
Upon arrival, the dispatcher immediately routes a customer to one of the two servers.
Each server has a dedicated queue with infinite capacity where assigned customers wait for service.
After routing, the customer instantaneously joins the queue, or starts service if the server is idle.
We assume no jockeying, i.e., once a customer is assigned to a queue it cannot switch to the other queue.
Service times at server $i\in\{1,2\}$ are independent and exponentially distributed with rate $\mu_i > 0$.

We study the queue length process $\mathbf{Q}^{\pi}(t)=(Q_1^{\pi}(t), Q_2^{\pi}(t)) \in \mathbb{N}^2_{\geq 0}$.
Here, $Q_i^{\pi}(t)$ denotes the number of customers at queue $i$ at time $t \geq 0$ under routing policy $\pi$; including the customer in service at server $i$, if any.
A routing policy $\pi:\N_{\geq 0}^2 \to \{1,2\}$ is characterized by a mapping from the set of queue length vectors $\mathbf{q}=(q_1,q_2)\in\N_{\geq 0}^2$ to a routing decision $D_\pi(\mathbf{q})\in \{1,2\}$.
That is, if a customer finds the system in state $\bfq$ upon arrival, the customer is routed to server $D_\pi(\bfq)$ under policy $\pi$.

\subsection{Routing using SED policies}

We introduce a set of policies $\SED{r}$ where $r\in\R_{>0}$ is used in the decision rule.
For $\bfq=(q_1,q_2)$, the delay estimates of servers~1 and~2 are $r(q_1+1)$ and $q_2+1$, respectively. 
Let $\calD_{\SED{r}}(\bfq)\subseteq \{1,2\}$ denote the set of servers attaining the minimum. 
We assume that ties are broken by routing to server~1 for simplicity, although our analysis also holds for other tie-breaking rules and randomized routing. 
With this assumption, we define the \SED{r} policy as
\begin{align}
	\label{eq:SEDr_dec}
	D_{\SED r}(\bfq)
	 & =
	\begin{cases}
	1 &\textrm{if} \ \frac{q_2+1}{q_1+1} \geq r, \\
	2 &\textrm{else}
	.
	\end{cases}
\end{align}
Note that the \gls{SED} policy is equivalent to $\SED{r^*}$ with $r^* := \mu_2/\mu_1$.

Next, we analyze the region where \SED{r} and \SED{\hat{r}}, $r\neq \hat{r}$, disagree.
In this region, the approximate policy is suboptimal (except for coincidentally lucky tie breaks).
The policies disagree in state $\mathbf{q}$ if and only if $\bfq$ lies between the lines $r(q_1+1)-1$ and $\hat{r}(q_1+1)-1$.
This is depicted in~\Cref{fig:SED_disagree}.
In other words,
\begin{align}
	\label{eq:disagree_ql_ratio}
	D_{\SED{\hat{r}}}(\bfq) \notin \calD_{\SED{r}}(\bfq)
	\
	\iff
	\
	\Bigl( \frac{q_2 + 1}{q_1+1} - r \Bigr) \Bigl( \frac{q_2 + 1}{q_1+1} - \hat{r} \Bigr)< 0
	.
\end{align}

In particular, we observe from~\eqref{eq:disagree_ql_ratio} that the closer $r$ and $\hat{r}$, the smaller the disagreement region between the two policies.
We formalize this idea in \Cref{lem:SED_agrees}, where we give a guarantee that the two policies agree in a finite region ${\{0,\dots,m\}}^2$.
To this end, we introduce $\uval{r}{m}$, which quantifies the minimal separation between the decision boundary of \SED{r} and any state in ${\{0,\dots,m\}}^2$ that does not lie on the boundary.
$\uval{r}{m}$ characterizes the smallest error on $r$ that can change the routing decisions of the \SED{r} policy within ${\{0,\dots,m\}}^2$.
In particular,
for $r\in\R_{>0}$ and $m\in\N_{\geq 0}$ let
\begin{align}
	\label{eq:def_u_M}
	\uval{r}{m}
	 & =
	\min_{\substack{q_1,q_2\in \{0,\dots,m\} \\ q_2+1 \neq r (q_1+1)}}
	\Bigl| \frac{q_2+1}{q_1+1} - r \Bigr|
	.
\end{align}
As $m$ increases, the admissible set of pairs $q_1,q_2$ over which the minimum is taken expands.
(The square ${[m+1]}^2$ contains the square ${[m]}^2$.)
This means that $\uval{r}{m}$ is non-increasing in $m$, and $\uval{r}{m}\to 0$ as $m\to\infty$.

\begin{lemma}%
	\label{lem:SED_agrees}
	Let $r,\hat{r} \in \R_{>0}$.
	If $|r-\hat{r}| < \uval{r}{m}$, then $D_{\SED{\hat{r}}}(\mathbf{q}) \in \calD_{\SED{r}}(\mathbf{q})$ for all $\mathbf{q}\in  {\{0,\dots,m\}}^2$.
\end{lemma}

\begin{proof}
	Let $r,\hat{r} \in \R_{>0}$ with $|r-\hat{r}| < \uval{r}{m}$.
	Suppose that there exists $(x,y)\in {\{0,\dots,m\}}^2$ such that $((y+1)/(x+1)-r)((y+1)/(x+1)-\hat{r})<0$.
	Then $(y+1)/(x+1)$ lies between $r$ and $\hat{r}$, and hence $|(y+1)/(x+1)-r| < |r-\hat{r}|$.
	However, by definition of $\uval{r}{m}$, we have $|(y+1)/(x+1)-r| \geq \uval{r}{m}$, which contradicts $|r-\hat{r}| < \uval{r}{m}$.
	Therefore, no such $(x,y)$ exists. The proof is concluded by~\eqref{eq:disagree_ql_ratio}.
\end{proof}

\section{Learning Adaptive SED}
\label{sec:learning_alg}

This section discusses our learning adaptive algorithm and its performance guarantees.

\subsection{Algorithm description}
\label{sec:als_description}

The algorithm is given in \Cref{alg:LASED}.

\begin{algorithm}[h]
	\caption{Learning Adaptive SED (LASED)}
	\label{alg:LASED}
	\begin{algorithmic}[1]
		\State {\bf Input:} $0 <\mumin <\mumax$ and exploration threshold sequence $\{\alpha(k)\}_{k\in \N_{\geq 1}}$.
		\State Let $\rmin = \mumin/\mumax$ and $\rmax = \mumax/\mumin$.
		\State For $i\in\{1,2\}$, set $N_i^{(1)} = 0$, $S_i^{(1)} = 0$, and $\hat\mu_i^{(1)} = 1$. Set $\bar{r}_1=(\rmax-\rmin)/2$. \Comment{Initialization.}
		\For{episode $k=1,2,\dots$}
		\If{$N_1^{(k)} \wedge N_2^{(k)} < \alpha(k)$} \Comment{Exploration phase.} \label{line:check_dep}
		\For{$i\in\{1,2\}$}
		\State Wait for $\bigl(\lceil \alpha(k)\rceil - N_i^{(k)}\bigr) \vee 1$ customer arrivals, route each customer to server $i$.  \label{line:forced_exploration}
		\EndFor \label{line:end_exploration}
		\EndIf
		\While{the system is non-empty} \Comment{Exploitation phase.}  \label{line:begin_exploitation}
		\State Route customers according to $\SED{\bar r_k}$.
		\EndWhile  \label{line:end_exploitation}
		\For{$i\in\{1,2\}$} \Comment{Update estimators.}
		\State Record the $Z_i^{(k)}\geq 0$ departures during episode~$k$ at server~$i$, with service durations $Y_{i,1}^{(k)},\ldots,Y_{i,Z_i^{(k)}}^{(k)}$. 
		\State Update
		\begin{align}
			\label{eq:alg_updates}
			N_i^{(k+1)} & = N_i^{(k)} + Z_i^{(k)}, \qquad
			S_i^{(k+1)} = S_i^{(k)} + \sum_{j=1}^{Z_i^{(k)}} Y_{i,j}^{(k)}, \qquad
			\hat\mu_i^{(k+1)} = \frac{N_i^{(k+1)}}{S_i^{(k+1)}}
			.
		\end{align}
		\EndFor
		\State Update
		$
			\hat{r}_{k+1} = \hat\mu_2^{(k+1)}/\hat\mu_1^{(k+1)}
			.
		$
		\Comment{Empirical estimator.}
		\label{line:def_hat_r}
		\State Update
		$
			\bar{r}_{k+1}
			=
			\rmax \wedge (\rmin \vee \hat r_{k+1})
			.
		$
		\Comment{Clipped estimator.}
		\label{line:clip}
		\EndFor
	\end{algorithmic}
\end{algorithm}

Note that in \Cref{alg:LASED}, episodes start and end when the system is empty.
This implies that all customers who arrive during an episode also depart within the same episode.
$Z_i^{(k)}$ denotes the number of departures during episode~$k$, and $N_i^{(k)}= \sum_{m=1}^{k-1} Z_i^{(m)}$ denotes the cumulative number of departures observed up to the start of episode~$k$.

An episode consists of an optional exploration phase (\Crefrange{line:check_dep}{line:end_exploration}) followed by an exploitation phase (\Crefrange{line:begin_exploitation}{line:end_exploitation}).
This depends on $N_i^{(k)}$ and $\alpha(k)$; see \Cref{line:check_dep}.
An exploration phase is designed to obtain sufficient samples for both servers, where ``sufficient'' is measured by the exploration threshold~$\alpha(k)$.
In the exploitation phase of episode~$k$, customers are routed according to the \SED{\bar{r}_k} policy.
Here, $\bar{r}_k$ is the empirical estimator $\hat{r}_k$ of the true ratio $\mu_2/\mu_1$, projected on $[\rmin,\rmax]$.
At the end of an episode, estimators are updated using the observed service completions.

Note by design of the exploration phases, we have a deterministic guarantee on the number of samples at the start of each episode. 
Specifically, for $k\in\N_{\geq 2}$, with probability one,
\begin{align}
	\label{eq:dep_geq_alpha}
	N_1^{(k)} \wedge N_2^{(k)} \geq \alpha(k-1)
	.
\end{align}
To see this, note that $N_i^{(k)} \geq N_i^{(k-1)}$ since $N_i^{(k)}$ is a cumulative count.
If $N_i^{(k-1)} \geq \alpha(k-1)$, then the claim follows trivially.
On the other hand, if $N_i^{(k-1)} < \alpha(k-1)$, then exploration is triggered in episode~$k-1$ and at least $\lceil \alpha(k-1) \rceil - N_i^{(k-1)}$ customers are routed to server~$i$; see \Cref{line:forced_exploration}.
Since every arriving customer departs within the same episode, we have $N_i^{(k)} \geq N_i^{(k-1)} +\lceil \alpha(k-1) \rceil - N_i^{(k-1)} \geq \alpha(k-1)$.

\subsection{Regret definition}

We next introduce a notion of regret of \gls{LASED} with respect to the \SED{r} policy.
The regret in~\eqref{eq:def_regret_episode} is the cumulative number of routing decisions made by \gls{LASED} that differ from \SED{r}.
Since the objective of \gls{LASED} is to recover the decision structure of \SED{r}, counting decision mismatches measures how quickly the algorithm learns the correct routing behavior.

Let $H_k \in \R_{\geq 0}$ denote the starting time of episode~$k\in\N_{\geq 1}$.
Note that $H_k$ is random since each episode ends at the completion of a busy period.
Let $\tilde{A}_1,\tilde{A}_2,\dots$ denote the customers' arrival times.
The routing decision of the $j$th arriving customer is based on the state of the system just before time $\tilde{A}_j$, which is denoted by $\tilde{A}_j^-$.
Let $\bfQ(t)$ denote the queue length vector at time $t\geq 0$ under \gls{LASED}.

Let $D_\algname(t,\bfq)$ denote the routing decision of \gls{LASED} at time $t$ when the system is in state $\bfq$.
Since \gls{LASED} uses policy \SED{\bar{r}_k} in episode~$k$, routing depends on time and state.
On the other hand, the \SED{r} policy is time-homogeneous so the set of optimal decisions $\calD_{\SED{r}}$ only depends on the state.

We consider the regret of \gls{LASED} with respect to \SED{r} over the first $n\in\N_{\geq 1}$ episodes to be
\begin{align}
	\label{eq:def_regret_episode}
	R_r^{\algname}(n)
	 & =
	\sum_{k=1}^n
	\sum_{j: \tilde{A}_j^- \in [H_k,H_{k+1})}
	\ind{
		D_{\algname}\bigl(\tilde{A}_j^-,\bfQ(\tilde{A}_j^-)\bigr)
		\notin
		\calD_{\SED{r}}\bigl(\bfQ(\tilde{A}_j^-)\bigr)
	}
	.
\end{align}

\subsection{Regret upper bound}

We show that \gls{LASED} achieves bounded regret with respect to the \gls{SED} policy in \Cref{thm:regret_ub}.
Recall that \gls{SED} is equivalent to \SED{r^*} with $r^* := \mu_2/\mu_1$.
While the proof outline is discussed below, the full proof is in \Cref{app:proof_thm_regret_ub}.

\begin{theorem}
	\label{thm:regret_ub}
	Let $0 < \mumin < \mumax < \infty$.
	Assume that $\mu_1,\mu_2 \in \mathbb{Q}_{>0} \cap (\mumin,\mumax)$
	and $\lambda\in\R_{> 0}$  satisfy
	\begin{align}
		\lambda < \mu_1+\mu_2
		.
		\label{eq:stab_cond}
	\end{align}
	Let
	\begin{align}
		\label{eq:def_alpha_k}
		\alpha(k)
		 & =
		\lceil \ln(k+1)^4 \rceil
		.
	\end{align}
	Then,
	$
		\lim_{n\to\infty}
		\E(R_{r^*}^\algname(n))
		<
		\infty
		.
	$
\end{theorem}

To prove~\Cref{thm:regret_ub}, we next identify three events that may cause \gls{LASED} to make different decisions than \gls{SED}, so that the total regret can be bounded by analyzing these events individually.
Let
\begin{align}
	\label{eq:def_M_k}
	m(k) := \lfloor c_1 \ln(k)\rfloor
	,
\end{align}
which will represent the radius of a `safety' region where the learning policy agrees with \gls{SED}.
Here, $c_1 > 0$ is a constant to be chosen later.
For episode~$k$, consider the following events: 
\begin{itemize}[labelwidth=2.4cm, leftmargin=!, align=left]
	\item[(exploration)]%
	    \begin{minipage}[t]{\linewidth}
			The event that episode~$k$ starts with exploration, i.e.,
			\begin{align}
				\label{eq:event_exploration}
				\Eexplr{k} &
				=
				\bigl\{ N_1^{(k)} \wedge N_2^{(k)} < \alpha(k) \bigr\}
				.
			\end{align}
			\end{minipage}
	\item[(misestimation)]%
		\begin{minipage}[t]{\linewidth}
			The event that $\hat{r}_k$ is not in $[\rmin,\rmax]$ or deviates from $r^*$ by more than $\uval{r^*}{m(k)}$, i.e.,
	      \begin{align}
		      \label{eq:event_est_error}
		      \Eest{k} & = \{\hat{r}_k \notin [\rmin,\rmax]\} \cup \{|\hat{r}_k - r^*| \geq \uval{r^*}{m(k)}\}
		      .
	      \end{align}
		\end{minipage}
	\item[(overflow)]%
		\begin{minipage}[t]{\linewidth}
			The event that the queue of a server exceeds $m(k)$ during episode~$k$, i.e.,
	      \begin{align}
		      \label{eq:event_overflow}
		      \Eovf{k} & = \bigl\{\exists t\in[H_k,H_{k+1}]
		      :
		      \
		      Q_1(t) \vee Q_2(t) > m(k)\bigr\}
		      .
	      \end{align}
		\end{minipage}
\end{itemize}

We claim that on $(\Eest{k})^c \cap (\Eovf{k})^c \cap (\Eexplr{k})^c$, all decisions of \gls{LASED} agree with those of $\SED{r^*}$ in episode~$k$.
Indeed, observe first that in episodes without exploration, customers are routed according to \SED{\bar{r}_k} throughout the entire episode.
Observe second that without misestimation, $\SED{\bar{r}_k}$ and $\SED{r^*}$ agree for states with $0\leq q_1,q_2\leq m(k)$ by \Cref{lem:SED_agrees}.
This is because $|\bar{r}_k - r^*| < \uval{r^*}{m(k)}$.
Finally, without overflow, the queues are guaranteed to have stayed below $m(k)$---the region on which we just argued that the two policies agree.

So, on $(\Eest{k})^c \cap (\Eovf{k})^c \cap (\Eexplr{k})^c$, no regret is incurred in episode~$k$.
This implies that
\begin{equation}
	\label{eq:regret_split}
	R_{r^*}(n)
	\leq
	\sum_{k=1}^{n} \Arrepi{k}
	\Bigl(
	\ind{\Eexplr{k}} + \ind{\Eest{k}} + \ind{\Eovf{k}}
	\Bigr)
	.
\end{equation}
Here, $\Arrepi{k}$ denotes the number of arrivals in episode~$k$.
Taking expectation, using the Cauchy--Schwarz inequality, linearity, and $\sqrt{a}+\sqrt{b}+\sqrt{c} \leq \sqrt{3(a+b+c)}$ for $a,b,c\geq 0$, we find that
\begin{align}
	\label{eq:E_regret_split}
	\E(R_{r^*}(n))
	\leq
	\sum_{k=1}^{n}
	\sqrt{
		3 \E(\Arrepi{k}^2)
		\bigl(
		\P(\Eexplr{k}) + \P(\Eest{k}) + \P(\Eovf{k})
		\bigr)
	}
	.
\end{align}

The sequence $m(k)$ balances two competing requirements: it grows
(i) slow enough for sufficient estimation accuracy, 
and 
(ii) fast enough such that queue length excursions outside the safety region are sufficiently rare. 
The result relies crucially on this balance.

The terms in \eqref{eq:E_regret_split} are bounded separately in \Cref{sec:proof_outline}, completing the proof of \Cref{thm:regret_ub}.
These bounds rely on structural properties of the \SED{r} policy, which are derived first in \Cref{sec:SED_policy}.
This section provides structural insight into the dynamics of the $\SED{r}$ routing policy, but readers who are willing to take these properties as given may proceed directly to \Cref{sec:proof_outline}.

\subsection{Discussion}
Below we discuss several aspects of \Cref{thm:regret_ub}.\\

{\bf Conditions.}
We rely on two conditions.
First, \eqref{eq:stab_cond} ensures that the $\SED{r}$ system is positive recurrent for $r\in\R_{>0}$, see \Cref{cor:stab_cond}.
This guarantees that episodes of \gls{LASED} have finite expected length.
Second, the specific choice of $\alpha(k)$ ensures that exploration phases occur sufficiently often to guarantee consistent estimation of $r^*$, while remaining sparse enough so that the exploration-induced regret is summable. \\

{\bf Parameter range.}
\Cref{thm:regret_ub} assumes that the service rates $\mu_1$ and $\mu_2$ lie in a bounded interval $(\mumin,\mumax)$.
This ensures that the service rates are bounded away from zero and infinity.
These bounds may be loose, as the regret upper bound is asymptotic and does not depend on $\mumin$ and $\mumax$.\\

{\bf Restriction to rational values.}
\Cref{thm:regret_ub} is restricted to rational service rates.
This assumption ensures that  there exists a strictly positive distance between $r^*$ and any other fraction with bounded denominator.
I.e., each suboptimal value $r\neq r^*$ has a strictly positive suboptimality gap.
Such a condition is standard in the \gls{MAB} literature, where positive separation between optimal and suboptimal actions is often required~\citep{Lai1985,Auer2002}.

The argument breaks down for $r^*\in\R\setminus\mathbb{Q}$, since irrationals can be approximated arbitrarily well by rationals, see Dirichlet's approximation theorem \citep[Thm.\ 36]{Hardy1975}.
Thus, although \gls{LASED} can still be applied to systems with irrational service rates, a regret upper bound would require a different approach. \\

{\bf Restriction to two servers.}
In the two-server system, \gls{SED} is fully determined by the ratio of service rates $r^* = \mu_2 / \mu_1$.
This ratio induces a one-dimensional decision boundary in the state space; recall \Cref{fig:grid_SED}.
Any estimation error creates a two-dimensional wedge and the regret bound follows by analyzing how often the process visits the disagreement wedge. 
Consider now a model with three servers. 
The queue length process lives in a three-dimensional space, and three ratios $(q_i+1)/\mu_i$, $i\in\{1,2,3\}$, must be compared to implement the \gls{SED} policy. 
Equivalently, the state space is partitioned by three hyperplanes determined by the fractions $\mu_2/\mu_1$, $\mu_3/\mu_2$, and $\mu_3/\mu_1$, and the disagreement region is formed by the union of three-dimensional wedges. 
An error in one of the service rate estimators thus changes both decision boundaries involving that server, and can therefore affect multiple decision regions at once, see \Cref{fig:3dim_error}. 
Similarly, for a model with $n$ servers, the $n$-dimensional state space is partitioned by $\binom{n}{2}$ hyperplanes, and estimation error in one service rate affects all $n-1$ hyperplanes at once. 
Hence, the difficulty of the learning problem scales with the dimension. 
Extending the regret analysis to systems with more than two servers is left for future work.

\tdplotsetmaincoords{68}{118}
\begin{figure}[h]
	\centering
	\begin{tikzpicture}[
		tdplot_main_coords,
		scale=1.05,
		line cap=round,
		line join=round,
		axis/.style={-{Latex[length=2.2mm]}, semithick},
		box edge/.style={gray!45, thin},
		true edge/.style={thick},
		estimated edge/.style={estimate, thick, dashed},
		panel title/.style={font=\small\bfseries},
		plane label/.style={
			font=\normalsize,
			inner sep=1.3pt,
			fill=white,
			fill opacity=.82,
			text opacity=1,
			rounded corners=1pt
		}
	]

	\def\SEDqmax{4}

	\begin{scope}[xshift=-3.8cm]
		\draw[box edge] (0,0,0) -- (\SEDqmax,0,0) -- (\SEDqmax,\SEDqmax,0) -- (0,\SEDqmax,0) -- cycle;
		\draw[box edge] (\SEDqmax,0,0) -- (\SEDqmax,0,\SEDqmax);
		\draw[box edge] (\SEDqmax,\SEDqmax,0) -- (\SEDqmax,\SEDqmax,\SEDqmax);
		\draw[box edge] (0,\SEDqmax,0) -- (0,\SEDqmax,\SEDqmax);
		\draw[box edge] (0,0,\SEDqmax) -- (\SEDqmax,0,\SEDqmax) -- (\SEDqmax,\SEDqmax,\SEDqmax)
			-- (0,\SEDqmax,\SEDqmax) -- cycle;

		\filldraw[fill=pairone, fill opacity=.15, draw=pairone, true edge]
			(0,0,0) -- (2.857,4,0) -- (2.857,4,4) -- (0,0,4) -- cycle;

		\filldraw[fill=pairtwo, fill opacity=.15, draw=pairtwo, true edge]
			(0,0,0) -- (4,0,0) -- (4,2.8,4) -- (0,2.8,4) -- cycle;

		\filldraw[fill=pairthree, fill opacity=.15, draw=pairthree, true edge]
			(0,0,0) -- (2,0,4) -- (2,4,4) -- (0,4,0) -- cycle;

		\node[plane label, text=pairone!80!black] at (1.9,2.66,.45) {$P_{12}$};
		\node[plane label, text=pairtwo!85!black] at (5,1.75,2.5) {$P_{23}$};
		\node[plane label, text=pairthree!75!black] at (.65,3.25,1.3) {$P_{13}$};

		\draw[axis] (0,0,0) -- (4.65,0,0)
			node[below left=-1pt, font=\small] {$q_1$};
		\draw[axis] (0,0,0) -- (0,4.65,0)
			node[right=1pt, font=\small] {$q_2$};
		\draw[axis] (0,0,0) -- (0,0,4.65)
			node[above, font=\small] {$q_3$};
	\end{scope}

	\begin{scope}[xshift=4.2cm]
		\draw[box edge] (0,0,0) -- (\SEDqmax,0,0) -- (\SEDqmax,\SEDqmax,0) -- (0,\SEDqmax,0) -- cycle;
		\draw[box edge] (\SEDqmax,0,0) -- (\SEDqmax,0,\SEDqmax);
		\draw[box edge] (\SEDqmax,\SEDqmax,0) -- (\SEDqmax,\SEDqmax,\SEDqmax);
		\draw[box edge] (0,\SEDqmax,0) -- (0,\SEDqmax,\SEDqmax);
		\draw[box edge] (0,0,\SEDqmax) -- (\SEDqmax,0,\SEDqmax) -- (\SEDqmax,\SEDqmax,\SEDqmax)
			-- (0,\SEDqmax,\SEDqmax) -- cycle;

		\filldraw[fill=pairone, fill opacity=.08, draw=pairone!65, true edge]
			(0,0,0) -- (2.857,4,0) -- (2.857,4,4) -- (0,0,4) -- cycle;
		\filldraw[fill=pairtwo, fill opacity=.08, draw=pairtwo!70, true edge]
			(0,0,0) -- (4,0,0) -- (4,2.8,4) -- (0,2.8,4) -- cycle;

		\filldraw[fill=pairthree, fill opacity=.15, draw=pairthree, true edge]
			(0,0,0) -- (2,0,4) -- (2,4,4) -- (0,4,0) -- cycle;

		\filldraw[fill=estimate, fill opacity=.07, estimated edge]
			(0,0,0) -- (3.636,4,0) -- (3.636,4,4) -- (0,0,4) -- cycle;
		\filldraw[fill=estimate, fill opacity=.07, estimated edge]
			(0,0,0) -- (4,0,0) -- (4,2.2,4) -- (0,2.2,4) -- cycle;

		\draw[black, very thick] (0,0,0) -- (2,2.8,4);
		\draw[estimate, very thick, dashed] (0,0,0) -- (2,2.2,4);

		\node[plane label, text=pairthree!75!black] at (.65,3.25,1.3)
			{$P_{13}=\widehat P_{13}$};
		\node[plane label, text=estimate] at (3.35,3.72,.55)
			{$\widehat P_{12}$};
		\node[plane label, text=estimate] at (3.35,1.38,2.55)
			{$\widehat P_{23}$};

		\draw[axis] (0,0,0) -- (4.65,0,0)
			node[below left=-1pt, font=\small] {$q_1$};
		\draw[axis] (0,0,0) -- (0,4.65,0)
			node[right=1pt, font=\small] {$q_2$};
		\draw[axis] (0,0,0) -- (0,0,4.65)
			node[above, font=\small] {$q_3$};
	\end{scope}

	\end{tikzpicture}
	\caption{Three pairwise decision boundaries in the three-dimensional case (left). An error in the service rate estimator $\hat{\mu}_2$ simultaneously shifts the two boundaries involving that server (right). }
	\label{fig:3dim_error}
\end{figure}
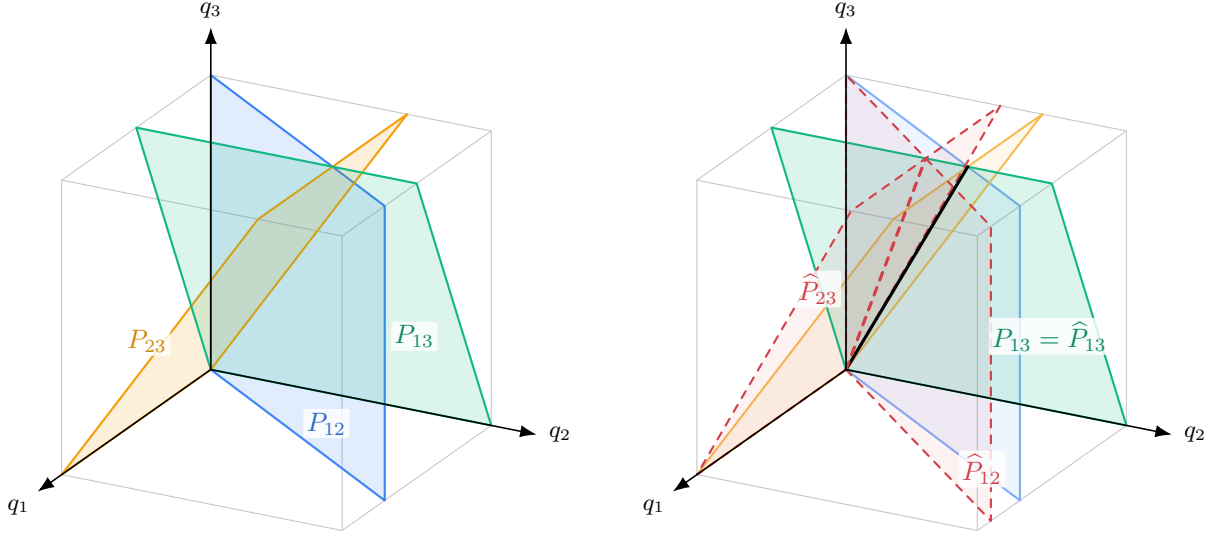
\section{The SED($r$) policy}
\label{sec:SED_policy}

In this section we analyze the stationary \SED{r} policy for arbitrary $r>0$.
Note that the \SED{1} policy is actually the \gls{JSQ} policy.
Stability conditions for the \gls{JSQ} policy are well-known~\citep{Haight1958,Moyal2017,Foley2001,Bramson2011}.
The two-server system under \gls{JSQ} is stable if and only if~\eqref{eq:stab_cond} holds, as shown in, e.g.,~\citep{Flatto1977}.
\citep[Theorem 4]{CoombsReyes2003} shows that~\eqref{eq:stab_cond} is a sufficient stability condition for the \gls{SED} policy in the two-server system.
Their proof relies on a fluid model.

We establish that~\eqref{eq:stab_cond} is a sufficient stability condition for \SED{r} for arbitrary $r>0$.
This is a result that, to the best of our knowledge, has not been established in the literature up to now.
To prove the stability condition we construct a Lyapunov function of the queue length process; see~\eqref{eq:def_lyap_func}.
We show in \Cref{lem:lyap_subgeom_drift} that the drift of this function is negative outside a finite region.
Stability then follows from the Foster--Lyapunov criterion~\citep{Meyn2009}.
Our approach is directly applicable to the queue length process, omitting the need for a fluid model.
Moreover, the Lyapunov function in \eqref{eq:def_lyap_func} can be used to derive other results such as expected return times, or moment bounds for the stationary distribution.
We later use the Lyapunov function to obtain bounds for first-entrance probabilities in \Cref{lem:hitting_probability}, bounds on the first and second moment of the busy period in \Cref{lem:E_B_bnd}, and tails on the maximum queue length during a busy period in \Cref{lem:overflow}.

We consider $\QSED{r}(t)$, the queue length process in the two-server system under the \SED{r} policy.
$\QSED{r}(t)$ is a two-dimensional \gls{BD} process with state-dependent birth and death rates.
We next introduce the infinitesimal generator $\calL_r$ of $\QSED{r}$.
If a customer arrives in state $\bfq = (q_1,q_2)$, \SED{r} routes it to server one if $q_2+1 \geq r(q_1 + 1)$, or to server two otherwise.
Hence, for any function $\psi:\N_{\geq 0}^2\to \R$,
\begin{align}
	\label{eq:def_drift_operator}
	(\calL_r \psi)(\bfq)
	 & =
	\lambda \ind{q_2+1 \geq r(q_1 + 1)} (\psi(q_1+1,q_2)- \psi(q_1,q_2))         \\
	 & \ 
	+
	\lambda \ind{q_2+1 < r(q_1 + 1)} (\psi(q_1,q_2+1) - \psi(q_1,q_2)) \nonumber \\
	 & \ 
	+
	\mu_1 \ind{q_1 > 0} (\psi(q_1-1,q_2)- \psi(q_1,q_2)) \\
	& \ +
	\mu_2 \ind{q_2 > 0} (\psi(q_1,q_2-1)- \psi(q_1,q_2)) \nonumber
	.
\end{align}

We are going to consider the function $\phi_r:\N_{\geq 0}^2\to \N$ defined by
\begin{align}
	\label{eq:def_lyap_func}
	\phi_r(q_1,q_2)
	 & =
	r \frac{q_1(q_1+1)}{2} + \frac{q_2(q_2+1)}{2}
	.
\end{align}
For $x\in\R$, we write $x^+ = \max\{x,0\}$.
Let
\begin{align}
	\delta_1  & = \mu_2 - (\lambda - \mu_1)^+ \label{eq:def_delta1},                                       \\
	\delta_2  & = \mu_1 - (\lambda - \mu_2)^+, \label{eq:def_delta2}                                       \\
	\deltamin & = \delta_1 \wedge \delta_2 = \min(\mu_1,\mu_2,\mu_1+\mu_2-\lambda) \label{eq:def_deltamin}
	.
\end{align}

Note that if~\eqref{eq:stab_cond} holds, then $\deltamin >0$.
Moreover, let
\begin{align}
	a_r & = \frac 14 \sqrt{(r \wedge 1)} \deltamin,
	\label{eq:def_a_const}
	\\
	b_r
	    & =
	2(r \vee 1)((\mu_1 \vee \mu_2)+\lambda)
	\label{eq:def_b_const}
	\\
	\scrC_r
	    & =
	\{\bfq \in\N_{\geq 0}^2: rq_1 + q_2 \leq \kappa_r\}, \qquad \kappa_r := \frac{4 b_r}{\deltamin}
	.
	\label{eq:def_C_set}
\end{align}
Note that $a_r,b_r$ and $\scrC_r$ also depend on $\lambda,\mu_1$, and $\mu_2$.
However, we omit this dependence in notation.
In the plane, the set $\scrC_r$ contains the points $(q_1,q_2)$ satisfying $q_2 \leq \kappa_r - rq_1$.
This means that $\scrC_r$ contains all points in the triangle defined by the points $(0,0)$, $(0,\kappa_r)$, and $(\kappa_r/r,0)$, as shown in \Cref{fig:plane_sets}.
If $r<\infty$, then $\scrC_r$ is finite.

We show in \Cref{lem:lyap_subgeom_drift} that $\QSED{r}(t)$ has negative drift with respect to~$\phi_r$ outside~$\scrC_r$.
The proof is in~\Cref{app:proof_lem_lyap_subgeom_drift}.

\begin{lemma}
	\label{lem:lyap_subgeom_drift}
	Let $r\in\R_{>0}$ and assume~\eqref{eq:stab_cond}.
	For $\bfq \in\N_{\geq 0}^2$,
	$
		(\calL_r\phi_r)(\bfq)
		\leq
		- a_r \sqrt{\phi_r(\bfq)} + b_r \ind{\bfq \in \scrC_r}
		.
	$
\end{lemma}

For $\scrA \subseteq \N_{\geq 0}^2$, we denote the first entrance time of $\QSED{r}(t)$ in $\scrA$ as
\begin{align}
	T_{\scrA}:= \inf_{t\geq 0}\{\QSED{r}(t) \in \scrA\}
	.
	\label{eq:def_first_entrance_time}
\end{align}
Note that $T_{\scrA}$ is a stopping time for $\QSED{r}(t)$ for any $\scrA$.\footnote{Since each subset of $\N_{\geq 0}^2$ is open and $\QSED{r}(t)$ is right continuous, we may write $\{T_{\scrA} < t\} = \bigcup_{q\in (0,t)\cap \mathbb{Q}} \{\QSED{r}(q) \in \scrA\}$. Now $\{\QSED{r}(q) \in \scrA\} \in \calF(q) \subseteq \calF(t)$ for all $q\in (0,t)\cap \mathbb{Q}$, where $\calF(t)$ is the natural filtration of $\QSED{r}(t)$. Hence, $T_{\scrA}$ is $\calF(t)$-measurable.}
Given the negative drift outside $\scrC_r$ in \Cref{lem:lyap_subgeom_drift}, we can now bound the first and second moment of the $T_{\scrC_r}$ for the set $\scrC_r$ in~\eqref{eq:def_C_set}.
This is stated in \Cref{lem:hitting_time_finite_moments}, which is proved in \Cref{app:proof_lem_hitting_time_finite_moments}. 
To prove the lemma, we also establish a stopped version of Dynkin's formula (see \Cref{lem:dynkin} in \Cref{app:dynkin}) that holds for finite-valued, possibly unbounded functions. 
Other references often assume bounded functions, see e.g.\ \cite{Dynkin1965,Meyn1993,Kushner1967}, while we needed this extension for the unbounded functions $\sqrt{\phi_r}$ and $\phi_r$. 
We use the notation $\P_\bfa$ and $\E_\bfa$ to denote probability and expectation conditional on the process $\QSED{r}(t)$ starting in state $\bfa\in\N_{\geq 0}^2$ at time $0$, i.e., $\QSED{r}(0)=\bfa$.

\begin{lemma}
	\label{lem:hitting_time_finite_moments}
	Let $r\in\R_{>0}$ and assume~\eqref{eq:stab_cond}.
	For $\bfq \in\N_{\geq 0}^2$,
	$\E_{\bfq}(T_{\scrC_r}) \leq 2a_r^{-1} \sqrt{\phi_r(\bfq)}$ and $\E_\bfq(T_{\scrC_r}^2) \leq 4a_r^{-2} \phi_r(\bfq)$.
\end{lemma}

It follows from \Cref{lem:lyap_subgeom_drift,lem:hitting_time_finite_moments} that $\QSED{r}(t)$ is stable for any $0 < r < \infty$.
This is stated in \Cref{cor:stab_cond}.

\begin{corollary}
	\label{cor:stab_cond}
	Let $r\in\R_{>0}$.
	If~\eqref{eq:stab_cond} holds, then $\QSED{r}(\cdot)$ is positive recurrent.
\end{corollary}

\begin{proof}
	$\QSED{r}(\cdot)$ is a Markov process on $\N_{\geq 0}^2$.
	$\QSED{r}(\cdot)$ is irreducible, since with positive probability there exists a path between any pair of states.
	Since $\QSED{r}(\cdot)$ is irreducible, it is $\psi$-irreducible where $\psi$ is the counting measure~\citep[Sec.~1.3.1]{Meyn2009}.

	Under condition~\eqref{eq:stab_cond}, $\deltamin > 0$ in~\eqref{eq:def_deltamin}, and therefore $a_r>0$ in~\eqref{eq:def_a_const}.
	Since $\scrC_r$ in~\eqref{eq:def_C_set} is finite and $\QSED{r}(\cdot)$ is an irreducible \gls{BD} process with uniformly bounded jump rates, the expected exit time of $\scrC_r$ is finite. 
	Moreover, the expected first entrance time to $\scrC_r$ from any state neighboring $\scrC_r$ is finite by \Cref{lem:hitting_time_finite_moments}.
	It thus follows that the first return time of $\scrC_r$ has finite expectation. 
	Every finite subset of a countable state space is petite; see, e.g.,~\citep[Sec.~8.4.3]{Meyn2009}.
	Therefore, positive recurrence of $\QSED{r}(\cdot)$ follows from~\citep[Thm.~10.4.10~(ii)]{Meyn2009}.
\end{proof}

\subsection{Hitting probability analysis}

In this section we bound the probability that $\QSED{r}(t)$ reaches a large queue length before entering the set $\scrC_r$ in~\eqref{eq:def_C_set}.
For $m\in\N_{\geq 0}$ we denote
\begin{align}
	\scrM_m := \{\bfx \in \N_{\geq 0}^2: \ x_1 \vee x_2 \geq m+1\}
	.
	\label{eq:def_set_M}
\end{align}
$\scrM_m$ is shown in \Cref{fig:plane_sets}.

We aim to bound the probability that $\QSED{r}(t)$ enters $\scrM_m$ before $\scrC_r$ when the process starts from within the region between $\scrC_r$ and $\scrM_m$.
We do so by applying the \gls{OST} to a suitably constructed supermartingale, which we define next.

Let $\QSEDdiscrete{r}{n}$, $n\in\N_{\geq 0}$, be the embedded discrete-time Markov jump chain obtained by uniformization of $\QSED{r}(t)$ at jump times.
In particular, $\QSEDdiscrete{r}{n}$ satisfies for any $\psi:\N_{\geq 0}^2 \to\R$,
\begin{align}
	\E(\psi(\QSEDdiscrete{r}{n+1}) - \psi(\QSEDdiscrete{r}{n}) \mid \QSEDdiscrete{r}{n} = \bfq) & = \frac{(\calL_r \psi)(\bfq)}{\nu_r(\bfq)}
	,
	\label{eq:Q_discrete_drift_op}
\end{align}
with $\calL_r$ in~\eqref{eq:def_drift_operator} and $\nu_r(\bfq) = \lambda + \mu_1 \ind{q_1 > 0}  + \mu_2 \ind{q_2>0}$ the total jump rate of $\QSED{r}(t)$ out of state $\bfq$.

Define
\begin{align}
	\label{eq:exp_martingale_theta}
	\theta_r
	 & :=
	\min
	\Bigl(
	\frac{\ln(2)}{\sqrt{2(r \vee 1)}}, \frac{a_r}{4(r \vee 1)(\lambda+\mu_1+\mu_2)}
	\Bigr)
	.
\end{align}
Note that $\theta_r > 0$ if $r\in\R_{>0}$ and~\eqref{eq:stab_cond} holds, since the latter implies $\deltamin > 0$ in~\eqref{eq:def_deltamin} and therefore $a_r>0$ in~\eqref{eq:def_a_const}.
We consider
\begin{align}
	\label{eq:def_W_martingale}
	W_r^n & = \exp(\theta_r \sqrt{\phi_r(\QSEDdiscrete{r}{n})})
\end{align}
with $\phi_r$ in~\eqref{eq:def_lyap_func},
and show that the stopped process $\{W_r^{n \wedge T_{\scrC_r \cup \scrM_m}}\}$ is a supermartingale in \Cref{lem:martingale}.
The proof in \Cref{app:proof_lem_martingale} uses that $\phi_r$ has negative drift outside $\scrC_r$, as shown in \Cref{lem:lyap_subgeom_drift}.
Let $\calF^n = \sigma(\QSEDdiscrete{r}{0},\dots,\QSEDdiscrete{r}{n})$ denote the natural filtration of $\QSEDdiscrete{r}{n}$.
\begin{lemma}
	\label{lem:martingale}
	Let $r\in\R_{>0}$, $m\in\N_{\geq 1}$, and assume~\eqref{eq:stab_cond}.
	For $n \in\N_{\geq 0}$,
	$
		\E(W_r^{(n+1) \wedge T_{\scrC_r \cup \scrM_m} } \mid \calF^n)
		\leq
		W_r^{n\wedge T_{\scrC_r \cup \scrM_m} }
		.
	$
\end{lemma}
\Cref{lem:hitting_probability} bounds the probability that $\QSED{r}(t)$ enters $\scrM_m$ before $\scrC_r$.
Its proof is in \Cref{app:proof_lem_hitting_probability}.
\begin{lemma}
	\label{lem:hitting_probability}
	Let $r\in\R_{>0}$, $m\in\N_{\geq 1}$, and assume~\eqref{eq:stab_cond}.
	For $\bfy\in\N_{\geq 0}^2\setminus (\scrC_r\cup\scrM_m)$,
	$
		\P_{\bfy}(\QSED{r}(T_{\scrC_r \cup \scrM_m}) \in \scrM_m) \leq \allowbreak
		\exp\bigl(-\theta \bigl(m \sqrt{(r \wedge 1)/2} - \sqrt{\phi(\bfy)}\bigr)\bigr)
		.
	$
\end{lemma}

\subsection{Tail bound for the maximum queue length in a busy period}
\label{sec:tail_bnd_QL}

\Cref{lem:overflow_bp} bounds the probability that $\QSED{r}(t)$ exceeds $m$ in a busy period.
While the proof is in \Cref{app:proof_lem_overflow_bp}, we outline the main arguments here.

$\QSEDsingle{r}{1}(t) \vee \QSEDsingle{r}{2}(t)$ exceeds $m$ if and only if $\QSED{r}(t)$ enters $\scrM_m$ in~\eqref{eq:def_set_M}.
A busy period starts from the origin $\mathbf{0} = (0,0)$ which lies in $\scrC_r$; recall~\eqref{eq:def_C_set}.
To reach $\scrM_m$ within a busy period, $\QSED{r}(t)$ must therefore exit $\scrC_r$.
The process can only exit $\scrC_r$ via its outer boundary
\begin{align}
	\scrB_r
	 & :=
	\{\bfq \in\N_{\geq 0}^2: \kappa_r < rq_1 + q_2  \leq \kappa_r + (r \vee 1) \}
	.
	\label{eq:def_set_B}
\end{align}
$\scrB_r$ is shown in \Cref{fig:plane_sets}.
If the process is in $\scrB_r$, the process can either (i) return to $\scrC_r$, or (ii) continue outward to hit $\scrM_m$.

The proof of \Cref{lem:overflow_bp} uses a gambler's ruin style argument between the regions $\scrC_r$, $\scrB_r$, and $\scrM_m$.
We use that there is a non-zero probability $p_r$ to hit $\mathbf{0}$ before entering $\scrB_r$ when starting from $\scrC_r$.
We also use that the probability to enter $\scrM_m$ when starting in $\bfx\notin\scrC_r \cup \scrM_m$ is related to $\phi_r(\bfx)$ via \Cref{lem:hitting_probability}.
Note that $\scrB_r$ is disjoint from $\scrC_r \cup \scrM_m$, so we may apply \Cref{lem:hitting_probability} for $\bfx\in\scrB_r$.
We therefore show that $\max_{\bfy\in\scrB_r} \phi_r(\bfy)$ is upper bounded by a constant~$\xi_r$.

We next specify~$p_r$ and~$\xi_r$.
The resulting bounds are not tight, but suffice for our purposes.

	{\it Emptying probability $p_r$.}
Suppose $\QSED{r}(0) = \bfx \in\scrC_r$.
$\QSED{r}(t)$ will hit $\bf0$ before entering $\scrB_r$ if all customers depart before the next arrival.
The arrival rate is $\lambda$ while the departure rate is $\mu_1\ind{x_1>0} + \mu_2\ind{x_2>0}$.
The probability that the next event is a departure is thus at least $\min_{i\in\{1,2\}} \mu_i / (\lambda+\mu_i) = (\mu_1 \wedge \mu_2) / (\lambda+(\mu_1 \wedge \mu_2))$.
The total number of customers in state $\bfx\in\scrC_r$ is $x_1+x_2\leq \kappa_r ((1/r) \vee 1)$ by~\eqref{eq:def_C_set}.
We therefore define
\begin{align}
	\label{eq:bnd_empty_before_arr}
	p_r :=
	\Bigl(
	\frac{\mu_1 \wedge \mu_2}{\lambda+(\mu_1 \wedge \mu_2)}
	\Bigr)
	^
	{\kappa_r ((1/r) \vee 1)}
	\in (0,1)
	.
\end{align}

{\it Boundary potential $\xi_r$.}
It can be verified that $\phi_r(\bfq) \leq ((1/r) \vee 1)(rq_1+q_2)^2 + (rq_1+q_2)$ for $\bfq\in\N_{\geq 0}^2$.
By definition of $\scrB_r$, in~\eqref{eq:def_set_B}, we have
\begin{align}
	\label{eq:bnd_phi_on_B}
	\max_{\bfy\in\scrB_r} \phi_r(\bfy)
	 & \leq
	((1/r) \vee 1) \bigl(\kappa_r + (r \vee 1)\bigr)^2 + \kappa_r + (r \vee 1) =: \xi_r
	.
\end{align}

We write $T_{\mathbf{0}} := T_{\{\mathbf{0}\}}$.
We can now state \Cref{lem:overflow_bp}.

\begin{lemma}
	\label{lem:overflow_bp}
	Let $r\in\R_{>0}$, $m\in\N_{\geq 1}$, and assume~\eqref{eq:stab_cond}.
	Then
	$
		\P_{\mathbf{0}}
		\bigl(
		\max_{0 \leq t\leq T_{\mathbf{0}}}
		\bigl(\QSEDsingle{r}{1}(t) \vee \QSEDsingle{r}{2}(t)\bigr)
		> m
		\bigr)
		\leq
		p_r^{-1} \exp\bigl(-\theta_r \bigl(m \sqrt{(r \wedge 1)/2} - \sqrt{\xi_r}\bigr)\bigr)
		.
	$
\end{lemma}

\begin{figure}[h]
	\begin{center}
		\includegraphics[width=.55\linewidth]{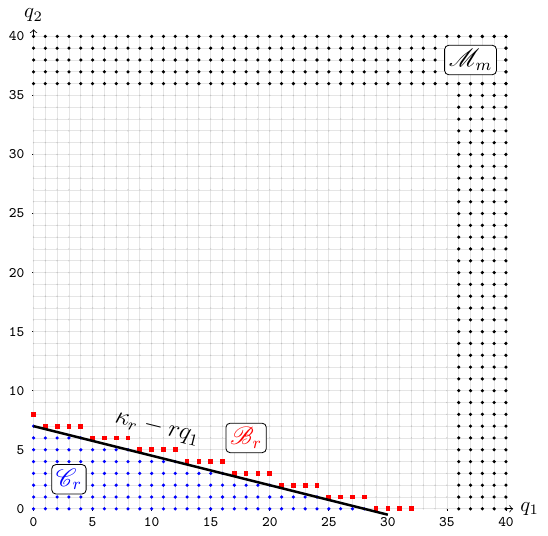}
	\end{center}
	\caption{
		Geometry of the sets $\scrC_r$, $\scrB_r$, and $\scrM_m$ in the plane $\N_{\geq 0}^2$ for $\lambda = 5$, $\mu_1 = 3$, $\mu_2 = 4$, $r=1/4$, and~$m=35$.
		In this case, the line $q_2 = \kappa_r - q_1$ intersects the $q_2$-axis at $q_2 = 7$ and the $q_1$-axis at $q_1 = 28$.
	}
	\label{fig:plane_sets}
\end{figure}

\subsection{Moment bounds for hitting times}
We next derive bounds on the first and second moment of $T_{\mathbf{0}}$.
We first control the hitting time of either the origin or the outer boundary $\scrB_r$ when $\QSED{r}(t)$ starts in $\scrC_r$ in \Cref{lem:bnd_hit_0B}.
Its proof is in \Cref{app:proof_lem_bnd_hit_0B}.

Let
\begin{align}
	\label{eq:def_ell_r}
	\ell_r
	 & :=
	\frac{\lceil ((1/r) \vee 1) \kappa_r \rceil }{\mu_1 \wedge \mu_2}
	.
\end{align}

\begin{lemma}
	\label{lem:bnd_hit_0B}
	Let $r\in\R_{>0}$ and assume~\eqref{eq:stab_cond}.
	Then
	$
		\max_{\bfy\in\scrC_r} \E_\bfy(T_{\mathbf{0}\cup\scrB_r})
		\allowbreak
		\leq
		\allowbreak
		2 \ell_r e^{\lambda \ell_r}
	$
	and \allowbreak
	$
		\max_{\bfy\in\scrC_r} \E_\bfy(T_{\mathbf{0}\cup\scrB_r}^2)
		\allowbreak
		\leq
		\allowbreak
		8 \ell_r^2 e^{2\lambda \ell_r}
		.
	$
\end{lemma}

\Cref{lem:E_B_bnd} bounds the first and second moment of $T_{\mathbf{0}}$.
Its proof in \Cref{app:proof_lem_E_B_bnd} uses the tail bound in \Cref{lem:bnd_hit_0B}.
Recall $p_r$ and $\xi_r$ as defined in~\eqref{eq:bnd_empty_before_arr} and~\eqref{eq:bnd_phi_on_B}, respectively.

\begin{lemma}
	\label{lem:E_B_bnd}
	Let $r\in\R_{>0}$ and assume~\eqref{eq:stab_cond}.
	For $\bfq\in\N_{\geq 0}^2$,
	\begin{align}
		\label{eq:bnd_B_final}
		\E_\bfq(T_{\mathbf{0}})
		 & \leq
		\beta_r(\bfq)
		,
		\qquad
		\beta_r(\bfq)
		:=
		2a_r^{-1} \sqrt{\phi_r(\bfq)}
		+
		p_r^{-1} (2 \ell_r e^{\lambda \ell_r} + 2a_r^{-1} \sqrt{\xi_r})
		,
		\\
		\label{eq:bnd_B_squared_final}
		\E_\bfq(T_{\mathbf{0}}^2)
		 & \leq
		\gamma_r(\bfq),
		\qquad
		\gamma_r(\bfq)
		:=
		8 a_r^{-2} \phi_r(\bfq)
		+
		16 p_r^{-2} (8 \ell_r^2 e^{2 \lambda \ell_r} + 4a_r^{-2} \xi_r)
		.
	\end{align}
\end{lemma}

\section{Proof of bounded regret}
\label{sec:proof_outline}

In this section we bound the terms in the regret decomposition~\eqref{eq:E_regret_split} individually, to complete the proof of \Cref{thm:regret_ub}.
Assume that $\mumin, \mumax, \mu_1, \mu_2, \lambda$, and $\alpha(k)$ satisfy the criteria of \Cref{thm:regret_ub}.
Define
\begin{equation}
	\rmin := \frac{\mumin}{\mumax}
	,
	\quad
	\rmax := \frac{\mumax}{\mumin}
	,
	\quad
	\textnormal{and}
	\quad
	r^* := \frac{\mu_2}{\mu_1}
	.
\end{equation}

\subsection{Bounding the misestimation probability}
\label{sec:estimation_error}

The misestimation event in~\eqref{eq:event_est_error} captures both the violation of the admissible range $[\rmin,\rmax]$, and large estimation error.
We show in \Cref{lem:est_out_of_range} that the event is dominated by the second term for sufficiently large $k$.

\begin{lemma}
	\label{lem:est_out_of_range}
	There exists $k_0 > 0$ such that for $k\geq k_0$,
	$
		\P(\Eest{k})
		\leq
		\P(|\hat{r}_k - r^*| \geq \uval{r^*}{m(k)})
		.
	$
\end{lemma}

\begin{proof}
	Since $\mu_1,\mu_2\in(\mumin,\mumax)$, there exists $\eps>0$ such that $[r^*-\eps,r^*+\eps]\subseteq[\rmin,\rmax]$.
	Note that
	(i) $m(k)\to\infty$ as $k\to \infty$, and
	(ii) $u_{r^*}(\ell)\to 0$ as $\ell\to \infty$.
	Therefore, $\uval{r^*}{m(k)}\to 0$ as $k\to\infty$, and there exists a $k_0(\eps)\in\N$ such that
	$
		\forall k\geq k_0(\eps): \ \uval{r^*}{m(k)} < \eps
		.
	$

	Let $k^*\geq k_0(\eps)$.
	If
	$
		|\hat{r}_{k^*}-r^*| < \uval{r^*}{m(k^*)},
	$
	then
	$
		\hat{r}_{k^*}
		\geq
		r^* - \uval{r^*}{m(k^*)}
		\geq
		r^* - \eps
		\geq
		\rmin
		.
	$
	Similarly, if $|\hat{r}_{k^*}-r^*| < \uval{r^*}{m(k^*)}$, then $\hat{r}_{k^*} \leq r^* + \uval{r^*}{m(k^*)} \leq \rmax$.
	Thus, for $k\geq k_0(\eps)$, $\{|\hat{r}_k-r^*| < \uval{r^*}{m(k)}\} \subseteq \{\hat{r}_k \in [\rmin,\rmax]\}$; or equivalently
	$
		\{\hat{r}_k \notin [\rmin,\rmax]\} \subseteq \{|\hat{r}_k-r^*| \geq \uval{r^*}{m(k)}\}
		.
	$
	This implies the result.
\end{proof}

\Cref{lem:est_error} bounds the probability that the estimator~$\hat{r}_k$ deviates far from the true value~$r^*$.
It uses the assumption $r^* \in \mathbb{Q}_{>0}$.
It follows from \Cref{lem:est_out_of_range,lem:est_error} that the misestimation probability in~\eqref{eq:event_est_error} decreases sufficiently fast.
In particular, the bound in \Cref{lem:est_error} is $o(k^{-p})$ for any $p>0$.

\begin{lemma}
	\label{lem:est_error}
	Let $r^* = a/b$ in lowest terms, i.e., $a,b\in\N_{\geq 1}$ with $\gcd(a,b)=1$.
	Let $k\in\N_{\geq 2}$ be such that
	$
		k
		\geq
		\exp(1/c_1)
		.
	$
	Then
	$
		\P\bigl(|\hat{r}_k - r^*| \geq \uval{r^*}{m(k)}\bigr)
		\leq
		4 k^{-\ln(k) / (512 a^2 c_1^2)}
		.
	$
\end{lemma}

The proof of \Cref{lem:est_error} in \Cref{app:proof_lem_est_error} uses \Cref{lem:number_approximation,lem:rhat_subexp}.
\Cref{lem:number_approximation} provides a lower bound on $u_r(m)$ in~\eqref{eq:def_u_M} for any rational $r>0$.

\begin{lemma}
	\label{lem:number_approximation}
	Let $r \in \mathbb{Q}_{>0}$ with $r = a/b$ in lowest terms.
	For $m\in\N_{\geq 1}$,
	$
		u_r(m) \geq 1/((m+1)b)
		.
	$
\end{lemma}

\begin{proof}
	For all $n_1,n_2\in\N_{\geq 1}$ such that $n_1/n_2\neq a/b$,
	$
		| n_1/n_2 - a/b |
		=
		|n_1 b - n_2 a| / (n_2 b)
		\geq
		1 / (n_2 b)
		.
	$
	Taking the minimum over $1 \leq n_1 \vee n_2 \leq m+1$ yields $u_r(m)\geq 1/((m+1)b)$.
\end{proof}

\Cref{lem:rhat_subexp} is a concentration inequality for $\hat{r}_k$.
Its proof is in \Cref{app:proof_lem_rhat_subexp}.

\begin{lemma}[Concentration inequality for $\hat{r}_k$]
	\label{lem:rhat_subexp}
	Let $k, m\in\N_{\geq 1}$.
	For any $0 < \eps < 2r^*$,
	$
		\P(|\hat r_k - r^*| \geq \eps, \ N_1^{(k)} \wedge N_2^{(k)} \geq m)
		\leq
		4 \exp(- m \eps^2 / (128 (r^*)^2))
		.
	$
\end{lemma}

\subsection{Bounding the exploration probability}
\label{sec:exploration}

We bound the probability that episode~$k$ starts with an exploration phase in \Cref{lem:exploration}.
In its proof in \Cref{app:proof_lem_exploration}, we show that the probability of observing at least one departure from each server in an episode is sufficiently large.
That implies that the number of departures grows linearly with high probability.
Since the exploration threshold sequence $\alpha(k)$ in~\eqref{eq:def_alpha_k} is logarithmic, the exploration probability decreases.

The difficulty in proving \Cref{lem:exploration} is the dependence between episodes: routing decisions depend on past observations via the estimator $\bar{r}_k$.
To address this, we exploit the regenerative structure of the episodes.
Conditional on the history up to the start of episode $k$, the estimator $\bar{r}_k$ is fixed and the queue dynamics during episode $k$ are independent of previous episodes.
We use that when the estimator is sufficiently concentrated, the probability of observing at least one departure from both servers within one episode is uniformly bounded from below; see \Cref{lem:at_least_one_departure}. 

Let $\rho_i := \lambda/\mu_i$, $i\in\{1,2\}$, and
\begin{align}
	\label{eq:def_vr}
	v^*
	 & :=
	\Bigl(
	\frac{\lambda}{\lambda+\mu_1} \frac{r^*}{r^*+2}
		(\rho_1 \wedge 1)^{2/r^*}
	\Bigr)
	\wedge
	\Bigl(
	\frac{\lambda}{\lambda+\mu_2} \frac{2}{3r^*+2}
		(\rho_2 \wedge 1)^{3r^*/2}
	\Bigr)
	\in (0,1)
	.
\end{align}
Recall that $Z_i^{(k)}$ is the number of departures from server~$i$ during episode $k$ of  \gls{LASED} and that $N_i^{(k)} = \sum_{m=1}^{k-1} Z_i^{(m)}$.
We denote the event that there is at least one departure from both servers in episode $k$ by
\begin{align}
	\label{eq:def_good_episode}
	\calG^{(k)}
	 & :=
	\bigl\{Z_1^{(k)}\wedge Z_2^{(k)} \geq 1\bigr\}
	.
\end{align}
Let $\calF_{m}$ denote the $\sigma$-algebra generated by all arrivals, departures, and routing decisions during episodes $1,\dots,m$.

\begin{lemma}
	\label{lem:exploration}
	Let $k\in\N_{\geq 1}$ satisfy
	$
		k \geq 4\alpha(k)/v^* + 1
		.
	$
	Then
	$
		\P(\Eexplr{k})
		\leq
		2\exp(-(v^*)^2 k/64)
		+
		2^{10} \exp(17) (k-1)^{-3}
		.
	$
\end{lemma}

We next establish a uniform lower bound on $\calG^{(k)}$ conditional on the history in \Cref{lem:at_least_one_departure}.
Let us give a heuristic argument of the proof.
The full proof is in \Cref{app:proof_lem_at_least_one_departure}.

If episode~$k$ starts with an exploration phase, then at least one customer is routed to each server during episode~$k$ by construction; see \Cref{line:forced_exploration} of \Cref{alg:LASED}.
Since an episode ends only when the system becomes empty, these customers complete service before the episode terminates.
Hence, both servers generate at least one departure with probability~1 in this case.

On the other hand, if episode~$k$ only has an exploitation phase then customers are routed according to $\SED{\bar r_k}$.
Consider for example the case $\bar{r}_k > 1$.
Then the first customer of the busy period is routed to server two.
Initially, the dynamics therefore behave as an $M/M/1$ queue at server two.
A customer is routed to server one only once the queue at server two becomes sufficiently large, namely when $Q_2(t) \geq \bar r_k - 1$; recall \Cref{fig:grid_SED}.
The proof therefore reduces to bounding the probability that the queue at server two reaches this level before the busy period ends, which can be done using standard $M/M/1$ arguments.
The case $\bar r_k \leq 1$ follows by symmetry.

\begin{lemma}
	\label{lem:at_least_one_departure}
	For $k\in\N_{\geq 1}$,
	$
		\P\bigl(\calG^{(k)} \mid \calF_{k-1}\bigr)
		\geq
		v^*\ind{|\bar{r}_k - r^*| < r^*/2}
	$
	almost surely.
\end{lemma}

The proof of \Cref{lem:exploration} also uses \Cref{lem:copulas}, a concentration bound on a sum of not necessarily independent Bernoulli random variables.
Its proof is in \Cref{app:proof_lem_copulas}.

\begin{lemma}
	\label{lem:copulas}
	Let $c>0$ and let $d\in\N_{\geq 2}$.
	Let $I_1,\dots,I_k$ be Bernoulli random variables with $\P(I_m=1)\geq 1-c m^{-d}$, $m=1,\dots,k$.
	If $k\geq 8$, then $\P\bigl(\sum_{m=1}^k I_m < k/2\bigr) \leq  2^{3d-1} c/((d-1) k^{d})$.
\end{lemma}

\subsection{Bounding the overflow probability}
\label{sec:overflow}

\Cref{lem:overflow} bounds the overflow probability conditional on the event that episode~$k$ has no exploration phase.
On this event, the episode is a busy period of the $\SED{\bar{r}_k}$ policy, so we can use \Cref{lem:overflow_bp}.
The proof of \Cref{lem:overflow} is in \Cref{app:proof_lem_overflow}.

\begin{lemma}
	\label{lem:overflow}
	Assume~\eqref{eq:stab_cond}.
	There exist $k_0\in \N_{\geq 1}$, $p_*, \xi_*, \theta_* \in (0,\infty)$ such that for $k\geq k_0$,
	$
		\P(\Eovf{k} \mid (\Eexplr{k})^c)
		\leq
		p_*^{-1} \exp(-\theta_* (m(k)-\xi_*))
		.
	$
\end{lemma}

\subsection{Bounding the expected arrivals}
\label{sec:arrivals}

\Cref{lem:bnd_arrivals_episode} bounds the first and second moment of the number of arrivals in one episode of \gls{LASED}.
The proof is in \Cref{app:proof_lem_bnd_arrivals_episode}.
Here, we have to distinguish between the two types of episodes of \gls{LASED}.
Some episodes only have an exploitation phase, and therefore coincide with a single busy period.
For these episodes, we can use the busy period analysis of $\QSED{r}(t)$ in \Cref{lem:E_B_bnd} directly.
The other episodes have additional arrivals due to exploration, and thus require more analysis.
We show that the number of arrivals during exploration is bounded from above by a constant by construction of the algorithm.
Recall $\beta_r(\cdot)$ and $\gamma_r(\cdot)$ as defined in~\eqref{eq:bnd_B_final} and~\eqref{eq:bnd_B_squared_final}, respectively.

\begin{lemma}
	\label{lem:bnd_arrivals_episode}
	Let $0<r<\infty$.
	For $k\in\N_{\geq 1}$,
	\begin{align}
		\E(\Arrepi{k}\mid \bar{r}_k = r)
		 & \leq
		14 + \lambda \beta_r(7,7),
		\label{eq:bnd_E_arr_epi_k}
		\\
		\E(\Arrepi{k}^2\mid \bar{r}_k = r)
		 & \leq
		196 + 30 \lambda \beta_r(7,7)  + 2\lambda^2 \gamma_r(7,7)
		\label{eq:bnd_E2_arr_epi_k}
		.
	\end{align}
\end{lemma}

The proof of \Cref{thm:regret_ub} is completed in \Cref{app:proof_thm_regret_ub} by combining the regret decomposition~\eqref{eq:E_regret_split} with the bounds on the event probabilities and arrival counts.

\section{Related work}%
\label{sec:literature}

\paragraph{Load balancing policies.}
In load balancing systems with identical servers, \gls{JSQ} has been proven to be delay-optimal under Markovian assumptions~\citep{Winston1977,Weber1978}.
However, optimality does not hold for every service time distribution~\citep{Whitt1986}.
In systems with heterogeneous servers, a delay-optimal policy is not known.
In fact, reference~\citep{Banawan1989} uses a semi-\gls{MDP} to prove that a delay-optimal policy based only on the instantaneous queue lengths generally does not exist.

The analysis of routing in heterogeneous systems is generally intractable since there is no symmetry among servers as in homogeneous systems~\citep{Ma2025}.
This means that state space reduction techniques do not readily apply.
Still, a wide range of policies have been studied in this setting, including heuristic extensions of policies that are known to perform well in homogeneous systems.
Some notable examples are
\gls{JSQ}~\citep{Haight1958,Cohen1998},
\gls{SED}~\citep{Banawan1992,Selen2016,Lui1995},
\gls{JIQ}~\citep{Lu2011},
Random Routing~\citep{Adan1991},
and
the Power-of-$d$-Choices algorithm~\citep{vanderBoor2022,HurtadoLange2021,Jaleel2022}.

Load balancing with heterogeneous servers is often studied in asymptotic regimes such as many-server or heavy-traffic limits~\citep{vanderBoor2022,HurtadoLange2021,Eryilmaz2012,Gardner2021,Jaleel2022,Chen2012,Bhambay2022,Liu2025}.
Restricting to an asymptotic regime typically enables a more detailed analysis.
For example, existing work establishes throughput and heavy-traffic optimality for variants of the Power-of-$d$-Choices algorithm~\citep{HurtadoLange2021,Gardner2021}, workload optimality in diffusion limits~\citep{Chen2012}, and delay optimality in fluid limits for speed-aware variants of \gls{JSQ}~\citep{Bhambay2022}.

\paragraph{Analysis of SED.}
The analysis of \gls{SED} in heterogeneous load balancing systems is typically intractable~\citep{Whitt1986}.
The vector of queue lengths at the different servers can be modeled as a non-negative random walk.
Since \gls{SED} is state-dependent, the random walk is inhomogeneous.
Challenges in the analysis of this random walk for as few as two servers are discussed in~\citep[Ch.~10]{Fayolle2017}.
There is no known closed-form solution to the stationary distribution for an arbitrary number of servers.

Yet, approximations and bounds do exist.
For example, reference~\citep{Selen2016} uses the compensation approach to prove that the stationary distribution for two servers can be expressed as a series of product forms with coefficients determined recursively.
Performance metrics such as the mean response time can then be computed.
And reference~\citep{Lui1995} provides upper and lower bounds on the mean response time of \gls{SED}  for two servers.
These results highlight that \gls{SED} is, at present, only partially understood.

\paragraph{Learning in load balancing/skill-based queues.}
In recent years, a variety of algorithms have been proposed for joint learning and optimization in load balancing and skill-based queues with unknown parameters.
These include bandit-inspired algorithms~\citep{Fu2023,Krishnasamy2021,Stahlbuhk2021,Freund2023,vanKempen2026}, policy gradient methods~\citep{Jali2024}, an \gls{UCB} guided MaxWeight algorithm~\citep{Yang2023}, and a learning algorithm based on marginal service rates~\citep{Zhang2025}.

A closely related model is studied in~\citep{Choudhury2021}, where a load balancing system is analyzed with unknown service rates and unobservable queue lengths.
The authors develop a bandit-based exploration algorithm for delay minimization and prove a regret upper bound, but their bound does not give a constant-regret guarantee.
Their analysis is restricted to static weighted random routing policies and regret is measured with respect to the best fixed randomized routing rule.
In contrast, we use the state-dependent \gls{SED} policy as the oracle benchmark.
Since its stationary distribution is unknown, we define regret statewise and analyze it at the sample-path level.

Our algorithm exploits the regenerative structure of the queueing process by using busy periods as episodes.
This way, each episode starts from the same state, and the stochastic evolution is, conditionally on the chosen policy, independent of other episodes.
This approach is inspired by the episodic learning framework for admission control proposed in~\citep{Cohen2024}.
Related regenerative structures in learning and control of queueing systems are also used in~\citep{Krishnasamy2021,Dai2022}.

\section{Numerical experiments}%
\label{sec:num_res}

This section analyzes the properties of \gls{LASED} numerically.
Throughout, $\mumin = 0.01$, $\mumax=100$, and $\alpha(k) = \lceil \ln(k+1)^4\rceil$.
Unless otherwise stated, the results are based on 100 independent replications and all plots show $95\%$-confidence bands.

\subsection{Empirical methods versus learning}
\label{sec:num_res_empirical}
In this section we analyze \gls{ESED}, which uses fixed estimates, to show that even small estimation error of $1\%$ can lead to poor performance of empirical policies. 
A small error means that the slopes of the true and estimated decision lines are close to each other, and hence the disagreement wedge has a small angle; see \Cref{fig:ESED_wedge}.

Under \gls{ESED}, the regret grows linearly in time, and the slope increases as the load becomes higher, as shown in \Cref{fig:regret_ESED}.
The process visits the disagreement region on a regular basis, and more often when the load is higher. 
The regret of \gls{LASED} grows more slowly over time, as shown in \Cref{fig:regret_LASED}. 
For low loads, the regret curve converges to a constant value. 
For high loads, this convergence takes longer. 
Unlike \gls{ESED}, the disagreement region under \gls{LASED} is continuously changing as parameters are updated. 
The decreasing rate of regret accumulation indicates that the disagreement region is visited less frequently over time.

These results show that even a small fixed estimation error can have a substantial impact on performance, while learning can mitigate this effect over time.

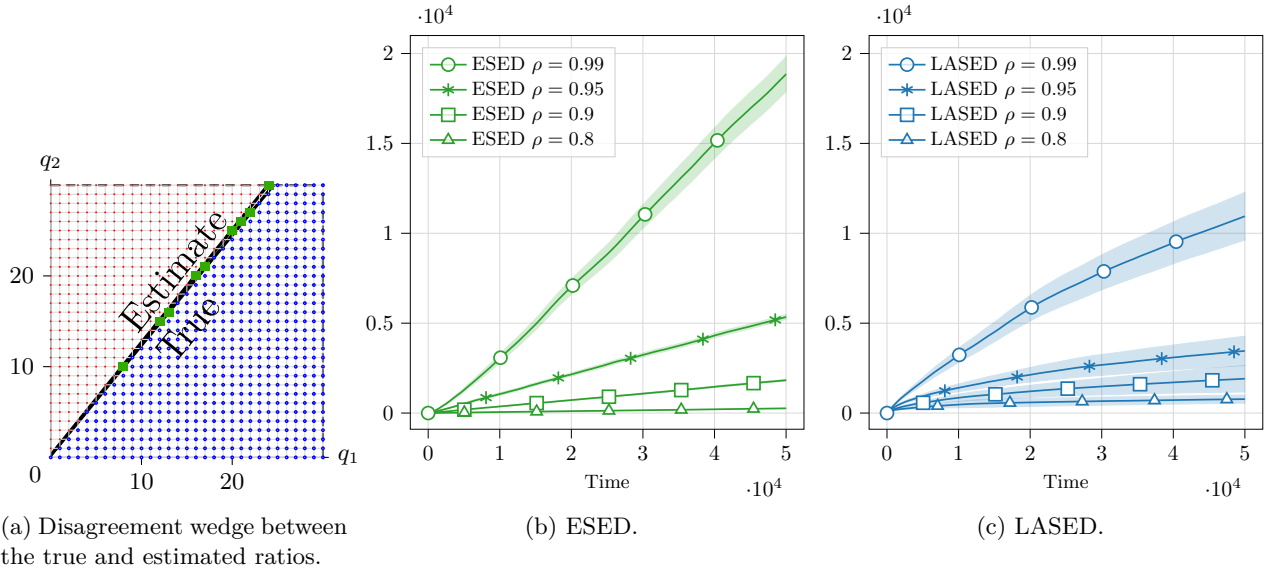
\begin{figure}[h]
	\centering
	\begin{subfigure}[t]{0.27\linewidth}
	\begin{tikzpicture}[scale=0.12, line cap=round, line join=round]
		\def\mone{30}
		\def\mtwo{30}
		\def\r{91/75} 
		\def\rn{92/74} 

		\draw  (0,0) node[below left] {\small 0} -- (\mone+0.6,0) node[ right] {\small $q_1$};
		\draw (0,0) -- (0,\mtwo+0.6) node[above] {\small $q_2$};

		\draw (10,0) -- (10,-0.8) node[below] {\small $10$};
		\draw (0,10) -- (-0.8,10) node[left] {\small $10$};
		\draw (20,0) -- (20,-0.8) node[below] {\small $20$};
		\draw (0,20) -- (-0.8,20) node[left] {\small $20$};

		\draw[dashed,thick] (\mone,0) -- (\mone,\mtwo);
		\draw[dashed,thick] (0,\mtwo) -- (\mone,\mtwo);

		\begin{scope}
			\clip (0,0) rectangle (\mone,\mtwo);
			\draw[very thick,black]
			plot[domain=0:\mone] (\x, {\r*(\x+1)-1});
		\end{scope}
		\begin{scope}
			\clip (0,0) rectangle (\mone,\mtwo);
			\draw[very thick,black]
			plot[domain=0:\mone] (\x, {\rn*(\x+1)-1});
		\end{scope}

		\draw[step=1,gray!30] (0,0) grid (\mone,\mtwo);

		\pgfmathsetmacro{\eps}{1e-2}

		\foreach \x in {0,...,\mone}{
				\foreach \y in {0,...,\mtwo}{
						\pgfmathsetmacro{\ylineA}{\r*(\x+1)-1}
						\pgfmathsetmacro{\ylineB}{\rn*(\x+1)-1}

						\pgfmathsetmacro{\ylow}{min(\ylineA,\ylineB)}
						\pgfmathsetmacro{\yhigh}{max(\ylineA,\ylineB)}

						\pgfmathsetmacro{\dA}{abs(\y-\ylineA)}
						\pgfmathsetmacro{\dB}{abs(\y-\ylineB)}

						\ifdim \dA pt < \eps pt
							\node[
								star,
								star points=5,
								draw=black,
								line width=0.8pt,
								fill=white,
								inner sep=1.5pt
							] at (\x,\y) {};
						\else\ifdim \dB pt < \eps pt
								\node[
									star,
									star points=5,
									draw=black,
									line width=0.8pt,
									fill=white,
									inner sep=1.5pt
								] at (\x,\y) {};

							\else\ifdim \y pt > \ylow pt
									\ifdim \y pt < \yhigh pt
										\node[text=darkgreen,fill=darkgreen,scale=0.25] at (\x,\y) {$\square$};

									\else
										\ifdim \y pt < \ylineA pt
											\draw[blue, line width=1pt] (\x,\y) circle (2pt);
										\else
											\fill[red] (\x,\y) circle (3pt);
										\fi
									\fi

								\else
									\ifdim \y pt < \ylineA pt
										\draw[blue, line width=1pt] (\x,\y) circle (2pt);
									\else
										\fill[red] (\x,\y) circle (3pt);
									\fi
								\fi\fi\fi
					}
			}

		\pgfmathsetmacro{\angle}{atan(\r)}
		\node[
			rotate=\angle,
			anchor=west,
			rounded corners=1pt,
			inner sep=2pt
		]  at (12,{\r*12-4}) {\Large True};

		\pgfmathsetmacro{\angle}{atan(\rn)}
		\node[
			rotate=\angle,
			anchor=west,
			rounded corners=1pt,
			inner sep=2pt
		]  at (8,{\rn*8 + 4}) {\Large Estimate};

	\end{tikzpicture}
	\caption{Disagreement wedge between the true and estimated ratios.}%
	\label{fig:ESED_wedge}
	\end{subfigure}
	\begin{subfigure}[t]{0.35\linewidth}
		\includegraphics[width=\linewidth]{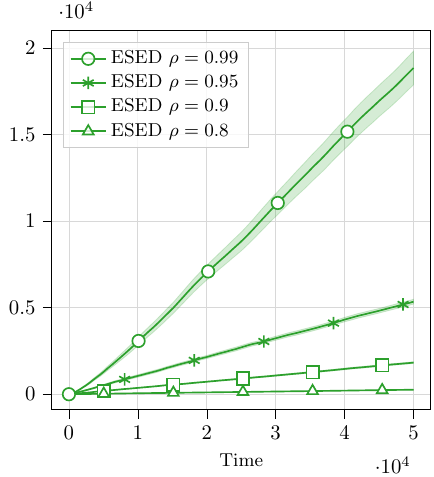}
		\caption{ESED.}
		\label{fig:regret_ESED}
	\end{subfigure}
	\begin{subfigure}[t]{0.35\linewidth}
		\includegraphics[width=\linewidth]{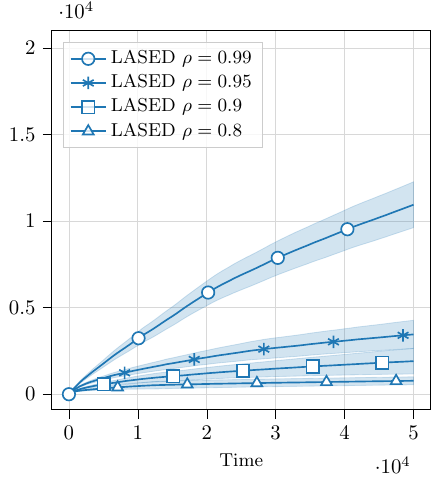} 
		\caption{LASED.}
		\label{fig:regret_LASED}
	\end{subfigure}
	\caption{Disagreement wedge and accumulated regret over time of \gls{ESED} and \gls{LASED}  in a two-server system for different traffic intensities. 
	The true service rates are $\mu_1 = 0.75$ and $\mu_2 = 0.91$, while the empirical estimators are initialized at $\hat{\mu}_1=0.74$ and $\hat{\mu}_2=0.92$. }
	\label{fig:regret_LASED_ESED}
\end{figure}

\subsection{When exploitation suffices}
\label{sec:num_res_exploration}
We next analyze the effect of exploration.
To this end, we compare \gls{LASED} with a Greedy policy that only ever exploits: routing decisions are based on empirical estimators that are updated after each service completion.

\Cref{fig:regret_size_loads} shows the average regret and number of customers in the system for different traffic intensities $\rho := \lambda/(\mu_1+\mu_2)$.
For moderate loads ($\rho \in\{0.3,0.6\}$) the regret of \gls{LASED} converges to a finite limit, see \Cref{fig:regret_loads}.
For $\rho=0.9$, the regret of \gls{LASED} takes longer to converge since the system has long busy periods and large queue lengths which delay learning.

Greedy accumulates less regret than \gls{LASED} and the difference increases with the load. 
Moreover, the number of customers in the system for Greedy is comparable or lower than for \gls{LASED}; see \Cref{fig:size_loads}.
In this model, the learning problem is structurally simple and so Greedy correctly identifies \gls{SED} without explicit exploration. 
Greedy also updates estimators after service completions, while \gls{LASED} only updates between episodes. 
Greedy's estimators therefore converge faster, leading to less regret. 

\begin{figure}[h]
	\centering
	\begin{subfigure}[t]{0.48\linewidth}
		\includegraphics[width=\linewidth]{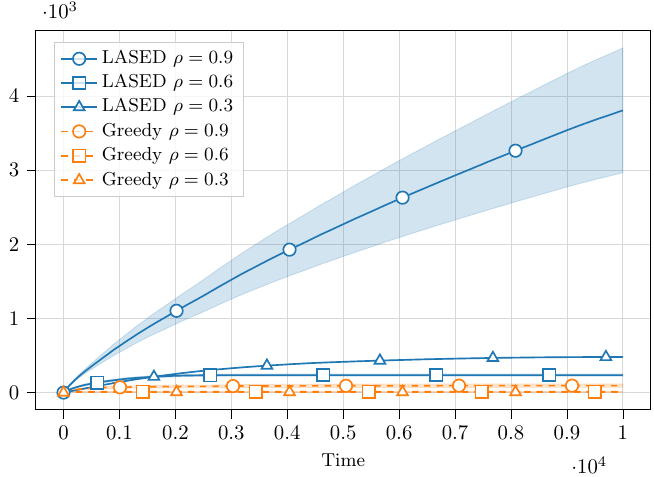} 
		\caption{Average regret.}
		\label{fig:regret_loads}
	\end{subfigure}
	\begin{subfigure}[t]{0.48\linewidth}
		\includegraphics[width=\linewidth]{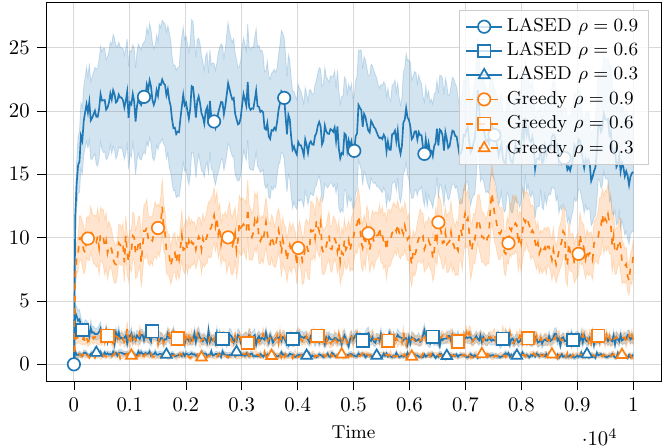}
		\caption{Number of customers.}
		\label{fig:size_loads}
	\end{subfigure}
	\caption{Performance of \gls{LASED} and a Greedy policy in a two-server system for different traffic intensities. The true service rates are $\mu_1 = \mu_2 = 1$, while the empirical estimators are initialized at $\hat{\mu}_1=1$ and $\hat{\mu}_2=10$. }
	\label{fig:regret_size_loads}
\end{figure}

To better understand the convergence of \gls{LASED}, \Cref{fig:episodes_loads} shows the cumulative count of completed episodes over time.
For $\rho\in\{0.3,0.6\}$, the curves become approximately linear over time, indicating that the episodes occur at an approximately constant long-run rate. 
The slope is lower for $\rho=0.3$ than $\rho=0.6$, which is consistent with the lower arrival intensity: fewer arrivals initiate busy periods over time. 

Under high load $\rho=0.9$, however, the curve in \Cref{fig:episodes_loads} remains convex rather than linear over the simulated time horizon. 
Its increasing slope indicates that episodes become shorter over time and that the episode rate has not yet stabilized.
The initially long episodes lead to large regret since estimators are only updated between episodes.

\begin{figure}[h]
	\centering
	\begin{minipage}[t]{0.7\linewidth}
		\centering
		\includegraphics[width=\linewidth]{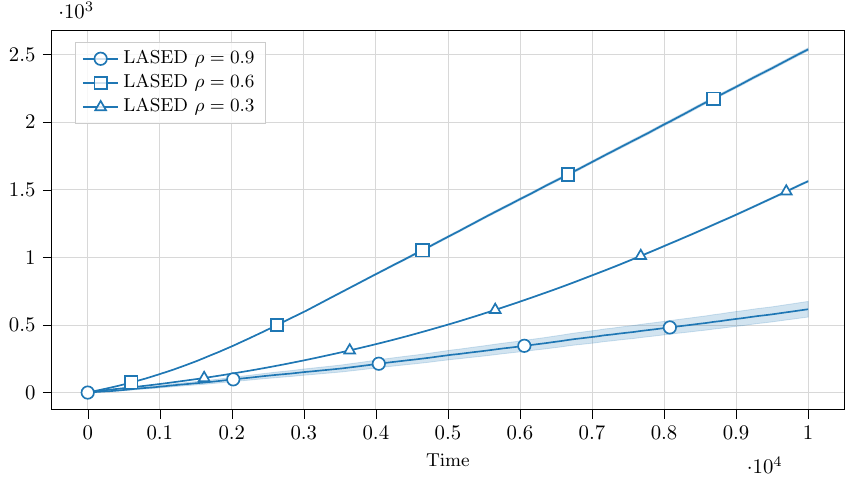}
		\captionof{figure}{Cumulative number of episodes of \gls{LASED} over time for the two-server system for different loads.}
		\label{fig:episodes_loads}
	\end{minipage}
	\hfill
	\begin{minipage}[t]{0.28\linewidth}
		\centering
		\includegraphics[width=\linewidth]{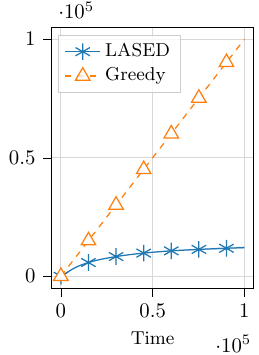}
		\caption{Average regret of \gls{LASED} and Greedy in a two-server system with poor initialization.}
		\label{fig:regret_bad_init}
	\end{minipage}
\end{figure}

\subsection{When exploration is beneficial}
We next illustrate an example where Greedy performs suboptimally due to poor initialization.
We let $\lambda=1$, $\mu_1=5$, and $\mu_2=10$.
The estimators are initialized at $\hat{\mu}_1 = 5$ and $\hat{\mu}_2=0.1$.
This means that the initial estimator $\hat{r}_1 = 0.02$ differs from the true $r^* = 2$ by a factor~$100$.
Due to the bad initialization, the first server is believed to be faster.
Since the load is low, the queue at the first server typically remains short, and Greedy therefore does not use the second server, which is actually faster. 
This leads to linear regret on this time horizon, as shown in \Cref{fig:regret_bad_init}.
On the other hand, \gls{LASED} explores both servers and correctly converges quickly to the \gls{SED} rule.

\subsection{The exploration--exploitation trade-off}
We have seen examples where Greedy collects sufficient samples and additional exploration is unnecessary.
On the other hand, we have seen that in some settings exploration helps to quickly recover the \gls{SED} rule, while Greedy remained stuck in a suboptimal regime. 
This leads to a central question: when is forced exploration beneficial? 

To investigate this, we run the \gls{LASED} and Greedy policies in different settings where we vary the traffic intensity $\rho$, the true ratio $r$, and the initial estimator ratio $\hat{r}_1$. 
Throughout, we take $\mu_1+\mu_2=2$ and $\lambda=(\mu_1+\mu_2)\rho$. 
\Cref{fig:regret_ratio} shows the difference between the average total accumulated regret of \gls{LASED} and Greedy.
A positive number indicates that \gls{LASED} accumulated more regret than Greedy, and a negative number indicates the opposite.
In the first panel, the imbalance between the true service rates is large, while in the second panel the servers are homogeneous.

We observe that \gls{LASED} outperforms Greedy in terms of regret when the servers are homogeneous, the load is low, and the initial estimation error is large. 
Indeed, the cells for $r=1$, $\rho=0.1$, and $\hat{r}_1=0.1,10$ in \Cref{fig:regret_ratio} have negative values.  
In these scenarios, Greedy only routes customers to the server that is believed to be the fastest.
Only when the queue length grows large, which does not happen often in the low-load scenario, Greedy sends a customer to the other server. 
Once a service time sample from that server is obtained, the initial error is corrected and then Greedy quickly recovers the correct balance. 
By contrast, \gls{LASED} enforces exploration and therefore obtains samples from both servers earlier, allowing it to recover the correct \gls{SED} policy more quickly.

Interestingly, the difference in regret does not look monotone in the initial estimation error, especially in the $r=1$ panel. 
Hence, accurate initialization does not uniformly imply the largest advantage for Greedy; its effect also depends on the load and the true service rate ratio.

In the first panel, \gls{LASED} suffers when the load is low.
In this case, the episodes are short since busy periods are brief. 
Since forced exploration is triggered by \gls{LASED} based on the number of samples, it often triggers exploration which leads to regret.
On the other hand, Greedy performs well since one server is truly much faster than the other, hence routing decisions are fairly robust: even with a poor estimate, the Greedy decision often agrees with the \gls{SED} decision. 
We observe that \gls{LASED} does improve as the load increases, since then exploration is triggered less often. 

In the second panel, the servers are homogeneous. 
In this case, \gls{LASED} suffers when the load is high and the initial error is non-zero, since estimators are only updated between episodes (busy periods), and busy periods take longer under high load. 
Greedy, on the other hand, updates after every service completion and can therefore adapt more quickly.

If we obtain samples from both servers continuously, then the empirical estimators will converge and only exploitation is sufficient.
In this case, Greedy will converge to the \gls{SED} policy quickly, with little accumulated regret. 
Exploration is thus beneficial only if Greedy does not obtain sufficient observations, which can happen when Greedy incorrectly starves one server.

\begin{figure}[h]
	\centering
	\includegraphics[width=\linewidth]{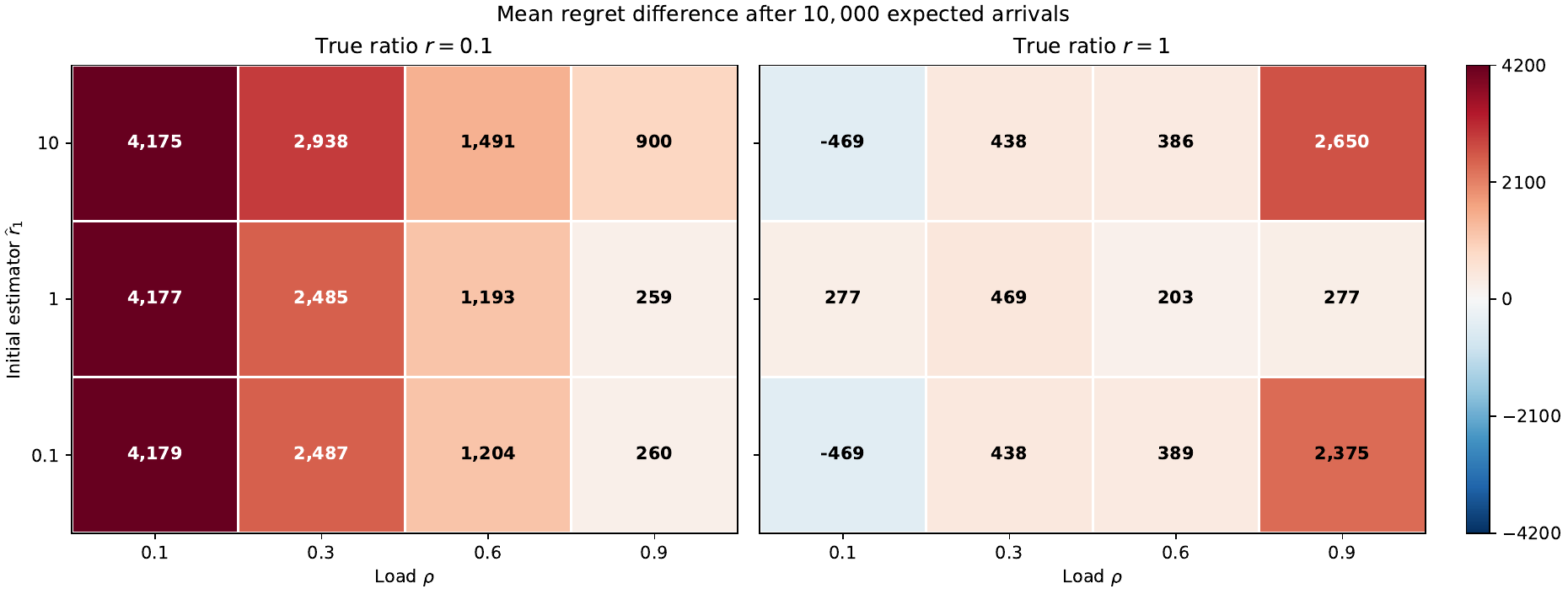}
	\caption{Difference in expected regret between the LASED and Greedy policies. Each cell is based on 50 independent simulations of both policies. Each simulation runs for the duration of $10^4$ expected arrivals, hence the value in each cell is in $[-10^4,10^4]$. }
	\label{fig:regret_ratio}
\end{figure}

\section{Conclusion}\label{sec:conclusion}

Suppose that you want to apply \gls{SED} to a two-server load balancing system with heterogeneous service rates, but unfortunately, the service rates are unknown to you.
The first idea you have is to estimate the service rates offline and then run \gls{SED}.
It turns out that this does not work for all arrival and service rates.
Your second idea is to run \gls{SED} while adjusting your estimates online.
This too, turns out, does not work for all system parameters.
The reason both ideas do not work is because of how \gls{SED} is defined: the slightest error in any one service rate creates a suboptimal decision region.
The relatively large size of this region is what causes linear regret accumulation.

To tackle this problem, we proposed \gls{LASED}.
This algorithm learns the service rates while applying \gls{SED}, and uses forced exploration and episodic updates to ensure finite regret.
For it, we provide a rigorous performance analysis that proves that the algorithm actually achieves asymptotically bounded regret.
The analysis works by decomposing the regret into three contributing factors: exploration errors, estimation errors, and overflow events.

In short, we conclude that in queueing systems with sensitive, state-dependent decision boundaries, standard learning algorithms are not enough.
Approaches that balance exploration and exploitation while routing in queues can stop regret from growing.

\section*{Acknowledgments}

This work is part of \emph{Valuable AI}, a research collaboration between the Eindhoven University of Technology and the Koninklijke KPN N.V.
Parts of this research have been funded by the EAISI's IMPULS program, and by Holland High Tech~\textbar~TKI HTSM via the PPS allowance scheme for public-private partnerships.

\addcontentsline{toc}{section}{References}
\printbibliography

\appendix
\crefalias{section}{appendix}

\section{Remaining proof details}
\subsection{Proof of \Cref{thm:regret_ub}}
\label{app:proof_thm_regret_ub}

\begin{proof}
	Consider~\eqref{eq:E_regret_split}.
	We bound the terms individually.

		{\it Exploration.}
	By \Cref{lem:exploration}, for any $k\in\N_{\geq 1}$ sufficiently large, with $v^*$ in~\eqref{eq:def_vr},
	\begin{align}
		\P(\Eexplr{k})
		 & \leq
		2\exp(-(v^*)^2 k/64) + 2^{10} \exp(17) (k-1)^{-3}
		.
		\label{eq:P_exp_v_star}
	\end{align}
	Since the exponential term in~\eqref{eq:P_exp_v_star} decays faster than the polynomial term, there exists $\tilde{c}_1>0$ such that for $k$ sufficiently large
	\begin{align}
		\P(\Eexplr{k})
		 & \leq
		\tilde{c}_1 (k-1)^{-3}
		.
		\label{eq:P_expl_bnd_proof}
	\end{align}

	{\it Misestimation.}
	By \Cref{lem:est_out_of_range}, 
	\begin{align}
		\label{eq:P_est_error_proof}
		\P(\Eest{k})
		 & \leq
		\P(|\hat{r}_k - r^*| \geq \uval{r^*}{m(k)})
		.
	\end{align}
	We obtain from \Cref{lem:est_error} that there exists $\tilde{c}_2 > 0$ such that for $k\in\N_{\geq 1}$ sufficiently large, 
	\begin{align}
		\P(|\hat{r}_k - r^*| \geq \uval{r^*}{m(k)})
		 & \leq
		\tilde{c}_2 k^{-3}
		.
		\label{eq:P_est_error_bnd_proof}
	\end{align}
	Here, we used that $k^{-\ln(k)} = o(k^{-p})$ for any $p > 0$. 

	{\it Overflow.}
	By the law of total probability for any $k\in\N_{\geq 1}$ sufficiently large
	\begin{align}
		\label{eq:overflow_LTP}
		\P(\Eovf{k})
		 & \leq
		\P(\Eexplr{k}) + \P(\Eovf{k} \mid (\Eexplr{k})^c)
		.
	\end{align}
	By \Cref{lem:overflow}, there exists $\tilde{c}_3>0$ such that
	\begin{align}
		\P(\Eovf{k} \mid (\Eexplr{k})^c)
		 & \leq
		p_*^{-1} \exp(-\theta_* (m(k)-\xi_*))                        \\
		 & \leq
		p_*^{-1} \exp(\theta_* (\xi_*-1)) \exp(-\theta_* c_1 \ln(k)) \\
		 & =
		p_*^{-1} \exp(\theta_* (\xi_*-1)) k^{-\theta_* c_1}          \\
		 & =
		\tilde{c}_3 k^{-\theta_* c_1}
		,
		\label{eq:P_ovf_bnd_proof}
	\end{align}
	where $\theta_* := \inf_{r \in [\rmin, \rmax]} \theta_r \sqrt{(r \wedge 1)/2}$, with $\theta_r$ defined in \eqref{eq:exp_martingale_theta}.
	We have $\theta_* \in (0,\infty)$. 
	Indeed, the interval $[\rmin,\rmax]$ is compact and bounded away from zero, and $a_r$ in~\eqref{eq:def_a_const} is uniformly positive on this interval since $\delta_{\min}>0$ by~\eqref{eq:stab_cond}.

	{\it Arrivals.}
	By construction, the estimator $\bar{r}_k$ in \gls{LASED} is always in $[\rmin,\rmax]$. 
	Moreover, $\beta_r(7,7)$ and $\gamma_r(7,7)$ are uniformly on that interval. 
	Therefore, by \Cref{lem:bnd_arrivals_episode} there exist $\tilde{c}_4 > 0$, depending on the system parameters $\lambda,\mu_1,\mu_2$, and algorithm parameters $\mumin$ and $\mumax$, such that for any episode $k$,
	\begin{align}
		\E(\Arrepi{k}^2) \leq \tilde{c}_4
		.
	\end{align}

	We next combine the bounds above in \eqref{eq:E_regret_split}. 
	There exists $k^* \in \N_{\geq 1}$ such that \eqref{eq:P_expl_bnd_proof}, \eqref{eq:P_est_error_bnd_proof}, and~\eqref{eq:P_ovf_bnd_proof} hold for all $k\geq k^*$.
	For the first $k^*-1$ episodes, we bound all probabilities in~\eqref{eq:E_regret_split} by 1. 
	This gives
	\begin{align}
		\E(R_{r^*}(n))
		 & \leq
		3 k^* \sqrt{\tilde{c}_4}
		+
		\sum_{k=k^*+1}^{n}
		\sqrt{
			3 \tilde{c}_4
			\Bigl(
			2\tilde{c}_1 (k-1)^{-3}
			+
			\tilde{c}_2 k^{-3}
			+
			\tilde{c}_3 k^{-\theta_* c_1}
			\Bigr)
		}
		.
		\label{eq:regret_bnd_sum}
	\end{align}
	The sum in~\eqref{eq:regret_bnd_sum} converges as $n\to\infty$ if we let $c_1 > 0$ such that
	$
		c_1 > 2/\theta_*
		.
	$
	This completes the proof.

\end{proof}

\subsection{Proof of~\Cref{lem:lyap_subgeom_drift}}
\label{app:proof_lem_lyap_subgeom_drift}

\begin{proof}
	Substituting~\eqref{eq:def_lyap_func} into~\eqref{eq:def_drift_operator}, we find
	\begin{align}
		(\calL_r\phi_r)(\bfq)
		 & =
		\lambda \ind{q_2+1 \geq r(q_1 + 1)} r(q_1+1)
		+
		\lambda \ind{q_2+1 < r(q_1 + 1)}  (q_2+1) \\
		 & \phantom{=}
		-
		\mu_1 \ind{q_1 > 0} rq_1
		-
		\mu_2 \ind{q_2 > 0} q_2 \nonumber
		.
	\end{align}
	We next partition the state space into two regions and analyze them separately.

		{\bf Region (I).}
	Consider the region where $q_2+1 \geq r(q_1+1)$.
	Since $r,q_i \geq 0$ and $\ind{q_i > 0}q_i = q_i$, $i\in\{1,2\}$, we have
	\begin{align}
		(\calL_r\phi_r)(\bfq)
		 & =
		r q_1 (\lambda - \mu_1) -  q_2 \mu_2 + r \lambda
		\leq
		r q_1 (\lambda - \mu_1)^+ -  q_2 \mu_2 + r \lambda
		.
	\end{align}
	In this region $rq_1 \leq q_2 + 1-r$.
	Then
	\begin{align}
		\label{eq:bnd_q2_delta1}
		(\calL_r\phi_r)(\bfq)
		 & \leq
		r q_1 (\lambda - \mu_1)^+ -  q_2 \mu_2 + r \lambda
		\leq
		q_2 ((\lambda - \mu_1)^+ - \mu_2)
		+
		(1-r)(\lambda - \mu_1)^+
		+
		r \lambda
		.
	\end{align}
	Using the definition of $\delta_1$ in~\eqref{eq:def_delta1} and $-q_2 \leq -(rq_1+r-1)$ we obtain for this region
	\begin{align}
		\label{eq:ineq_q2_delta1}
		-q_2 \delta_1
		=
		- \frac{q_2 \delta_1}{2} - \frac{q_2 \delta_1}{2}
		\leq
		-\frac{q_2 \delta_1}{2} - \frac{(rq_1 + r-1) \delta_1}{2}
		,
	\end{align}
	where the inequality uses that $\delta_1 > 0$ by~\eqref{eq:stab_cond}.
	Substitution of~\eqref{eq:ineq_q2_delta1} into~\eqref{eq:bnd_q2_delta1} gives
	\begin{align}
		q_2 ((\lambda - \mu_1)^+ - \mu_2)
		+
		(1-r)(\lambda - \mu_1)^+
		+
		r \lambda
		 & \leq
		-(rq_1+q_2) \frac{\delta_1}{2}
		+
		(1-r)((\lambda - \mu_1)^+ + \frac{\delta_1}{2})
		+
		r \lambda \\
		 & =
		-(rq_1+q_2) \frac{\delta_1}{2}
		+
		\frac{1-r}{2}(\mu_2 + (\lambda - \mu_1)^+)
		+
		r \lambda
		.
		\label{eq:bnd_drift_region1}
	\end{align}

	{\bf Region (II).}
	For the other region, we have $q_2 < r(q_1+1)-1$.
	We find similarly
	\begin{align}
		- r q_1 \mu_1 + q_2 (\lambda - \mu_2)^+ + \lambda
		 & \leq
		r q_1 ((\lambda - \mu_2)^+ - \mu_1)
		+
		(r-1) (\lambda - \mu_2)^+
		+
		\lambda \\
		 & \leq
		-(rq_1 + q_2) \frac{\delta_2}{2} + \frac{r-1}{2} (\mu_1 + (\lambda-\mu_2)^+) + \lambda
		.
		\label{eq:bnd_drift_region2}
	\end{align}

	By~\eqref{eq:bnd_drift_region1},~\eqref{eq:bnd_drift_region2} and the definitions of $\deltamin$ in~\eqref{eq:def_deltamin} and $b_r$ in~\eqref{eq:def_b_const},
	\begin{align}
		(\calL_r\phi_r)(\bfq)
		 & \leq
		-(rq_1 + q_2) \frac{\deltamin}{2}
		+
		\max
		\Bigl(
		\tfrac{1-r}{2}(\mu_2 + (\lambda - \mu_1)^+) + r \lambda ,
		\tfrac{r-1}{2} (\mu_1 + (\lambda-\mu_2)^+) + \lambda
		\Bigr)                                                 \\
		 & \leq
		-(rq_1 + q_2) \frac{\deltamin}{2}
		+
		|r-1| ((\mu_1 \vee \mu_2)+\lambda) + \lambda(r \vee 1) \\
		 & \leq
		-(rq_1 + q_2) \frac{\deltamin}{2}
		+
		2(r \vee 1)((\mu_1 \vee \mu_2)+\lambda)                \\
		 & =
		-(rq_1 + q_2) \frac{\deltamin}{2}
		+
		b_r
		.
	\end{align}
	We next write $1/2 = 1/4+1/4$ and use that $(rq_1+q_2)\deltamin/4 \geq b_r$ for $\bfq\notin \scrC_r$  in~\eqref{eq:def_C_set} to obtain
	\begin{align}
		(\calL_r\phi_r)(\bfq)
		 & \leq
		-(rq_1 + q_2) \frac{\deltamin}{4} + b_r \ind{\bfq \in \scrC_r}
		.
		\label{eq:drift_delta_ind}
	\end{align}

	Note that for all $q_1,q_2 \in \N_{\geq 0}$,
	\begin{align}
		rq_1 + q_2
		 & \geq
		\sqrt{(r \wedge 1)} (\sqrt{r} q_1 + q_2)             \\
		 & \stackrel{\textrm{(i)}}{\geq}
		\sqrt{(r \wedge 1)} \sqrt{r q_1^2 + q_2^2}           \\
		 & \stackrel{\textrm{(ii)}}{\geq}
		\sqrt{(r \wedge 1)}
		\sqrt{r \frac{q_1(q_1+1)}{2} + \frac{q_2(q_2+1)}{2}} \\
		 & =
		\sqrt{(r \wedge 1)} \sqrt{\phi_r(\bfq)}
		.
		\label{eq:phi_sum_geq_srtq}
	\end{align}
	Here, (i) follows since $x+y\geq \sqrt{x^2 + y^2}$ for any $x,y\geq 0$, and (ii) follows since $x^2 \geq x(x+1)/2$ for any $x\in\N_{\geq 0}$.
	The result follows from~\eqref{eq:drift_delta_ind} and~\eqref{eq:phi_sum_geq_srtq}.
	This completes the proof.
\end{proof}

\subsection{A version of Dynkin's formula}
\label{app:dynkin}

In preparation of the proof of \Cref{lem:hitting_time_finite_moments}, we present a version of Dynkin's formula.

\begin{lemma}
	\label{lem:dynkin}

	Let $X(t)$ for $t\geq 0$ be a \gls{BD} process with state space $\calS$, infinitesimal generator $\calL$, and transition rates $\kappa(\bfx,\bfy)$ for $\bfx,\bfy\in\calS$.
	Let $T^*$ be a stopping time for $X$ and let $f:\calS\to\R_{\geq 0}$.
	Suppose that there exists $c < \infty$ such that
	$
		\sum_{\bfy \neq \bfx} \kappa(\bfx,\bfy) \leq c
	$
	for all $\bfx\in\calS$, and that
	\begin{align}
		\label{eq:dynkin_integrability_cond}
		\E_{\bfx}\Bigl( \int_0^{t\wedge T^*} \sum_{\bfy\neq X(s)} \kappa(X(s),\bfy) \bigl| f(\bfy)-f(X(s)) \bigr| \rmd s \Bigr) < \infty, \qquad \bfx\in\calS, \ t \geq 0
		.
	\end{align}
	Then
	\begin{align}
		\label{eq:stopped_dynkin_bd_process}
		\E_{\bfx}\bigl(f(X(t\wedge T^*))\bigr)
		 & =
		f(\bfx)
		+
		\E_{\bfx}\Bigl( \int_0^{t\wedge T^*} (\calL f)(X(s)) \rmd s \Bigr), \qquad \bfx\in\calS, \ t \geq 0
		.
	\end{align}
\end{lemma}

\begin{proof}
	Let $f^n(\bfx) := f(\bfx) \wedge n$ for $n\in\N$.
	Since $f^n$ is bounded and the transition rates are uniformly bounded,  $\calL f^n$ is bounded.
	Moreover, $X$ is a strong Markov process since it is a \gls{BD} process.
	We may thus apply Dynkin's formula~\citep[(2.9)]{Kushner1967} on the bounded stopping time $t\wedge T^*$:
	\begin{equation}
		\E_{\bfx}\bigl(f^n(X(t\wedge T^*))\bigr)
		=
		f^n(\bfx)
		+
		\E_{\bfx}\Bigl( \int_0^{t\wedge T^*} (\calL f^n)(X(s)) \rmd s \Bigr), \qquad \bfx\in\calS, \ t \geq 0, \ n\in\N
		.
		\label{eq:dynkin_truncated}
	\end{equation}
	We are now going to study the limit $n\to\infty$ of each term.

	For the left term, note that
	$
		f^n(X(t\wedge T^*)) \uparrow f(X(t\wedge T^*))
	$
	almost surely because $f^n(\bfx) \uparrow f(\bfx)$ for every $\bfx$.
	Because $f^n$ is non-negative, the monotone convergence theorem gives
	$
		\lim_{n\to\infty} \E_{\bfx}\bigl(f^n(X(t\wedge T^*))\bigr)
		= \E_{\bfx}\bigl(f(X(t\wedge T^*))\bigr)
		.
	$

	The middle term simply converges to $f(\bfx)$ as $n \to \infty$.

	To study the limit of the right-most term, we are going to apply the dominated convergence theorem.
	For this, we will show that $\calL f^n$ converges pointwise and that $|\calL f^n|$ is upper bounded by an integrable function.
	To conclude the pointwise convergence, note that we may interchange the limits in the following equation because the rate sum is finite by assumption and boundedness of $f^n$:
	\begin{align}
		\lim_{n\to\infty} (\calL f^n)(\bfx)
		 & =
		\lim_{n\to\infty} \Bigl(\sum_{\bfy\neq \bfx} \kappa(\bfx,\bfy) (f^n(\bfy) - f^n(\bfx))\Bigr)
		=
		\sum_{\bfy\neq \bfx} \Bigl(\lim_{n\to\infty} \kappa(\bfx,\bfy) (f^n(\bfy) - f^n(\bfx))\Bigr)
		=
		(\calL f)(\bfx)
		.
	\end{align}
	To conclude the integrability, note that $|(a\wedge n) - (b\wedge n)| \leq |a-b|$ for any $a$, $b\geq 0$ and all $n\in\N$.
	Using this and the triangle inequality gives
	\begin{align}
		\ind{\{s\leq t\wedge T^*\}} |(\calL f^n)(X(s))|
		\leq
		\ind{\{s\leq t\wedge T^*\}} \sum_{\bfy\neq X(s)} \kappa(X(s),\bfy) |f(\bfy)-f(X(s))|
		,
	\end{align}
	the right-hand side of which is integrable by assumption~\eqref{eq:dynkin_integrability_cond}.
	Thus, the dominated convergence theorem gives
	\begin{align}
		\lim_{n\to\infty}
		\E_{\bfx}\Bigl( \int_0^{t\wedge T^*} (\calL f^n)(X(s)) \rmd s \Bigr)
		=
		\E_{\bfx}\Bigl( \int_0^{t\wedge T^*} (\calL f)(X(s))\rmd s \Bigr)
		.
	\end{align}
	This completed the proof.
\end{proof}

\subsection{Proof of~\Cref{lem:hitting_time_finite_moments}}
\label{app:proof_lem_hitting_time_finite_moments}

\begin{proof}
	Let $r\in\R_{>0}$ and $\bfq\in\N_{\geq 0}^2$.
	If $\bfq\in\scrC_r$, then $\E_{\bfq}(T_{\scrC_r}) = \E_{\bfq}(T_{\scrC_r}^2) = 0$ so the claim holds.
	We next discuss the case $\bfq\notin\scrC_r$.

		{\bf First moment.}
	Let $t\geq 0$.
	Note that $\QSED{r}(\cdot)$ is a \gls{BD} process with uniformly bounded jump rates and that $T_{\scrC_r}$ is a stopping time for $\QSED{r}$.
	Moreover, by the definition of $\phi_r$ in~\eqref{eq:def_lyap_func}, we have
	$
		|\sqrt{\phi_r(\bfq')}-\sqrt{\phi_r(\bfq)}| \leq (1 \vee \sqrt{r})
	$
	for all neighbors $\bfq'\in\N_{\geq0}^2$.
	Hence, the conditions of \Cref{lem:dynkin} are satisfied, and we thus have
	\begin{align}
		\label{eq:dynkin}
		\E_\bfq\Bigl( \sqrt{\phi_r(\QSED{r}(t\wedge T_{\scrC_r}))} \Bigr)
		 & =
		\sqrt{\phi_r(\bfq)} +
		\E_\bfq\Bigl( \int_0^{t\wedge T_{\scrC_r}} (\calL_r \sqrt{\phi_r})(\QSED{r}(s)) \rmd s\Bigr)
		.
	\end{align}
	We next analyze the generator of $\sqrt{\phi_r}$.
	Note that
	\begin{align}
		\label{eq:sqrt_xy_property}
		\sqrt{y} - \sqrt{x} \leq \frac{y-x}{2\sqrt{x}}
		\quad
		\textnormal{for}
		\quad
		x,y > 0,
	\end{align}
	because this is equivalent to ${(\sqrt{x}-\sqrt{y})}^2 \geq 0$ for $x, y > 0$.
	Applying~\eqref{eq:sqrt_xy_property} to $\sqrt{\phi_r}$ (which is allowed because $\phi_r > 0$; see~\eqref{eq:def_lyap_func}), and then using \Cref{lem:lyap_subgeom_drift} gives
	\begin{align}
		(\calL_r \sqrt{\phi_r})(\bfq)
		 & \leq
		\frac{1}{2\sqrt{\phi_r(\bfq)}} (\calL_r \phi_r)(\bfq)
		\leq
		\frac{1}{2\sqrt{\phi_r(\bfq)}}
		(-a_r \sqrt{\phi_r(\bfq)})
		=
		-\frac{a_r}{2}
		\label{eq:sqrt_drift_a}
	\end{align}
	for $\bfq\notin\scrC_r$.
	Applying~\eqref{eq:sqrt_drift_a} in~\eqref{eq:dynkin} gives
	\begin{align}
		\label{eq:bnd_sqrt_drift_split}
		\E_\bfq\Bigl( \sqrt{\phi_r(\QSED{r}(t\wedge T_{\scrC_r}))} \Bigr)
		\leq
		\sqrt{\phi_r(\bfq)}
		-
		\frac {a_r}{2}
		\E_\bfq(t\wedge T_{\scrC_r})
		.
	\end{align}
	Since $\phi_r\geq 0$, the left side in~\eqref{eq:bnd_sqrt_drift_split} is non-negative.
	Letting $t\to\infty$ and invoking the monotone convergence theorem gives
	\begin{align}
		\label{eq:bnd_E_T}
		\E_{\bfq}(T_{\scrC_r}) \leq 2a_r^{-1} \sqrt{\phi_r(\bfq)}
		.
	\end{align}
	This concludes the bound on the first moment of the hitting time.

		{\bf Second moment.}\\
	Let us first rewrite $\E_\bfq(T_{\scrC_r}^2)$.
	By first using the fundamental theorem of calculus, and then Tonelli's theorem~\citep[Thm.~2.37a]{Folland1999} to move the expectation inside (which is allowed because the integrand is non-negative), we find that
	\begin{align}
		\E_\bfq(T_{\scrC_r}^2)
		 &
		=
		2 \E_\bfq\Bigl(\int_0^{T_{\scrC_r}} \bigl(T_{\scrC_r} - s\bigr) \rmd s\Bigr)
		=
		2 \int_0^\infty
		\E_\bfq
		\Bigl( \bigl(T_{\scrC_r} - s\bigr) \ind{s<T_{\scrC_r}} \Bigr)
		\rmd s
		.
		\label{eq:T_squared_integral}
	\end{align}
	Let now $\calF_s = \sigma(\QSED{r}(u): 0 \leq u \leq s)$ denote the natural filtration of $\QSED{r}(s)$ for $s\geq 0$.
	By first applying the law of total expectation with respect to this filtration and leveraging that $\QSED{r}(t)$ satisfies the strong Markov property, and next invoking Fubini's theorem~\citep[Thm~2.37b]{Folland1999}, we obtain that
	\begin{equation}
		\begin{aligned}
			\E_\bfq(T_{\scrC_r}^2)
			=
			2
			\int_0^{\infty}
			\E_\bfq\bigl(\ind{s<T_{\scrC_r}}\E_{\QSED{r}(s)}(T_{\scrC_r}) \bigr)  \rmd s
			 &
			=
			2
			\E_\bfq\Bigl(\int_0^{T_{\scrC_r}} \E_{\QSED{r}(s)}(T_{\scrC_r}) \rmd s \Bigr)
			.
		\end{aligned}
		\label{eq:Fubini2}
	\end{equation}

	Let us next bound $\E_\bfq(T_{\scrC_r}^2)$.
	Together with~\eqref{eq:bnd_E_T}, \eqref{eq:Fubini2} implies that
	\begin{equation}
		\E_\bfq(T_{\scrC_r}^2)
		\leq
		4 a_r^{-1} \E_{\bfq}\Bigl(\int_0^{T_{\scrC_r}} \sqrt{\phi_r(\QSED{r}(s))} \rmd s \Bigr)
		;
		\label{eq:T_squared_split2}
	\end{equation}
	just substitute $\bfq$ for $\QSED{r}(s)$ in~\eqref{eq:bnd_E_T}.
	Let $t\geq 0$.
	Similar to the first moment, we apply \Cref{lem:dynkin}.
	Now, however, the one-step increments of $\phi_r$ are not uniformly bounded.
	Nonetheless: because
	$|\phi_r(\bfq') - \phi_r(\bfq)| \leq (1\vee r)(1+q_1+q_2)$
	and $\QSED{r}$ has unit jump sizes and uniformly bounded jump rates, the map $s\mapsto \E_{\bfq}(Q_1(s) + Q_2(s))$ is bounded for $s\in[0,t]$.
	Hence, condition~\eqref{eq:dynkin_integrability_cond} is satisfied for $\phi_r$.
	We may therefore apply \Cref{lem:dynkin}, giving
	\begin{equation}
		\begin{aligned}
			0
			\leq
			\E_\bfq\bigl(\phi_r(\QSED{r}(t\wedge T_{\scrC_r}))\Bigr)
			 &
			=
			\phi_r(\bfq) +
			\E_\bfq\Bigl( \int_0^{t\wedge T_{\scrC_r}} (\calL_r \phi_r)(\QSED{r}(s)) \rmd s\Bigr)
			\\
			 &
			\leq
			\phi_r(\bfq)
			-
			a_r
			\E_\bfq\Bigl( \int_0^{t\wedge T_{\scrC_r}} \sqrt{\phi_r(\QSED{r}(s))} \rmd s\Bigr)
		\end{aligned}
		\label{eq:int_hit_t_infty}
	\end{equation}
	for $\bfq\notin\scrC_r$.
	Thus, letting $t\to\infty$, one finds the upper bound
	\begin{align}
		\label{eq:dynkin_phi}
		\E_\bfq\Bigl(\int_0^{T_{\scrC_r}} \sqrt{\phi_r(\QSED{r}(s))} \rmd s\Bigr)
		 & \leq
		a_r^{-1} \phi_r(\bfq)
	\end{align}
	for $\bfq\notin\scrC_r$.
	This is exactly what we still had to bound in \eqref{eq:T_squared_split2}, and completes the proof.
\end{proof}

\subsection{Proof of \Cref{lem:martingale}}
\label{app:proof_lem_martingale}

\begin{proof}
	Let $r\in\R_{>0}$, $m\in\N_{\geq 1}$, and $n\in\N_{\geq 0}$.
	In this section we drop the subscript $r$ and write $T = T_{\scrC\cup\scrM}$.

	Using linearity and that $W^T\ind{n\geq T}$ is $\calF^n$-measurable since $T$ is a stopping time, we obtain
	\begin{align}
		\label{eq:W_martingale_split}
		\E(W^{(n+1) \wedge T} \mid \calF^n)
		 & =
		W^T \ind{n\geq T} + \E(W^{n+1} \ind{n < T}  \mid \calF^n)
		.
	\end{align}
	To prove the lemma, it suffices to show that the last term is bounded by $W^n \ind{n < T}$, which we do next.
	To this end, define the stepwise difference
	\begin{align}
		\label{eq:def_step_diff_W}
		\Delta^n := \sqrt{\phi(\QSEDdiscrete{}{n+1})} - \sqrt{\phi(\QSEDdiscrete{}{n})}
		.
	\end{align}
	Then, by definition of $W^n$, since it is $\calF^n$-measurable, we find
	\begin{align}
		\label{eq:W_exp}
		\E(W^{n+1} \ind{n < T} \mid \calF^n)
		 & =
		W^n \E(\exp(\theta \Delta^n) \ind{n < T} \mid \calF^n)
		.
	\end{align}

	To bound the exponential term, we use the following.
	By Taylor's Theorem with Lagrange remainder (see, e.g., \citep[Thm.~5.15]{Rudin1976}) we have that for any $y\in\R$ there exists  $\xi \in ((0 \wedge y),(0\vee y))$ such that
	\begin{align}
		\label{eq:exp_bnd}
		\exp(\theta y)
		 & \leq 1+ \theta y + \frac 12 \theta^2 \exp(\theta \xi) y^2
		\leq 1+ \theta y + \frac 12 \theta^2 \exp(\theta |y|) y^2
		.
	\end{align}
	Hence,  using also linearity and that $\ind{n<T}$ is $\calF^n$-measurable,
	\begin{equation}
		\begin{aligned}
			\E(\exp(\theta \Delta^n) \ind{n < T} \mid \calF^n)
			\leq
			 &
			\ind{n < T}
			+
			\theta \E(\Delta^n \ind{n < T} \mid \calF^n)
			+
			\frac 12 \theta^2
			\E
			\bigl(
			\exp(\theta |\Delta^n|) (\Delta^n)^2 \ind{n < T} \mid \calF^n
			\bigr)
			.
		\end{aligned}
		\label{eq:W_exp_taylor}
	\end{equation}

	We next analyze the last two terms in~\eqref{eq:W_exp_taylor}.
	Using the elementary square root inequality in~\eqref{eq:sqrt_xy_property} and that $\QSEDdiscrete{}{n}$ is $\calF^n$-measurable gives
	\begin{align}
		\E(\Delta^n \ind{n < T} \mid \calF^n)
		 & \leq
		\frac{\E\bigl((\phi(\QSEDdiscrete{}{n+1}) - \phi_r(\QSEDdiscrete{}{n}))\ind{n < T} \mid \calF^n\bigr)}{2\sqrt{\phi_r(\QSEDdiscrete{}{n})}}
		.
		\label{eq:E_Delta_n_smaller_T}
	\end{align}
	The one-step drift of the discretized process is governed by the drift of the continuous time process $\{\QSED{}(t)\}$ scaled by the total jump rate, see~\eqref{eq:Q_discrete_drift_op}.
	On $n < T$, the process did not yet hit $\scrC_r$ and therefore the drift is strictly negative by \Cref{lem:lyap_subgeom_drift}.
	Moreover, the total jump rate in any state is bounded by $\lambda+\mu_1+\mu_2$.
	This gives
	\begin{equation}
		\begin{aligned}
			\E\bigl((\phi(\QSEDdiscrete{}{n+1}) - \phi_r(\QSEDdiscrete{}{n}))\ind{n < T} \mid \calF^n\bigr)
			 &
			=
			\frac{(\calL\phi)(\QSEDdiscrete{}{n})}{\nu(\QSEDdiscrete{}{n})} \ind{n < T}
			\\
			 &
			\leq
			\frac{- a_r \sqrt{\phi(\QSEDdiscrete{}{n})}}{\nu(\QSEDdiscrete{}{n})} \ind{n < T}
			\leq
			\frac{- a_r \sqrt{\phi(\QSEDdiscrete{}{n})}}{\lambda+\mu_1+\mu_2} \ind{n < T}
			.
		\end{aligned}
		\label{eq:Q_n_diff_L}
	\end{equation}
	Here, we used that $a_r >0$ since~\eqref{eq:stab_cond} holds by assumption.
	Combining~\eqref{eq:E_Delta_n_smaller_T} and \eqref{eq:Q_n_diff_L}, the square root cancels and we find
	\begin{align}
		\E(\Delta^n \ind{n < T} \mid \calF^n)
		 & \leq
		\frac{- a_r}{2(\lambda+\mu_1+\mu_2)} \ind{n < T}
		.
		\label{eq:V_negative_drift}
	\end{align}

	We next analyze the last term in~\eqref{eq:W_exp_taylor}.
	The key point is that the increments of $\sqrt{\phi}$ are uniformly bounded.
	In particular, we show that $|\Delta^n| \leq \sqrt{2(r \vee 1)}$ by considering all possible jumps of the \gls{BD} process $\QSEDdiscrete{}{n}$.
	If no jump occurs, then $\Delta^n = 0$ clearly.
	For a birth $(q_1,q_2)\mapsto(q_1+1,q_2)$, the contribution of the second coordinate can only decrease the difference of the square roots.
	Concretely, $\sqrt{a+c}-\sqrt{b+c} \leq \sqrt a - \sqrt b$ for $a \geq b \geq 0,\, c\geq 0$, hence
	\begin{align}
		|\Delta^n|
		 & \leq
		\sqrt{\frac{r}{2}} \Bigl(\sqrt{(q_1+1)(q_1+2)} - \sqrt{q_1(q_1+1)}\Bigr)
		\leq
		\sqrt{\frac{r}{2}} ((q_1+2) - q_1)
		=
		\sqrt{2 r}
		.
		\label{eq:birth_delta_bnd}
	\end{align}
	Thus a birth in the first coordinate satisfies $|\Delta^n|\leq \sqrt{2r}$.
	The same argument applies to a death in the first coordinate.
	Similarly, it can be verified that jumps in the second coordinate satisfy $|\Delta^n|\leq \sqrt{2}$.
	This proves $|\Delta^n| \leq \sqrt{2(r \vee 1)}$.

	We use this to bound the last term in~\eqref{eq:W_exp_taylor},
	\begin{align}
		\exp(\theta |\Delta^n|) (\Delta^n)^2
		 & \leq
		2(r \vee 1) \exp\bigl(\theta \sqrt{2(r \vee 1)}\bigr)
		\stackrel{\eqref{eq:exp_martingale_theta}}{\leq}  4(r \vee 1)
		.
		\label{eq:diff_V_leq_L}
	\end{align}
	By~\eqref{eq:W_exp_taylor},~\eqref{eq:V_negative_drift}, and~\eqref{eq:diff_V_leq_L} we have
	\begin{align}
		\E(W^{n+1} \ind{n < T}  \mid \calF^n)
		 & \leq
		W^{n}
		\Bigl(
		1 - \theta \frac{a_r}{2(\lambda+\mu_1+\mu_2)}
		+
		2\theta^2 (r \vee 1)
		\Bigr)
		\ind{n < T}
		\stackrel{\eqref{eq:exp_martingale_theta}}{\leq}
		W^{n} \ind{n < T}
		.
		\label{eq:W_before_tau}
	\end{align}
	Hence, the proof is completed by~\eqref{eq:W_martingale_split} and~\eqref{eq:W_before_tau}.
\end{proof}

\subsection{Proof of \Cref{lem:hitting_probability}}
\label{app:proof_lem_hitting_probability}

\begin{proof}
	Let $r\in\R_{>0}$, $m\in\N_{\geq 1}$, and $\bfy\in\N_{\geq 0}^2\setminus (\scrC_r\cup\scrM_m)$.
	In this section we drop the subscripts $r$ and $m$, and write $T = T_{\scrC\cup\scrM}$.

	We consider the discrete-time jump chain $\QSEDdiscrete{}{n}$ since
	\begin{align}
		\label{eq:hitting_P_cont_discr}
		\P_{\bfy}(\QSED{}(T) \in \scrM)
		 & =
		\P_\bfy(\QSEDdiscrete{}{T} \in \scrM)
		.
	\end{align}
	Note $\P_\bfy(\QSEDdiscrete{}{T} \in \scrC \cup\scrM) = 1$ by definition of $T$, hence by the law of total probability,
	\begin{align}
		\label{eq:W_T_split}
		\E_\bfy(W^T)
		 & =
		\underbrace{\E_\bfy(W^T \mid \QSEDdiscrete{}{T} \in \scrC)\P_\bfy(\QSEDdiscrete{}{T} \in \scrC)}_{\geq 0} +
		\E_\bfy(W^T \mid \QSEDdiscrete{}{T} \in \scrM)\P_\bfy(\QSEDdiscrete{}{T} \in \scrM)
		.
	\end{align}
	We next bound $\E_\bfy(W^T \mid \QSEDdiscrete{}{T} \in \scrM) $ and $\E_\bfy(W^T)$.

	$\QSEDdiscrete{}{T} \in \scrM$ implies $\max(Q_1^{T},Q_2^{T}) \geq m+1$ by definition of $\scrM$, and therefore $\phi(\QSEDdiscrete{}{T}) \geq (r \wedge 1)m^2/2$ by~\eqref{eq:def_lyap_func}.
	Hence, by~\eqref{eq:def_W_martingale},
	\begin{align}
		\label{eq:W_T_bnd}
		\E_\bfy(W^T \mid \QSEDdiscrete{}{T} \in \scrM)
		 & \geq
		\exp\bigl(\theta m \sqrt{(r \wedge 1)/2}\bigr)
		.
	\end{align}

	To bound $\E_\bfy(W^T)$, we will apply the following version of Doob's \gls{OST} \citep[Thm.~2.28]{vanderHofstad2016}:
	\emph{%
		Let $(V^n)_{n\geq 0}$ be a supermartingale and $S$ a stopping time, both with respect to $(\calF^n)_{n\geq 0}$.
		If $V^{n\wedge S}$ is bounded, then $\E(V^S) \leq \E(V^0)$.
	}%

	We first show that $W^{n\wedge T}$ satisfies the conditions of the \gls{OST}.
	By \Cref{lem:martingale}, $W^{n\wedge T}$ is a supermartingale.
	If $\QSEDdiscrete{}{0} \notin \scrM$, then
	$
		Q_1^{n\wedge T}, Q_2^{n\wedge T} \leq m+1
		,
	$
	since in that case $\QSEDdiscrete{}{n}$ can only enter $\scrM$ via a boundary state $\bfx$ with $\max(x_1,x_2) = m+1$.
	In that case, $\phi(\QSEDdiscrete{}{n\wedge T})\leq (r \vee 1) (m+1)(m+2)$ by definition of $\phi$ in~\eqref{eq:def_lyap_func}, and so $W^{n\wedge T} \leq \exp(\theta \sqrt{(r \vee 1) (m+1)(m+2)})$ is bounded.

	We may now apply the \gls{OST} since~$\bfy\notin\scrM$,
	\begin{align}
		\E_\bfy(W^T)
		 & \stackrel{\eqref{eq:def_W_martingale}}{=}
		\E_\bfy\Bigl(\exp(\theta \sqrt{\phi(\QSEDdiscrete{}{T})})\Bigr)
		\stackrel{\textrm{OST}}{\leq}
		\E_\bfy(\exp(\theta \sqrt{\phi(\QSEDdiscrete{}{0})}))
		=
		\exp(\theta \sqrt{\phi(\bfy)})
		.
		\label{eq:OST_W}
	\end{align}

	By~\eqref{eq:W_T_split}, \eqref{eq:W_T_bnd}, and~\eqref{eq:OST_W},
	\begin{align}
		\label{eq:P_Q_hit_M_exp}
		\P_\bfy(\QSEDdiscrete{}{T} \in \scrM)
		 & \leq
		\exp\bigl(-\theta \bigl(m \sqrt{(r \wedge 1)/2} - \sqrt{\phi(\bfy)}\bigr)\bigr)
		.
	\end{align}
	The proof is concluded from~\eqref{eq:hitting_P_cont_discr} and~\eqref{eq:P_Q_hit_M_exp}.
\end{proof}

\subsection{Proof of \Cref{lem:overflow_bp}}
\label{app:proof_lem_overflow_bp}

\begin{proof}
	In this section we drop the subscripts $r$ and $m$, as well as the superscript $\textrm{SED}$.
	Recall the definitions of $\scrM$ in~\eqref{eq:def_set_M} and $T_{\bf{0}} := T_{\{\mathbf{0}\}}$.
	We have
	\begin{align}
		\P_{\mathbf{0}}
		\bigl(
		\max_{0 \leq t\leq T_{\mathbf{0}}}
		\bigl(
		Q_1(t) \vee Q_2(t)
		\bigr)
		\leq m
		\bigr)
		 & =
		\P_{\mathbf{0}}(\bfQ(T_{\mathbf{0}\cup\scrM}) = \mathbf{0})
		\geq
		\inf_{\bfx \in \scrC} \P_{\bfx}(\bfQ(T_{\mathbf{0}\cup\scrM}) = \mathbf{0})
		,
		\label{eq:large_ql_hit_0M}
	\end{align}
	where the last inequality follows since $\mathbf{0}\in\scrC$; recall~\eqref{eq:def_C_set}.
	We next lower bound the right hand side.

	Let $\bfx \in\scrC$.
	By the law of total probability,
	\begin{align}
		\P_{\bfx}(\bfQ(T_{\mathbf{0}\cup\scrM}) = \mathbf{0})
		 & =
		\underbrace{\P_{\bfx}(\bfQ(T_{\mathbf{0}\cup\scrM}) = \mathbf{0} \mid \bfQ(T_{\mathbf{0}\cup\scrB}) = \mathbf{0})}_{:=\textrm{(I)}}
		\P_{\bfx}(\bfQ(T_{\mathbf{0}\cup\scrB}) = \mathbf{0}) \notag \\
		 & \ +
		\underbrace{\P_{\bfx}(\bfQ(T_{\mathbf{0}\cup\scrM}) = \mathbf{0} \mid \bfQ(T_{\mathbf{0}\cup\scrB}) \in\scrB)}_{:=\textrm{(II)}}
		\P_{\bfx}(\bfQ(T_{\mathbf{0}\cup\scrB}) \in\scrB)
		.
		\label{eq:P_A_hit_0}
	\end{align}
	Note that $\textrm{(I)} = 1$ since the process starting from $\scrC$ cannot reach $\scrM$ without entering $\scrB$; see \Cref{fig:plane_sets}.

	We next bound $\textrm{(II)}$.
	By the law of total probability,
	\begin{align}
		\label{eq:hit_identity_inf_B}
		\textrm{(II)}
		 & \geq
		\inf_{\bfb\in\scrB} \P_{\bfx}(\bfQ(T_{\mathbf{0}\cup\scrM}) = \mathbf{0} \mid \bfQ(T_{\mathbf{0}\cup\scrB}) = \bfb)
		.
	\end{align}
	Once the process hits $\scrB$, the future evolution depends only on the state at the hitting time, and the past may be discarded.
	More precisely, $\bfQ(t)$ is a strong Markov process since it is a \gls{BD} process.
	Hence, by the strong Markov property~\citep[Def.~17.12]{Klenke2020} applied at the stopping time $T_{\mathbf{0}\cup\scrB}$
	\begin{align}
		\P_{\bfx}(\bfQ(T_{\mathbf{0}\cup\scrB} + t)\in \scrA \mid \calF_{T_{\mathbf{0}\cup\scrB}})
		 & =
		\P_{\bfQ(T_{\mathbf{0}\cup\scrB})}(\bfQ(t) \in \scrA),
		\qquad \scrA \subseteq \N_{\geq 0}^2, \ t \geq 0
		,
	\end{align}
	where $\calF_s = \sigma(\bfQ(u): 0 \leq u \leq s)$ denotes the natural filtration of $\QSED{r}(s)$ for $s\geq 0$.
	Letting $\tilde{\bfb} \in \scrB$, we thus obtain
	\begin{align}
		\label{eq:hit_identity_x_b}
		\P_\bfx(\bfQ(T_{\mathbf{0}\cup\scrM}) = \mathbf{0} \mid \bfQ(T_{\mathbf{0}\cup\scrB}) = \tilde{\bfb})
		 & =
		\P_{\tilde{\bfb}}(\bfQ(T_{\mathbf{0}\cup\scrM}) = \mathbf{0})
		.
	\end{align}
	By the law of total probability,
	\begin{align}
		\textrm{RHS \eqref{eq:hit_identity_x_b}}
		 & =
		\underbrace{\P_{\tilde{\bfb}}(\bfQ(T_{\mathbf{0}\cup\scrM}) = \mathbf{0} \mid \bfQ(T_{\scrC\cup\scrM}) \in \scrC)}_{:= \textrm{(III)}}
		\P_{\tilde{\bfb}}(\bfQ(T_{\scrC\cup\scrM}) \in \scrC) \\
		 & \ +
		\underbrace{\P_{\tilde{\bfb}}(\bfQ(T_{\mathbf{0}\cup\scrM}) = \mathbf{0} \mid \bfQ(T_{\scrC\cup\scrM}) \in \scrM)}_{:= \textrm{(IV)}}
		\P_{\tilde{\bfb}}(\bfQ(T_{\scrC\cup\scrM}) \in \scrM)
		.
		\label{eq:P_B_hit_0}
	\end{align}
	Note that $\textrm{(IV)} = 0$. Moreover, by the law of total probability,
	\begin{align}
		\label{eq:hit_identity_inf_C}
		\textrm{(III)}
		 & \geq
		\inf_{\bfy\in\scrC} \P_{\tilde{\bfb}}(\bfQ(T_{\mathbf{0}\cup\scrM}) = \mathbf{0} \mid \bfQ(T_{\scrC\cup\scrM}) = \bfy)
		.
	\end{align}
	Similar to~\eqref{eq:hit_identity_x_b}, we find by the strong Markov property for $\tilde{\bfy}\in\scrC$,
	\begin{align}
		\label{eq:hit_identity_b_y}
		\P_{\tilde{\bfb}}(\bfQ(T_{\mathbf{0}\cup\scrM}) = \mathbf{0} \mid \bfQ(T_{\scrC\cup\scrM}) = \tilde{\bfy})
		 & =
		\P_{\tilde{\bfy}}(\bfQ(T_{\mathbf{0}\cup\scrM}) = \mathbf{0})
		.
	\end{align}
	Thus, \eqref{eq:hit_identity_inf_C} and~\eqref{eq:hit_identity_b_y} give
	$
		\label{eq:hit_identity_inf_y}
		\textrm{(III)}
		\geq
		\inf_{\bfy\in\scrC} \P_\bfy(\bfQ(T_{\mathbf{0}\cup\scrM}) = \mathbf{0})
		,
	$
	which means by \eqref{eq:P_B_hit_0},
	\begin{align}
		\textrm{RHS \eqref{eq:hit_identity_x_b}}
		 & \geq
		\inf_{\bfy\in\scrC} \bigl\{\P_\bfy(\bfQ(T_{\mathbf{0}\cup\scrM}) = \mathbf{0}) \bigr\}
		\P_{\tilde{\bfb}}(\bfQ(T_{\scrC\cup\scrM}) \in \scrC)
		.
		\label{eq:hit_identity_b_hit_0}
	\end{align}

	Combining~\eqref{eq:hit_identity_inf_B}, \eqref{eq:hit_identity_x_b}, and~\eqref{eq:hit_identity_b_hit_0} yields
	\begin{align}
		\textrm{(II)}
		 & \geq
		\inf_{\bfy\in\scrC}
		\bigl\{ \P_\bfy(\bfQ(T_{\mathbf{0}\cup\scrM}) = \mathbf{0}) \bigr\}
		\inf_{\bfb\in\scrB}
		\bigl\{ \P_\bfb(\bfQ(T_{\scrC\cup\scrM})\in\scrC) \bigr\}
		.
		\label{eq:hit_identity_x_OM}
	\end{align}

	Combining~\eqref{eq:P_A_hit_0} and~\eqref{eq:hit_identity_x_OM}, and using that
	$
		\P_{\bfx}(\bfQ(T_{\mathbf{0}\cup\scrB})\in\scrB)
		=
		1-\P_{\bfx}(\bfQ(T_{\mathbf{0}\cup\scrB})=\mathbf{0})
	$
	gives
	\begin{align}
		\P_{\bfx}(\bfQ(T_{\mathbf{0}\cup\scrM}) = \mathbf{0})
		 & \geq
		\P_{\bfx}(\bfQ(T_{\mathbf{0}\cup\scrB}) = \mathbf{0})
		+
		\inf_{\bfy\in\scrC}
		\bigl\{ \P_\bfy(\bfQ(T_{\mathbf{0}\cup\scrM}) = \mathbf{0}) \bigr\} \notag
		\\
		 & \
		\cdot
		\inf_{\bfb\in\scrB}
		\bigl\{ \P_\bfb(\bfQ(T_{\scrC\cup\scrM})\in\scrC) \bigr\}
		(1-\P_{\bfx}(\bfQ(T_{\mathbf{0}\cup\scrB})=\mathbf{0}))
		.
		\label{eq:hitting_fixed_point}
	\end{align}
	The right side is increasing in $\P_{\bfx}(\bfQ(T_{\mathbf{0}\cup\scrB})=\mathbf{0})$, since both infima are in $(0,1)$.
	Taking the infimum over $\bfx\in\scrC$ on both sides in~\eqref{eq:hitting_fixed_point} and rearranging gives

	\begin{align}
		\inf_{\bfx\in\scrC}
		\P_{\bfx}(\bfQ(T_{\mathbf{0}\cup\scrM}) = \mathbf{0})
		 & \geq
		\frac{\inf_{\bfx\in\scrC} \P_{\bfx}(\bfQ(T_{\mathbf{0}\cup\scrB})=\mathbf{0})}{1 - \inf_{\bfb\in\scrB}
			\bigl\{ \P_\bfb(\bfQ(T_{\scrC\cup\scrM})\in\scrC) \bigr\} ( 1  - \inf_{\bfx\in\scrC} \P_{\bfx}(\bfQ(T_{\mathbf{0}\cup\scrB})=\mathbf{0}))}
		.
		\label{eq:hitting_fixed_point_sol2}
	\end{align}
	\eqref{eq:hitting_fixed_point_sol2} is of the form $a / (1-b(1-a))$ which is increasing in $a$ and $b$ for $a,b\in (0,1)$.
	We thus require lower bounds for the two infima.
	Since $\scrB$ does not intersect with $\scrC \cup \scrM$ and $\N_{\geq 0}^2 \setminus\{\scrC\cup\scrM\}$ is finite, we obtain from \Cref{lem:hitting_probability} and~\eqref{eq:bnd_phi_on_B},
	\begin{align}
		\inf_{\bfb\in\scrB}
		\P_\bfb(\bfQ(T_{\scrC\cup\scrM}) \in \scrC)
		 & \geq
		1-\exp\bigl(-\theta \bigl(m \sqrt{(r \wedge 1)/2} - \sqrt{\xi}\bigr)\bigr)
		=:
		1 - \ell
		.
		\label{eq:hit_AM_from_B_bnd}
	\end{align}
	Moreover, as discussed in \Cref{sec:tail_bnd_QL},
	$
		\inf_{\bfx\in\scrC} \P_\bfx(\bfQ(T_{\mathbf{0}\cup\scrB}) = \mathbf{0}) \geq p
		,
	$
	with $p$ in~\eqref{eq:bnd_empty_before_arr}.
	Hence,
	\begin{align}
		\label{eq:hit_0M_fromA_bnd}
		\inf_{\bfx\in\scrC}
		\P_{\bfx}(\bfQ(T_{\mathbf{0}\cup\scrM}) = \mathbf{0})
		 & \geq
		\frac{p}{1-(1-\ell)(1-p)}
		=
		\frac{1}{1+\frac{1-p}{p}\ell}
		\geq
		1-\frac{1-p}{p}\ell
		.
	\end{align}
	Here, we used $1/(1+x) \geq 1- x$ for all $ x \geq 0$.

	The proof is completed by~\eqref{eq:large_ql_hit_0M} and~\eqref{eq:hit_0M_fromA_bnd}.
\end{proof}

\subsection{Proof of \Cref{lem:bnd_hit_0B}}
\label{app:proof_lem_bnd_hit_0B}

\begin{proof}
	Let $r\in\R_{>0}$, $\bfy\in\scrC_r\setminus \{\mathbf{0}\}$, and suppose that $\QSED{r}(0) = \bfy$.
	If, in $[0,\ell_r]$, there are no arrivals \emph{and} all $y_1+y_2$ customers depart, then $T_{\mathbf{0}\cup\scrB_r}\leq \ell_r$.
	Thus,
	\begin{align}
		\label{eq:P_B_leq_t0}
		\P_\bfy(T_{\mathbf{0}\cup\scrB_r} \leq \ell_r)
		 & \geq
		\P_\bfy(M(\ell_r) = 0, \ D(\ell_r) \geq y_1+y_2)
		=
		\P_\bfy(D(\ell_r) \geq y_1+y_2 \mid M(\ell_r) = 0) e^{-\lambda \ell_r}
		,
	\end{align}
	where $M(t)\sim\poi{\lambda t}$ denotes the number of arrivals and $D(t)$ the number of departures in $[0,t]$.
	As long as there are customers in the system, there is at least one active server.
	Hence, the service rate is at least $(\mu_1 \wedge \mu_2)$ until the system empties.
	Therefore,
	\begin{align}
		\P_\bfy(D(\ell_r) \geq y_1+y_2 \mid M(\ell_r) = 0)
		 & \geq
		\P(Y \geq y_1+y_2),
		\qquad
		Y\sim\poi{(\mu_1 \wedge \mu_2)\ell_r}
		.
	\end{align}
	By definition of $\scrC_r$ in~\eqref{eq:def_C_set} and by~\eqref{eq:def_ell_r},
	$
		y_1+y_2
		\leq
		((1/r) \vee 1) \kappa_r
		\leq
		(\mu_1 \wedge \mu_2) \ell_r
		.
	$
	Therefore,
	\begin{align}
		\label{eq:p_bnd_t0}
		\P(Y \geq y_1+y_2)
		 & \geq
		\P(Y \geq (\mu_1 \wedge \mu_2) \ell_r)
		\geq
		\tfrac 12
		,
	\end{align}
	where the last inequality uses $\P(X\geq x) \geq 1/2$ for $X\sim\poi{x}$ and $x\in\N_{\geq 0}$.
	Since $\bfy\in\scrC_r$ was arbitrary, we obtain from~\eqref{eq:P_B_leq_t0} up to~\eqref{eq:p_bnd_t0} that
	\begin{align}
		\label{eq:bnd_B_t0}
		\max_{\bfy\in\scrC_r} \P_\bfy(T_{\mathbf{0}\cup\scrB_r} > \ell_r)
		 & \leq
		1 - \tfrac 12 e^{-\lambda \ell_r}
		.
	\end{align}

	We next show that
	\begin{align}
		\label{eq:bnd_B_n_t0}
		\max_{\bfy\in\scrC_r} \P_\bfy(T_{\mathbf{0}\cup\scrB_r} > n \ell_r)
		 & \leq
		\bigl(1 - \tfrac 12 e^{-\lambda \ell_r}\bigr)^n, \qquad n\in\N_{\geq 1}
		.
	\end{align}
	Let  $n\in\N_{\geq 1}$, $\calA := \{T_{\mathbf{0}\cup\scrB_r} > (n-1)\ell_r\}$, and $\calB := \{ \QSED{r}((n-1) \ell_r) \in\scrC_r\setminus\{\mathbf{0}\}\}$.
	We have
	\begin{align}
		\P_\bfy(T_{\mathbf{0}\cup\scrB_r} > n \ell_r)
		 & =
		\P_\bfy(T_{\mathbf{0}\cup\scrB_r} > n \ell_r \mid \calA)
		\P_\bfy(\calA)
		.
		\label{eq:P_exit_C_n}
	\end{align}
	Note
	\begin{align}
		\P_\bfy(T_{\mathbf{0}\cup\scrB_r} > n \ell_r \bigm| \calA)
		=
		\P_\bfy(T_{\mathbf{0}\cup\scrB_r} > n \ell_r
		\bigm|
		\calA, \,
		\calB)
		,
		\label{eq:recursion_start}
	\end{align}
	since $\bfy \in\scrC_r$ and $\scrB_r$ is the boundary of $\scrC_r$.
	By the law of total probability and the definition of $\calB$,
	\begin{align}
		 & \P_\bfy(T_{\mathbf{0}\cup\scrB_r} > n \ell_r
		\bigm|
		\calA, \,
		\calB) \notag                                 \\
		 & \ =
		\sum_{\bfx \in \scrC_r\setminus\{\mathbf{0}\}}
		\P_\bfy(T_{\mathbf{0}\cup\scrB_r} > n \ell_r
		\bigm|
		\calA, \,
		\calB, \,
		\QSED{r}((n-1)\ell_r) =\bfx )
		\P_\bfy(\QSED{r}((n-1)\ell_r) =\bfx
		\bigm|
		\calA, \,
		\calB)
		.
		\label{eq:recursion_mid1}
	\end{align}
	By the Markov property,
	\begin{align}
		\P_\bfy(T_{\mathbf{0}\cup\scrB_r} > n \ell_r
		\bigm|
		\calA, \,
		\calB, \,
		\QSED{r}((n-1)\ell_r) =\bfx )
		 & =
		\P_\bfx(T_{\mathbf{0}\cup\scrB_r} > \ell_r)
		.
		\label{eq:recursion_mid2}
	\end{align}
	Hence, by~\eqref{eq:recursion_mid1} and~\eqref{eq:recursion_mid2},
	\begin{align}
		\P_\bfy(T_{\mathbf{0}\cup\scrB_r} > n \ell_r
		\bigm|
		\calA, \,
		\calB)
		 & =
		\sum_{\bfx \in \scrC_r\setminus\{\mathbf{0}\}}
		\P_\bfx(T_{\mathbf{0}\cup\scrB_r} > \ell_r)
		\P_\bfy(\QSED{r}((n-1)\ell_r) =\bfx
		\bigm|
		\calA, \,
		\calB)  \\
		 & \leq
		\max_{\bfz\in\scrC_r} \P_\bfz(T_{\mathbf{0}\cup\scrB_r} > \ell_r)
		\sum_{\bfx \in \scrC_r\setminus\{\mathbf{0}\}}  \P_\bfy(\QSED{r}((n-1)\ell_r) =\bfx
		\bigm|
		\calA, \,
		\calB)  \\
		 & =
		\max_{\bfz\in\scrC_r} \P_\bfz(T_{\mathbf{0}\cup\scrB_r} > \ell_r)
		.
		\label{eq:recursion_n}
	\end{align}
	By~\eqref{eq:recursion_start} and \eqref{eq:recursion_n},
	\begin{align}
		\P_\bfy(T_{\mathbf{0}\cup\scrB_r} > n \ell_r
		\mid
		\calA)
		 & \leq
		\max_{\bfz\in\scrC_r} \P_\bfz(T_{\mathbf{0}\cup\scrB_r} > \ell_r)
		.
		\label{eq:recursion_exit_C}
	\end{align}
	Iterating~\eqref{eq:P_exit_C_n} and using~\eqref{eq:bnd_B_t0} and~\eqref{eq:recursion_exit_C}  gives~\eqref{eq:bnd_B_n_t0}.

	For the first moment, we use that $\E(X) = \int_0^\infty \P(X>t)\rmd t$ for any $X\geq 0$~\citep[Eq.~(21.9)]{Billingsley1995},
	\begin{align}
		\max_{\bfy \in \scrC_r} \E_\bfy(T_{\mathbf{0}\cup\scrB_r})
		 & =
		\max_{\bfy\in\scrC_r} \int_0^\infty \P_\bfy(T_{\mathbf{0}\cup\scrB_r} > t)\rmd t
		\leq
		\int_0^\infty \max_{\bfy\in\scrC_r} \P_\bfy(T_{\mathbf{0}\cup\scrB_r} > t)\rmd t
		.
	\end{align}
	If $t\in[n\ell_r,(n+1)\ell_r]$, then $\P(T_{\mathbf{0}\cup\scrB_r} > t) \leq \P(T_{\mathbf{0}\cup\scrB_r} > n\ell_r)$, hence
	\begin{align}
		\int_0^\infty \max_{\bfy\in\scrC_r} \P_\bfy(T_{\mathbf{0}\cup\scrB_r} > t)\rmd t
		 & \leq
		\sum_{n=0}^\infty
		\int_{n\ell_r}^{(n+1)\ell_r}
		\max_{\bfy\in\scrC_r} \P_\bfy(T_{\mathbf{0}\cup\scrB_r} > n\ell_r)\rmd t
		\leq
		\ell_r \sum_{n=0}^\infty \max_{\bfy \in\scrC_r} \P_\bfy(T_{\mathbf{0}\cup\scrB_r} > n\ell_r)
		.
		\label{eq:tailbound_M_C}
	\end{align}
	Using~\eqref{eq:bnd_B_n_t0} and that $\sum_{n=0}^\infty (1-x)^n = 1/x$ for $x\in (0,1)$, we obtain
	\begin{align}
		\label{eq:M_bnd_t0_p}
		\max_{\bfy \in \scrC_r} \E_\bfy(T_{\mathbf{0}\cup\scrB_r})
		 & \leq
		\ell_r \sum_{n=0}^\infty
		\Bigl(1 - \frac 12 e^{-\lambda \ell_r} \Bigr)^n
		=
		2 \ell_r e^{\lambda \ell_r}
		.
	\end{align}

	For the second moment, we use that $\E(X^2) = 2 \int_0^\infty t \P(X>t)\rmd t$ for any $X\geq 0$, which can be verified by applying~\citep[Eq.~(21.9)]{Billingsley1995} to $X^2$.
	This gives
	\begin{align}
		\max_{\bfy \in \scrC_r} \E_\bfy(T_{\mathbf{0}\cup\scrB_r}^2)
		 & \leq
		2 \int_0^\infty t \max_{\bfy\in\scrC_r} \P_\bfy(T_{\mathbf{0}\cup\scrB_r} > t)\rmd t                  \\
		 & \leq
		2\ell_r^2 \sum_{n=0}^\infty  (n+1)\max_{\bfy \in\scrC_r} \P_\bfy(T_{\mathbf{0}\cup\scrB_r} > n\ell_r) \\
		 & \leq
		2\ell_r^2 \sum_{n=0}^\infty (n+1) \bigl(1 - \tfrac 12 e^{-\lambda \ell_r}\bigr)^n                   \\
		 & \leq
		8 \ell_r^2 e^{2\lambda \ell_r}
		.
	\end{align}
	Here, the last inequality uses $\sum_{n=0}^\infty (n+1)(1-x)^n = 1/x^2$ for $x\in (0,1)$.
	This concludes the proof.
\end{proof}

\subsection{Proof of \Cref{lem:E_B_bnd}}
\label{app:proof_lem_E_B_bnd}

\begin{proof}
	Let $\bfq\in\N_{\geq 0}^2$.
	We have $T_{\mathbf{0}} \geq T_{\scrC_r}$ since $\mathbf{0} \in \scrC_r$; recall~\eqref{eq:def_C_set}.
	Therefore, we may write $\E_\bfq(T_{\mathbf{0}}) = \E_\bfq(T_{\scrC_r}) + \E_\bfq(T_{\mathbf{0}} - T_{\scrC_r})$.
	By the law of total expectation,
	\begin{align}
		\label{eq:strong_markov_hitting1}
		\E_{\bfq}(T_{\mathbf{0}}-T_{\scrC_r})
		 & =
		\sum_{\bfz \in\scrC_r}
		\E_{\bfq}(T_{\mathbf{0}}-T_{\scrC_r} \mid \QSED{r}(T_{\scrC_r}) = \bfz)
		\P_{\bfq}(\QSED{r}(T_{\scrC_r}) = \bfz)
		.
	\end{align}
	Applying the strong Markov property~\citep[Def.~17.12]{Klenke2020} at stopping time $T_{\scrC_r}$ gives $\E_{\bfq}(T_{\mathbf{0}}-T_{\scrC_r} \mid \QSED{r}(T_{\scrC_r}) = \bfz)  = \E_{\bfz}(T_{\mathbf{0}})$.
	Hence,
	\begin{align}
		\label{eq:strong_markov_hitting2}
		\E_{\bfq}(T_{\mathbf{0}}-T_{\scrC_r})
		 & =
		\sum_{\bfz \in\scrC_r}
		\E_{\bfz}(T_{\mathbf{0}})
		\P_{\bfq}(\QSED{r}(T_{\scrC_r}) = \bfz)
		\leq
		\max_{\bfy \in\scrC_r} \E_{\bfy}(T_{\mathbf{0}}) \sum_{\bfz \in\scrC_r} \P_{\bfq}(\QSED{r}(T_{\scrC_r}) = \bfz)
		=
		\max_{\bfy \in\scrC_r} \E_{\bfy}(T_{\mathbf{0}})
		.
	\end{align}

	By~\eqref{eq:strong_markov_hitting2} and \Cref{lem:hitting_time_finite_moments},
	\begin{align}
		\label{eq:bnd_B_split}
		\E_\bfq(T_{\mathbf{0}})
		 & \leq
		\E_\bfq(T_{\scrC_r}) + \max_{\bfy \in \scrC_r} \E_\bfy(T_{\mathbf{0}})
		\leq
		2a_r^{-1} \sqrt{\phi_r(\bfq)} + \max_{\bfy \in \scrC_r} \E_\bfy(T_{\mathbf{0}})
		.
	\end{align}

	For the second moment, we use
	$(x+y)^2 \leq 2x^2 + 2y^2$ for any $x,y\geq 0$ and \Cref{lem:hitting_time_finite_moments} to obtain
	\begin{align}
		\E_\bfq(T_{\mathbf{0}}^2)
		 & \leq
		2 \E_\bfq(T_{\scrC_r}^2) + 2\E_\bfq((T_{\mathbf{0}} - T_{\scrC_r})^2)
		\leq
		8 a_r^{-2} \phi_r(\bfq) + 2\E_\bfq((T_{\mathbf{0}} - T_{\scrC_r})^2)
		.
	\end{align}
	Following similar steps as in~\eqref{eq:strong_markov_hitting1} and~\eqref{eq:strong_markov_hitting2} with, by the strong Markov property, $\E_{\bfq}((T_{\mathbf{0}}-T_{\scrC_r})^2 \mid \QSED{r}(T_{\scrC_r}) = \bfz)  = \E_{\bfz}(T_{\mathbf{0}}^2)$ for any $\bfz \in\scrC_r$, we obtain
	\begin{align}
		\label{eq:bnd_B_split_squared}
		\E_\bfq(T_{\mathbf{0}}^2)
		 & \leq
		8 a_r^{-2} \phi_r(\bfq) + 2\E_\bfq((T_{\mathbf{0}} - T_{\scrC_r})^2)
		\leq
		8 a_r^{-2} \phi_r(\bfq) + 2 \max_{\bfy \in \scrC_r} \E_\bfy(T_{\mathbf{0}}^2)
		.
	\end{align}

	We next bound $\max_{\bfy \in \scrC_r} \E_\bfy(T_{\mathbf{0}})$ and $\max_{\bfy \in \scrC_r} \E_\bfy(T_{\mathbf{0}}^2)$.
	Let $\bfy\in\scrC_r$.
	Consider an experiment where $\QSED{r}(t)$ starts in $\bfy$ at time $0$.
	If $T_{\mathbf{0}} < T_{\scrB_r}$ then the experiment is considered a success, and a failure otherwise.
	If a failure occurs, then $\QSED{r}(T_{\mathbf{0}\cup\scrB_r})\in\scrB_r$.
	In that case we wait for the process to hit $\scrC_r$ again, and then another attempt starts.
	Let $N$ denote the first success, let $L_i$ denote the duration of attempt~$i$, and let $L_i'$ denote the duration to return to $\scrC_r$ after a failed attempt, with $L_N' := 0$.
	Then
	\begin{align}
		\label{eq:T0_attempt_decomp}
		T_{\mathbf{0}} & =\sum_{i=1}^{N}(L_i+L_i')
		.
	\end{align}

	We first bound $\E_\bfy(T_{\mathbf{0}})$.
	We can apply the monotone convergence theorem since the summands are non-negative,
	\begin{align}
		\label{eq:E_sum_attempts}
		\E_\bfy\Bigl(\sum_{i=1}^{N} L_i \Bigr)
		 & =
		\E_\bfy\Bigl(\sum_{i=1}^\infty \ind{i \leq N} L_i \Bigr)
		=
		\sum_{i=1}^\infty \E_\bfy\Bigl(\ind{i \leq N} L_i \Bigr)
		.
	\end{align}
	Let $\sigma_i$ denote the starting time of the $i$th attempt and let $\calF_s$ denote the natural filtration of $\QSED{r}(s)$, $s\geq 0$.
	By the strong Markov property,
	\begin{align}
		\label{eq:attempt_length_bnd_F}
		\ind{i \leq N} \E_{\bfy}(L_i \mid \calF_{\sigma_i})
		 & =
		\ind{i \leq N} \E_{\QSED{r}(\sigma_i)}(T_{\mathbf{0}\cup\scrB_r})
		\leq
		\ind{i \leq N} \max_{\bfz \in \scrC_r} \E_{\bfz}(T_{\mathbf{0}\cup\scrB_r})
		,
	\end{align}
	where the last inequality follows since $\QSED{r}(\sigma_i) \in\scrC_r$.
	Using the law of total expectation and that $\ind{i \leq N}$ is $\calF_{\sigma_i}$-measurable yields
	\begin{align}
		\label{eq:bnd_E_attempt_i_leq_N}
		\E_\bfy\bigl(\ind{i \leq N} L_i \bigr)
		 & =
		\E_\bfy\bigl(\E_\bfy\bigl(\ind{i \leq N} L_i \mid \calF_{\sigma_i} \bigr) \bigr)
		=
		\E_\bfy\bigl(\ind{i \leq N} \E_\bfy\bigl(L_i \mid \calF_{\sigma_i}\bigr) \bigr)
		\stackrel{\eqref{eq:attempt_length_bnd_F}}{\leq}
		\P_\bfy(i \leq N)
		\max_{\bfz \in \scrC_r} \E_{\bfz}(T_{\mathbf{0}\cup\scrB_r})
		.
	\end{align}
	Repeating \eqref{eq:attempt_length_bnd_F} and \eqref{eq:bnd_E_attempt_i_leq_N} with $L_i'$ replacing $L_i$, we find
	\begin{align}
		\label{eq:E_sum_return_length}
		\E_\bfy\bigl(\ind{i \leq N} L_i' \bigr)
		 & \leq
		\P_\bfy(i \leq N)
		\max_{\bfx \in \scrB_r} \E_{\bfx}(T_{\scrC_r})
		.
	\end{align}
	By \Cref{lem:hitting_time_finite_moments,lem:bnd_hit_0B},
	\begin{align}
		\label{eq:bnds_hitting_attempts}
		\max_{\bfz \in \scrC_r} \E_{\bfz}(T_{\mathbf{0}\cup\scrB_r})
		 & \leq
		2 \ell_r e^{\lambda \ell_r}, \qquad
		\max_{\bfx \in \scrB_r} \E_{\bfx}(T_{\scrC_r})
		\leq
		2a_r^{-1}  \max_{\bfx \in \scrB_r} \sqrt{\phi_r(\bfx)}
		\stackrel{\eqref{eq:bnd_phi_on_B}}{\leq}
		2a_r^{-1} \sqrt{\xi_r}
		.
	\end{align}
	Since each attempt starts from a state in $\scrC_r$, its success probability is at least $p_r>0$ in \eqref{eq:bnd_empty_before_arr}, independent of previous attempts.
	We therefore have $\P_\bfy(i \leq N) \leq (1-p_r)^{i-1}$.
	Therefore, by \eqref{eq:T0_attempt_decomp}--\eqref{eq:bnds_hitting_attempts},
	\begin{align}
		\label{eq:bnd_E_T0}
		\E_{\bfy}(T_{\mathbf{0}})
		 & \leq
		\Bigl(\sum_{i=1}^\infty \P_\bfy(i \leq N) \Bigr) (2 \ell_r e^{\lambda \ell_r} + 2a_r^{-1} \sqrt{\xi_r})
		=
		p_r^{-1} (2 \ell_r e^{\lambda \ell_r} + 2a_r^{-1} \sqrt{\xi_r})
		.
	\end{align}
	Since the right side in \eqref{eq:bnd_E_T0} is independent of $\bfy$, we may take the maximum over all $\bfy\in\scrC_r$.
	Bound~\eqref{eq:bnd_B_final} follows from~\eqref{eq:bnd_B_split} and~\eqref{eq:bnd_E_T0}.

	We next bound $\E_\bfy(T_{\mathbf{0}}^2)$.
	By \eqref{eq:T0_attempt_decomp} and the monotone convergence theorem,
	\begin{align}
		\sqrt{\E_{\bfy}(T_{\mathbf{0}}^2) }
		 & =
		\sqrt{\E_\bfy\Bigl( \Bigl( \sum_{i=1}^\infty \ind{i\leq N}(L_i+L_i') \Bigr)^2 \Bigr)}
		=
		\lim_{n\to\infty}
		\sqrt{\E_\bfy\Bigl( \Bigl( \sum_{i=1}^n \ind{i\leq N}(L_i+L_i') \Bigr)^2 \Bigr)}
		.
	\end{align}
	We use the Minkowski inequality
	$
		\sqrt{\E\Bigl(\bigl(\sum_{i=1}^n a_i\bigr)^2 \Bigr)}
		\leq
		\sum_{i=1}^n \sqrt{\E(a_i^2)}
	$
	for $a_i \geq 0$
	to obtain
	\begin{align}
		\sqrt{\E_\bfy\Bigl( \Bigl( \sum_{i=1}^n \ind{i\leq N}(L_i+L_i') \Bigr)^2 \Bigr)}
		 & \leq
		\sum_{i=1}^n \sqrt{\E_\bfy(\ind{i\leq N}(L_i+L_i')^2)}
		.
	\end{align}
	Again applying the Minkowski inequality gives
	\begin{align}
		\sqrt{\E_\bfy(\ind{i\leq N} (L_i + L_i')^2)}
		 & \leq
		\sqrt{\E_\bfy(\ind{i\leq N} L_i^2)} + \sqrt{\E_\bfy(\ind{i\leq N} (L_i')^2)}
		.
	\end{align}
	Hence,
	\begin{align}
		\sqrt{\E_{\bfy}(T_{\mathbf{0}}^2)}
		 & \leq
		\sum_{i=1}^\infty
		\sqrt{\E_\bfy(\ind{i\leq N} L_i^2)} + \sqrt{\E_\bfy(\ind{i\leq N} (L_i')^2)}
		.
	\end{align}

	Following the same steps as in \eqref{eq:attempt_length_bnd_F}--\eqref{eq:bnd_E_attempt_i_leq_N} with $L_i^2$ replacing $L_i$ and using \Cref{lem:bnd_hit_0B} gives
	\begin{align}
		\label{eq:bnd_E_attempt_sq}
		\E_\bfy\bigl(\ind{i\leq N} L_i^2\bigr)
		 & \leq
		\P_\bfy(i\leq N) \max_{\bfz\in\scrC_r}\E_\bfz(T_{\mathbf{0}\cup\scrB_r}^2)
		\leq
		\P_\bfy(i\leq N)\, 8\ell_r^2 e^{2\lambda\ell_r}
		,
	\end{align}
	and analogously, using \Cref{lem:hitting_time_finite_moments},
	\begin{align}
		\label{eq:bnd_E_return_sq}
		\E_\bfy\bigl(\ind{i\leq N} (L_i')^2\bigr)
		 & \leq
		\P_\bfy(i\leq N) \max_{\bfx\in\scrB_r}\E_\bfx(T_{\scrC_r}^2)
		\leq
		\P_\bfy(i\leq N) 4 a_r^{-2} \max_{\bfx\in\scrB_r} \phi_r(\bfx)
		\stackrel{\eqref{eq:bnd_phi_on_B}}{\leq}
		\P_\bfy(i\leq N) 4a_r^{-2}\xi_r
		.
	\end{align}

	Then, using $\P_\bfy(i\leq N) \leq (1-p_r)^{i-1}$,
	\begin{align}
		\label{eq:E_T2_sqrt}
		\sqrt{\E_{\bfy}(T_{\mathbf{0}}^2)}
		 & \leq
		\sum_{i=1}^\infty \sqrt{(1-p_r)^{i-1}}  \, \Bigl( \sqrt{8\ell_r^2 e^{2\lambda\ell_r}} + \sqrt{4a_r^{-2}\xi_r}  \Bigr)
		=
		\frac{1}{1-\sqrt{1-p_r}} \,  \Bigl( \sqrt{8\ell_r^2 e^{2\lambda\ell_r}} + \sqrt{4a_r^{-2}\xi_r}  \Bigr)
		.
	\end{align}
	Now
	\begin{align}
		\frac{1}{1-\sqrt{1-p_r}}
		 & =
		\frac{1+\sqrt{1-p_r}}{(1-\sqrt{1-p_r})(1+\sqrt{1-p_r})}
		=
		\frac{1+\sqrt{1-p_r}}{p_r}
		\leq
		2 p_r^{-1}
		.
		\label{eq:p_algebra}
	\end{align}
	Squaring \eqref{eq:E_T2_sqrt}, using \eqref{eq:p_algebra} and that $(x+y)^2 \leq 2x^2 + 2y^2$ for all $x,y$ gives
	\begin{align}
		\label{eq:bnd_E_T0_squared}
		\E_\bfy(T_{\mathbf{0}}^2)
		 & \leq
		8 p_r^{-2} (8 \ell_r^2 e^{2\lambda \ell_r} + 4a_r^{-2} \xi_r)
		.
	\end{align}
	Since the right side in \eqref{eq:bnd_E_T0_squared} is independent of $\bfy$, we may take the maximum over all $\bfy\in\scrC_r$.
	Bound~\eqref{eq:bnd_B_squared_final} follows from~\eqref{eq:bnd_B_split_squared} and~\eqref{eq:bnd_E_T0_squared}.
	This completes the proof.

\end{proof}

\subsection{Proof of~\Cref{lem:est_error}}
\label{app:proof_lem_est_error}

\begin{proof}
	Let $r^* =  a/b$ in lowest terms and let $k\in\N_{\geq 1}$ satisfy $k\geq \exp(1/c_1)$.
	By \Cref{lem:number_approximation} and the definition of $m(k)$ in~\eqref{eq:def_M_k},
	\begin{align}
		\P\bigl(|\hat r_k - r^*| \geq \uval{r^*}{m(k)}\bigr)
		 & \leq
		\P\Bigl(|\hat r_k - r^*| \geq \frac{1}{(m(k)+1)b}\Bigr)
		\leq
		\P\bigl(|\hat r_k - r^*| \geq \zeta_k^{-1}\bigr),
		\qquad
		\zeta_k := (c_1 \ln(k)+1)b
		.
		\label{eq:P_r_gamma}
	\end{align}
	Note that $\ln(k) > 0$  since $k\geq 2$, so $0 < \zeta_k^{-1} < 1/b < 2a/b = 2r^*$.
	By~\eqref{eq:dep_geq_alpha}, we have $N_1^{(k)}\wedge N_2^{(k)}\geq \alpha(k-1)$ almost surely.
	Hence, by \Cref{lem:rhat_subexp},
	\begin{align}
		\P\bigl(|\hat r_k - r^*| \geq \zeta_k^{-1}\bigr)
		 & =
		\P\bigl(
		|\hat r_k-r^*| \geq \zeta_k^{-1}, \ N_1^{(k)}\wedge N_2^{(k)} \geq \alpha(k-1)
		\bigr)
		\leq
		4\exp\left(
		-\frac{\alpha(k-1)}
			{128 (r^*)^2\zeta_k^2}
		\right)
		.
	\end{align}
	Note that $\alpha(k-1) = \lceil \ln(k)^4\rceil \geq \ln(k)^4$ and that $r^* = a/b$.
	Moreover,  $(x+1)^2 \leq 4x^2$ for $x\geq 1$, and $c_1\ln(k) \geq 1$ by assumption.
	This concludes the proof,
	\begin{align}
		\P\bigl(|\hat r_k - r^*| \geq \zeta_k^{-1}\bigr)
		 & \leq
		4\exp\left(
		-\frac{\ln(k)^4}
			{128 a^2(c_1\ln(k)+1)^2}
		\right)
		\leq
		4\exp\left(
		-\frac{\ln(k)^2}
			{512 a^2 c_1^2 }
		\right)
		=
		4 k^{-\ln(k) / (512 a^2 c_1^2)}
		.
	\end{align}
\end{proof}

\subsection{Proof of~\Cref{lem:rhat_subexp}}
\label{app:proof_lem_rhat_subexp}

\begin{proof}
	We first derive concentration inequalities for $\hat{\mu}_i^{(k)}$.
	Let $k,m\in\N_{\geq 1}$ and  $i\in\{1,2\}$.
	Let $X_{ij}$, $j\in\N_{\geq 1}$ be i.i.d.\ random variables with $X_{i1} \sim \textrm{Exp}(\mu_i)$.
	We have by the definitions of $\hat{\mu}_i^{(k)}$, $S_i^{(k)}$, and $N_i^{(k)}$,
	\begin{align}
		\hat{\mu}_i^{(k)}
		 & =
		\frac{N_i^{(k)}}{\sum_{j=1}^{N_i^{(k)}} X_{ij}}
		.
	\end{align}
	Let $0 < \delta < 1$ and $M_n := \sum_{j=1}^n \mu_i X_{ij}$.
	Note that
	\begin{align}
		\bigl\{|\hat{\mu}_i^{(k)} / \mu_i - 1| \geq \delta \bigr\}
		\cap
		\{ N_i^{(k)} \geq m\}
		\subseteq
		\bigl\{\exists n \geq m: |n / M_n - 1| \geq \delta \bigr\}
		.
		\label{eq:dep_samples_inclusion}
	\end{align}

	We next bound the probabilities of $M_n\leq n/(1+\delta)$ and $M_n\geq n/(1-\delta)$ for some $n\geq m$ separately.
	These bounds follow by constructing a martingale and applying Ville's inequality.
	In particular, we use the following as stated in \cite[Sec.~12.5, Ex.~(15)]{Grimmett2001}:
	{\it
	Let $\xi_1,\xi_2,\dots,$ i.i.d.\ and let $\theta \neq 0$ such that $ \E(\exp(\theta \xi_1)) < \infty$.
	Then,  $(\E(\exp(\theta \xi_1)))^{-n} \exp(\theta \sum_{j=1}^n \xi_j)$ is a martingale.
	}

	Since
	\begin{align}
		\E(\exp(-\delta \mu_i X_{i1}))
		= \int_0^\infty \exp(-\delta \mu_i t) \mu_i \exp(-\mu_i t) \rmd t
		= \frac{1}{1+\delta} < \infty
		,
	\end{align}
	it follows that
	$(\E(\exp(-\delta \mu_i X_{i1})))^{-n} \exp(-\delta M_n) =
		(1+\delta)^n \exp(-\delta M_n)$ is a martingale.
	Since this martingale is also non-negative, Ville's inequality \cite[Thm.~1.1]{Koolen2026} gives
	\begin{align}
		\P\Bigl(\exists n\geq 0:  (1+\delta)^n \exp(-\delta M_n) \geq c \Bigr)
		 & \leq
		\frac 1c
		,
		\qquad c>0
		.
		\label{eq:ville}
	\end{align}

	Suppose there exists $n_0 \geq m$ with $M_{n_0} \leq n_0/(1+\delta)$, then
	\begin{align}
		(1+\delta)^{n_0} \exp(-\delta M_{n_0})
		 & \geq
		\exp\bigl(n_0(\log(1+\delta) - \tfrac{\delta}{1+\delta})\bigr)
		\geq
		\exp\bigl(m(\log(1+\delta) - \tfrac{\delta}{1+\delta})\bigr)
		,
		\label{eq:Sn_subset_martingale}
	\end{align}
	where the last inequality follows since $\log(1+\delta) - \delta/(1+\delta) \geq 0$ and $n_0 \geq m$.
	Therefore,
	\begin{align}
		\Bigl\{\exists n\geq m: M_n \leq \frac{n}{1+\delta}\Bigr\}
		\subseteq
		\Bigl\{\exists n\geq 0: (1+\delta)^n \exp(-\delta M_n) \geq \exp\bigl(m(\log(1+\delta) - \tfrac{\delta}{1+\delta})\bigr) \Bigr\}
		.
		\label{eq:Sn_subset}
	\end{align}
	Combining~\eqref{eq:ville} and \eqref{eq:Sn_subset} and subsequently that $\log(1+\delta) - \delta/(1+\delta) \geq \delta^2/8$  gives
	\begin{align}
		\P\Bigl(\exists n\geq m: M_n \leq \frac{n}{1+\delta}\Bigr)
		\leq
		\exp\bigl(- m(\log(1+\delta) - \tfrac{\delta}{1+\delta})\bigr)
		\leq
		\exp(-\tfrac{m\delta^2}{8})
		.
	\end{align}

	Using similar steps, but now with the martingale $(1-\delta)^n\exp(\delta M_n)$, one finds that
	\begin{align}
		\P\Bigl(\exists n\geq m: M_n \geq \frac{n}{1-\delta}\Bigr)
		\leq
		\exp\bigl(- m( \tfrac{\delta}{1-\delta} + \log(1-\delta))\bigr)
		\leq
		\exp(-\tfrac{m\delta^2}{8})
		.
	\end{align}
	Combining these bounds with \eqref{eq:dep_samples_inclusion} gives
	\begin{align}
		\P(|\hat{\mu}_i^{(k)} / \mu_i - 1| \geq \delta, \, N_i^{(k)} \geq m)
		 & \leq
		2\exp(-\tfrac{m\delta^2}{8}),
		\qquad  0 < \delta < 1, \ i = 1,2
		.
		\label{eq:muhat_stopped_concentration}
	\end{align}

	Let now $\eps < 2r^*$ and take
	$
		\delta:= \eps / (4r^*) < 1/2
		.
	$
	Since $\hat{r}_k = \hat{\mu}_2^{(k)} / \hat{\mu}_1^{(k)}$ and $r^* = \mu_2/\mu_1$, we have that if $|\hat\mu_i^{(k)} / \mu_i -1|<\delta$ for $i=1,2$, then
	\begin{align}
		\left|\frac{\hat r_k}{r^*}-1\right|
		<
		\frac{2\delta}{1-\delta}
		<
		\frac{\eps}{r^*}.
	\end{align}
	Hence,
	\begin{align}
		\{|\hat r_k-r^*|\geq\eps\}
		\subseteq
		\bigcup_{i=1}^2
		\left\{
		\left|\frac{\hat\mu_i^{(k)}}{\mu_i}-1\right|
		\geq\delta
		\right\}
		.
		\label{eq:r_mu_inclusion}
	\end{align}
	Combining \eqref{eq:muhat_stopped_concentration} and \eqref{eq:r_mu_inclusion} concludes the proof,
	\begin{align}
		\P(|\hat r_k - r^*| \geq \eps, \ N_1^{(k)} \wedge N_2^{(k)} \geq m)
		\leq
		4 \exp\Bigl(- \frac{m \eps^2 }{128 (r^*)^2}\Bigr)
		.
	\end{align}

\end{proof}

\subsection{Proof of \Cref{lem:exploration}}
\label{app:proof_lem_exploration}

\begin{proof}
	By~\eqref{eq:event_exploration},
	\begin{align}
		\label{eq:exploration_split}
		\P(\Eexplr{k})
		\leq
		\P\bigl(N_1^{(k)} \wedge N_2^{(k)} < \alpha(k)\bigr)
		\leq
		\P\bigl(N_1^{(k)} < \alpha(k)\bigr) + \P(N_2^{(k)} < \alpha(k))
		.
	\end{align}
	Let $k\in\N_{\geq 1}$ satisfy $\alpha(k) \leq (k-1) v^*/4$.
	We have, since $Z_i^{(m)} \geq 0$,
	\begin{align}
		N_i^{(k)}
		 & =
		\sum_{m=1}^{k-1} Z_i^{(m)}
		=
		\sum_{m=1}^{k-1} Z_i^{(m)}\bigl(\ind{\calG^{(m)}} + \ind{(\calG^{(m)})^c} \bigr)
		\stackrel{\textrm{\eqref{eq:def_good_episode}}}{\geq}
		\sum_{m=1}^{k-1} \ind{\calG^{(m)}}
		.
	\end{align}
	Therefore,
	\begin{align}
		\label{eq:relation_N_G}
		\P\bigl(N_i^{(k)} < \alpha(k)\bigr)
		 & \leq
		\P\Bigl( \sum_{m=1}^{k-1}  \ind{\calG^{(m)}} < \alpha(k)\Bigr)
		.
	\end{align}

	For $m\in\N_{\geq 1}$, let
	\begin{align}
		\tilde{I}_m := \ind{|\bar{r}_m-r^*|< r^*/2}
		.
	\end{align}
	Then,
	\begin{align}
		\label{eq:good_events_split}
		\P\Bigl(\sum_{m=1}^{k-1} \ind{\calG^{(m)}} < \alpha(k)\Bigr)
		      & \leq
		\P\Bigl(\sum_{m=1}^{k-1} \ind{\calG^{(m)}} < \alpha(k), \, \sum_{m=1}^{k-1} \tilde{I}_m \geq \tfrac{k-1}{2} \Bigr)
		+
		\P\Bigl(\sum_{m=1}^{k-1} \tilde{I}_m < \tfrac{k-1}{2} \Bigr)
		.
	\end{align}

	We first bound the last term in~\eqref{eq:good_events_split}.
	We have $N_1^{(m)} \wedge N_2^{(m)} \geq \alpha(m-1) \geq \ln(m)^4$ almost surely by~\eqref{eq:dep_geq_alpha} and~\eqref{eq:def_alpha_k}.
	Hence, by \Cref{lem:rhat_subexp},
	\begin{align}
		\P(|\hat{r}_m - r^*| \geq r^*/2)
		 & =
		\P\left(
		|\hat r_m-r^*|\geq r^*/2, \
		N_1^{(m)} \wedge N_2^{(m)}
		\geq
		\alpha(m-1)
		\right) \\
		 & \leq
		4\exp\left(
		-\frac{
				\ln(m)^4 (r^*/2)^2
			}{
				128(r^*)^2
			}
		\right)
		=
		4\exp\left(
		-\frac{\ln(m)^4}{512}
		\right)
		\quad
		\textnormal{for}
		\quad
		m\geq 1
		.
		\label{eq:P_Hm_large}
	\end{align}
	Note that $4\exp(-\ln(m)^4 / 512)  = 4m^{-3} \exp\bigl(3\ln(m) - \ln(m)^4/512\bigr)$.
	Moreover, the function $x\mapsto 3x - x^4/512$ attains a global maximum at $x=4\cdot 6^{1/3}$, with value $9 \cdot 6^{1/3} < 17$.
	Therefore,
	\begin{align}
		\P(|\hat{r}_m - r^*| \geq r^*/2)
		 &
		\leq
		4\exp(17) m^{-3}
		\quad
		\textnormal{for}
		\quad
		m\geq 1
		.
	\end{align}
	By assumption $r^*\in[\rmin,\rmax]$ and so $|\bar{r}_m-r^*| \leq |\hat{r}_m-r^*|$; recall \Cref{line:clip} in \Cref{alg:LASED}.
	Therefore,
	\begin{align}
		\label{eq:bnd_P_L_m_1}
		\P(\tilde{I}_m=0) \leq \P(|\hat{r}_m - r^*| \geq r^*/2) & \leq 4\exp(17) m^{-3}
		.
	\end{align}
	It can be verified that $\alpha(k) \leq (k-1) v^*/4$ implies that $k\geq 9$.
	Hence, given that~\eqref{eq:bnd_P_L_m_1} holds for all $m\geq 1$, we can apply \Cref{lem:copulas} with $c=4\exp(17)$ and $d=3$ to obtain
	\begin{align}
		\label{eq:P_not_conc}
		\P\Bigl(\sum_{m=1}^{k-1} \tilde{I}_m < \tfrac{k-1}{2}\Bigr)
		\leq
		2^{9} \exp(17) (k-1)^{-3}
		.
	\end{align}

	We next bound the other term in \eqref{eq:good_events_split}.
	Consider the martingale difference sequence $D_m := \tilde{I}_m \ind{\calG^{(m)}} - \E(\tilde{I}_m \ind{\calG^{(m)}} \mid \calF_{m-1})$.
	On the event $\{\sum_{m=1}^{k-1} \ind{\calG^{(m)}} < \alpha(k), \, \sum_{m=1}^{k-1} \tilde{I}_m \geq (k-1)/2\}$, we have, using that $\tilde{I}_m$ is $\calF_{m-1}$-measurable and \Cref{lem:at_least_one_departure},
	\begin{align}
		\sum_{m=1}^{k-1} D_m 
		&\leq
		\alpha(k) - \sum_{m=1}^{k-1} \tilde{I}_m  \E(\ind{\calG^{(m)}} \mid \calF_{m-1})
		\leq  
		\alpha(k) -  v^* \tfrac{k-1}{2} 
		\leq 
		- v^* \tfrac{k-1}{4}
		,
	\end{align}
	where the last inequality follows since $\alpha(k)\leq (k-1)v^*/4$ by assumption.
	We may apply the Azuma--Hoeffding inequality \citep[Thm.~1]{Azuma1967} since $|D_m|\leq 1$,
	\begin{align}
		\P\Bigl(\sum_{m=1}^{k-1} \ind{\calG^{(m)}} < \alpha(k), \, \sum_{m=1}^{k-1} \tilde{I}_m \geq \tfrac{k-1}{2}\Bigr) 
		&\leq 
		\P\Bigl(\sum_{m=1}^{k-1} D_m \leq - v^* \tfrac{k-1}{4} \Bigr) \\
		&\leq 
		\exp\Bigl(-\frac{(v^* \tfrac{k-1}{4})^2}{2(k-1)}\Bigr) \\
		&= 
		\exp(-\tfrac{1}{32}  (v^*)^2 (k-1)) \\
		&\leq 
		\exp(-\tfrac{1}{64} (v^*)^2 k)
		.
		\label{eq:P_good_suf_concen}
	\end{align}
	Here, the last inequality follows since $k-1\geq k/2$ for $k\geq 2$. 
	
	The proof is concluded by combining~\eqref{eq:exploration_split},~\eqref{eq:relation_N_G},~\eqref{eq:good_events_split},~\eqref{eq:P_not_conc}, and~\eqref{eq:P_good_suf_concen}.
\end{proof}

\subsection{Proof of \Cref{lem:at_least_one_departure}}
\label{app:proof_lem_at_least_one_departure}

\begin{proof}
	Let $f_i(h)$ denote the probability that the queue length of an $M/M/1$ queue with arrival rate $\lambda$ and service rate $\mu_i$ reaches $h\in\N_{\geq 1}$ during one busy period, $i\in\{1,2\}$.
	By~\citep[Eq.\ (2.50)]{Cohen1982},
	\begin{align}
		\label{eq:def_Fih}
		f_i(h)
		 & =
		\frac{\rho_i^{h-1}}{\sum_{m=0}^{h-1} \rho_i^m}
		.
	\end{align}

	Let $k\in\N_{\geq 1}$ and denote  $\calA_k = \{|\bar{r}_k - r^*| < r^*/2\}$ and  $\calD_k = \{N_1^{(k)}\wedge N_2^{(k)} \geq \alpha(k)\}$.
	Both $\bar{r}_k$ and $\calD_k$ are $\calF_{k-1}$-measurable.
	Moreover, conditional on $\calF_{k-1}$, the arrival and potential service processes during episode~$k$ are independent of the history.
	On $\calD_k^c$, $Z_1^{(k)} \wedge Z_2^{(k)}\geq 1$ by \Cref{line:forced_exploration} of \Cref{alg:LASED}, hence 
	\begin{align}
		\P\bigl(\calG^{(k)} \mid \calF_{k-1}\bigr) = 1 \qquad \textrm{on} \ \calD_k^c
		.
		\label{eq:P_G_forced_conditional}
	\end{align}
	Suppose now that $\calD_k$ occurs. 
	Then \gls{LASED} routes customers according to \SED{\bar{r}_k} during episode $k$.
	Moreover, there is no exploration phase and so episode~$k$ is one busy period.

	Suppose that $\calD_k\cap\{\bar{r}_k>1\}$. 
	The first arrival in episode~$k$ is then routed to server two, so $Z_2^{(k)}\geq 1$.
	Thus, by definition of $\calG^{(k)}$ in~\eqref{eq:def_good_episode},
	\begin{align}
		\label{eq:P_G_geq_Z1}
		\P(\calG^{(k)} \mid \calF_{k-1})
		 & =
		\P(Z_1^{(k)} \geq 1 \mid \calF_{k-1}) \qquad \textrm{on} \ \calD_k\cap\{\bar{r}_k>1\}
		.
	\end{align}
	A sufficient event for $Z_1^{(k)} \geq 1$ is that the queue length at server two reaches $\bar{r}_k - 1$ and there is an arrival.
	Until such a time, queue~2 operates as an $M/M/1$ system.
	The probability that the second queue exceeds $\bar{r}_k-1$ during episode $k$ thus equals $f_2(\lceil\bar{r}_k\rceil)$.
	If the system is in state $(Q_1,Q_2)(t) = (0,j)$ with $j > 0$, then the probability that the next event is an arrival (before a departure) is $\lambda/(\lambda+\mu_2)$ by the memoryless property of the exponential distribution~\citep{Grimmett2001}.
	We thus find using the law of total probability,
	\begin{align}
		\label{eq:P_G_geq}
		\P(Z_1^{(k)} \geq 1 \mid \calF_{k-1})
		\geq
		\frac{\lambda}{\lambda+\mu_2} f_2(\lceil\bar{r}_k\rceil) \qquad \textrm{on} \ \calD_k\cap\{\bar{r}_k>1\}
		.
	\end{align}

	Similarly, on $\calD_k\cap\{\bar{r}_k\leq1\}$, there is a departure from server~1 with probability~1, and
	\begin{align}
		\label{eq:P_G_leq}
		\P(\calG^{(k)} \mid \calF_{k-1})
		 & =
		\P(Z_2^{(k)}\geq 1 \mid \calF_{k-1})
		\geq
		\frac{\lambda}{\lambda+\mu_1} f_1(\lceil1/\bar{r}_k \rceil) \qquad \textrm{on} \ \calD_k\cap\{\bar{r}_k\leq1\}
		.
	\end{align}

	If $\calA_k$ occurs, then $r^*/2 < \bar{r}_k <  3r^*/2$.
	Moreover, $f_i(h)$ in~\eqref{eq:def_Fih} is non-increasing in $h$.
	Therefore, by~\eqref{eq:P_G_geq} and~\eqref{eq:P_G_leq},
	\begin{align}
		\P\bigl(\calG^{(k)} \bigm| \calF_{k-1}\bigr)
		 & \geq
		\Bigl(
		\frac{\lambda}{\lambda+\mu_1} f_1(\lceil 2/r^* \rceil)
		\Bigr)
		\wedge
		\Bigl(
		\frac{\lambda}{\lambda+\mu_2} f_2(\lceil 3r^*/2 \rceil)
		\Bigr), 
		\qquad \textrm{on} \ \calA_k\cap\calD_k
		.
		\label{eq:P_good_bnd}
	\end{align}
	Note that
	\begin{align}
		f_i(h)
		 & \geq
		\frac{\rho_i^{h-1}}{h(\rho_i^{h-1} \vee 1)}
		=
		\frac{1}{h} (\rho_i\wedge 1)^{h-1}
		,
		\qquad
		h\in\N_{\geq 1}
		.
		\label{eq:f_bnd}
	\end{align}
	Recall that $v^*$ is defined in~\eqref{eq:def_vr}.
	We obtain from~\eqref{eq:P_good_bnd} and~\eqref{eq:f_bnd}, using $1/\lceil x \rceil \geq 1/(x+1)$ and $(\rho_i\wedge 1)^{\lceil x\rceil - 1} \geq (\rho_i \wedge 1)^x$,
	that
	$
		\P\bigl(\calG^{(k)} \bigm| \calF_{k-1}\bigr)
		\geq v^*
	$
	on $\calA_k\cap\calD_k$. 
	Together with \eqref{eq:P_G_forced_conditional} and $v^* \leq 1$, 
	\begin{align}
		\P\bigl(\calG^{(k)} \bigm| \calF_{k-1}\bigr)
		&\geq 
		v^* \ind{\calA_k \cap \calD_k}
		+ \ind{\calD_k^c}
		\geq 
		v^* \ind{\calA_k}
		.
	\end{align}
	This concludes the proof.
\end{proof}

\subsection{Proof of \Cref{lem:copulas}}
\label{app:proof_lem_copulas}

\begin{proof}
	Let $\bar I_m = 1-I_m$, then
	\begin{align}
		\label{eq:P_sum_I_not}
		\P\Bigl(\sum_{m=1}^k I_m < k/2\Bigr)
		 & =
		\P\Bigl(\sum_{m=1}^k \bar I_m > k/2\Bigr)
		.
	\end{align}
	Since $\bar I_m\in\{0,1\}$ for $m=1,\dots,k$, $\sum_{m=1}^{\lfloor k/4 \rfloor} \bar I_m \leq \lfloor k/4\rfloor \leq k/4$.
	Therefore,
	\begin{align}
		\label{eq:P_H_k4}
		\P\Bigl(\sum_{m=1}^k \bar I_m > k/2\Bigr)
		 & \leq
		\P\Bigl(\sum_{m=\lfloor k/4\rfloor+1}^k \bar I_m > k/4 \Bigr)
		.
	\end{align}
	By Markov's inequality,
	\begin{align}
		\label{eq:H_k4_markov}
		\P\Bigl(\sum_{m=\lfloor k/4\rfloor+1}^k \bar I_m > k/4\Bigr)
		 & \leq
		\frac{\E\Bigl( \sum_{m=\lfloor k/4\rfloor+1}^k \bar I_m\Bigr)}{k/4}
		.
	\end{align}
	We next use $\E(\bar I_m) = \P(\bar I_m = 1) \leq cm^{-d}$, linearity of expectation, and the bound
	$\sum_{m=\ell+1}^\infty m^{-d} \leq \int_\ell^\infty x^{-d} \rmd x$ for $\ell\in\N_{\geq 1}$,
	to obtain
	\begin{align}
		\E\Bigl(\sum_{m=\lfloor k/4\rfloor+1}^k \bar I_m \Bigr)
		 & \leq
		\sum_{m=\lfloor k/4\rfloor+1}^k cm^{-d}
		\leq
		\int_{\lfloor k/4\rfloor}^\infty cx^{-d} \rmd x
		=
		\frac{c}{(d-1) \lfloor k/4\rfloor^{d-1}}
		.
	\end{align}
	By (i) $\lfloor k/4\rfloor \geq k/4 - 1 = (k-4)/4$, and (ii) $k-4 \geq k/2$ for $k \geq 8$,
	\begin{align}
		\E\Bigl(\sum_{m=\lfloor k/4\rfloor+1}^k \bar I_m \Bigr)
		 & \stackrel{\textrm{(i)}}{\leq}
		\frac{4^{d-1} c}{(d-1) (k-4)^{d-1}}
		\stackrel{\textrm{(ii)}}{\leq}
		\frac{4^{d-1}2^{d-1} c}{(d-1) k^{d-1}}
		=
		\frac{2^{3d-3} c}{(d-1) k^{d-1}}
		\label{eq:H_k4_integral}
		.
	\end{align}
	The proof is concluded from~\eqref{eq:P_sum_I_not}--\eqref{eq:H_k4_integral}.
\end{proof}

\subsection{Proof of \Cref{lem:overflow}}
\label{app:proof_lem_overflow}

\begin{proof}
	Let $F_k(x) := \P(\bar{r}_k \leq x \mid (\Eexplr{k})^c)$.
	We have $\bar{r}_k \in [\rmin,\rmax]$ with probability~1 by \Cref{line:clip} in \Cref{alg:LASED}.
	Therefore,
	\begin{align}
		\P(\Eovf{k} \mid (\Eexplr{k})^c)
		 & =
		\int_{[\rmin,\rmax]}
		\P\bigl(\Eovf{k} \mid (\Eexplr{k})^c, \, \bar{r}_k = r\bigr) \,
		\rmd F_k(r)
		.
		\label{eq:Eovf_int_r_value}
	\end{align}
	Conditional on $(\Eexplr{k})^c$, episode~$k$ of \gls{LASED} is a busy period under the $\SED{\bar{r}_k}$ policy; recall \Cref{alg:LASED}.
	Using the definition of $\Eovf{k}$ in \eqref{eq:event_overflow} and  \Cref{lem:overflow_bp},
	\begin{align}
		\P\bigl(\Eovf{k} \mid (\Eexplr{k})^c, \, \bar{r}_k = r\bigr)
		 & =
		\P_{\mathbf{0}}
		\Bigl(
		\max_{0 \leq t\leq T_{\mathbf{0}}}
		\bigl(\QSEDsingle{r}{1}(t) \vee \QSEDsingle{r}{2}(t)\bigr)
		> m(k)
		\Bigr)  \\
		 & \leq
		p_{r}^{-1} \exp\bigl(-\theta_{r} \bigl(m(k) \sqrt{(r \wedge 1)/2} - \sqrt{\xi_{r}}\bigr)\bigr)
		.
		\label{eq:Eovf_BP_bnd}
	\end{align}
	Here, $\theta_{r} $, $p_{r}$, and $\xi_{r}$ are defined in~\eqref{eq:exp_martingale_theta}, \eqref{eq:bnd_empty_before_arr}, and~\eqref{eq:bnd_phi_on_B}, respectively.
	Define
	\begin{align}
		\label{eq:def_inf_sup_p_xi_theta}
		\theta_* := \inf_{r \in [\rmin, \rmax]} \theta_r \sqrt{(r \wedge 1)/2},
		\qquad
		p_* & := \inf_{r \in [\rmin, \rmax]} p_r,
		\qquad
		\xi_* := \sup_{r \in [\rmin, \rmax]} \sqrt{\frac{\xi_r}{(r\wedge 1)/2}}
		.
	\end{align}
	Each of $\theta_r$, $p_r$, and $\xi_r$ is strictly positive and finite on $[\rmin, \rmax]$:
	$\theta_r$ is the minimum of two positive continuous functions of $r$ (using that $a_r > 0$ by~\eqref{eq:stab_cond}),
	$p_r$ is a positive power of a number in $(0,1)$,
	and $\xi_r$ is a sum of positive terms.
	All three are continuous in $r$, so the infima and suprema over the compact interval $[\rmin, \rmax]$ are attained, giving $p_*, \theta_*, \xi_* \in (0,\infty)$.

	Let $k_0$ be such that $m(k_0) > \xi_*$, which exists since $m(k)$ in \eqref{eq:def_M_k} is non-decreasing in~$k$ and $m(k) \to\infty$ as $k\to\infty$.
	By the preceding assertion that $p_* \in (0,\infty)$, and since $p_* = \inf_{r \in [\rmin,\rmax]} p_r$, we have $p_r \ge p_* > 0$ and hence $p_r^{-1} \le p_*^{-1}$ for every $r \in [\rmin,\rmax]$.
	The proof is concluded from \eqref{eq:Eovf_int_r_value}, \eqref{eq:Eovf_BP_bnd}, and \eqref{eq:def_inf_sup_p_xi_theta},
	\begin{align}
		\P(\Eovf{k} \mid (\Eexplr{k})^c)
		 & \leq
		p_*^{-1} \exp(-\theta_* (m(k)-\xi_*))
		\int_{[\rmin,\rmax]}
		\rmd F_k(r)
		=
		p_*^{-1} \exp(-\theta_* (m(k)-\xi_*))
		, \quad
		\forall \, k \geq k_0
		.
	\end{align}

\end{proof}

\subsection{Proof of \Cref{lem:bnd_arrivals_episode}}
\label{app:proof_lem_bnd_arrivals_episode}

\begin{proof}
	{\bf First moment.}
	Let $r\in\R_{>0}$ and $k\in\N_{\geq 1}$.
	We condition on $\calD_k = \{N_1^{(k)} \wedge N_2^{(k)} \geq \alpha(k)\}$.
	Using the law of total probability,
	\begin{align}
		\label{eq:bnd_EA_phases}
		\E(\Arrepi{k}\mid \bar{r}_k = r)
		 & \leq
		\max
		\bigl(
		\E(\Arrepi{k}\mid \bar{r}_k = r, \calD_k), \
		\E(\Arrepi{k}\mid \bar{r}_k = r, \calD_k^c)
		\bigr)
		.
	\end{align}

	On $\calD_k$, episode~$k$ has no exploration phase; see \Cref{line:check_dep} in \Cref{alg:LASED}.
	In this case, episode~$k$ is exactly one busy period where customers are routed according to \SED{\bar{r}_k}; see \Crefrange{line:begin_exploitation}{line:end_exploitation}.
	Thus,
	\begin{align}
		\label{eq:A_epi_SED}
		\E(\Arrepi{k}\mid \bar{r}_k = r, \calD_k) = \E_{(0,0)}(\ArrSED{r})
		,
	\end{align}
	where $\ArrSED{r}$ denotes the number of arrivals in one busy period of the $\QSED{r}(\cdot)$ process as introduced in \Cref{sec:SED_policy}, and $\E_\bfa$ denotes expectation conditioned on $\QSED{r}(0) = \bfa$.

	Suppose that $\calD_k^c$ occurs.
	The number of customers routed to server~$i$ during exploration, see \Cref{line:forced_exploration} in \Cref{alg:LASED}, is then
	\begin{align}
		\bigl(\lceil \alpha(k)\rceil - N_i^{(k)}\bigr) \vee 1
		 & \stackrel{\eqref{eq:dep_geq_alpha}}{\leq}
		\bigl(\lceil \alpha(k)\rceil - \alpha(k-1)\ind{k\geq 2}\bigr) \vee  1 \\
		 & \leq
		1 + \alpha(k) - \alpha(k-1) \ind{k\geq 2}                             \\
		 & \stackrel{\eqref{eq:def_alpha_k}}{=}
		1 + \lceil \ln(k+1)^4 \rceil - \lceil \ln(k)^4 \rceil \ind{k\geq 2}   \\
		 & \stackrel{\textrm{(i)}}{\leq}
		7
		.
	\end{align}
	Here, $\textrm{(i)}$ follows since $\log(m+1)^4 - \log(m)^4 \leq 6$ for any $m\geq 1$.
	Since there are two servers, the number of arrivals during exploration is bounded by $14$.
	As soon as there have been $\bigl(\lceil \alpha(k)\rceil - N_i^{(k)}\bigr) \vee 1$ arrivals, the exploration phase ends.
	The exploitation phase of episode~$k$ then starts with at most $7$ customers in each queue, and ends when the system is empty.
	We therefore have
	\begin{align}
		\label{eq:A_epi_SED_bnd}
		\E(\Arrepi{k}\mid \bar{r}_k = r, \calD_k^c)
		 & \leq
		14 + \E_{(7,7)}(\ArrSED{r})
		.
	\end{align}

	Substituting~\eqref{eq:A_epi_SED} and~\eqref{eq:A_epi_SED_bnd} into~\eqref{eq:bnd_EA_phases},
	\begin{align}
		\E(\Arrepi{k}\mid \bar{r}_k = r)
		 & \leq
		\max\bigl(\E_{(0,0)}(\ArrSED{r}), \ 14 + \E_{(7,7)}(\ArrSED{r})\bigr)
		.
		\label{eq:E_A_epi_split}
	\end{align}

	Recall that $T_{\mathbf{0}}$ denotes the hitting time of $\mathbf{0}$ of $\QSED{r}$.
	Let $M(t)$ denote the arrival counting process with rate $\lambda$, then $\E_{\bfx}(\ArrSED{r}) = \E_{\bfx}(M(T_{\mathbf{0}}))$ for $\bfx\in\N_{\geq 0}^2$.
	Since $M(t)-\lambda t$ is a martingale \citep[Prob.~3.4]{Karatzas1998} and $T_{\mathbf{0}} \wedge n$ is a bounded stopping time, we have by the \gls{OST} \citep[Thm.~3.22]{Karatzas1998}
	$
		\E_{\bfx}\bigl(M(T_{\mathbf{0}} \wedge n)\bigr)
		=
		\lambda \E_{\bfx}\bigl(T_{\mathbf{0}} \wedge n \bigr)
	$
	for all  $n\in\N$ and $\bfx\in\N_{\geq 0}^2$.
	Letting $n\to\infty$ and using the monotone convergence theorem gives
	\begin{align}
		\E_{\bfx}(\ArrSED{r}) & =\E_{\bfx}(M(T_{\mathbf{0}})) = \lambda\E_{\bfx}(T_{\mathbf{0}}),
		\qquad \bfx \in\N_{\geq 0}^2
		.
		\label{eq:poi_E_A_B}
	\end{align}
	Applying~\eqref{eq:poi_E_A_B} in~\eqref{eq:E_A_epi_split} gives
	\begin{align}
		\E(\Arrepi{k}\mid \bar{r}_k = r)
		 & \leq
		\max\bigl(\lambda \E_{(0,0)}(T_{\mathbf{0}}), 14 + \lambda \E_{(7,7)}(T_{\mathbf{0}})\bigr)
		.
		\label{eq:E_A_epi_split_Br}
	\end{align}
	Note that $\beta_r$ in~\eqref{eq:bnd_B_final} is non-decreasing in $\bfq$ by monotonicity of $\phi_r$  and since $a_r>0$ by assumption~\eqref{eq:stab_cond}.
	Hence, applying~\Cref{lem:E_B_bnd} concludes the proof  of~\eqref{eq:bnd_E_arr_epi_k},
	\begin{align}
		\E(\Arrepi{k}\mid \bar{r}_k = r)
		 & \leq
		\max\bigl(\lambda \beta_r(0,0), \ 14 + \lambda \beta_r(7,7)\bigr)
		\leq
		14 + \lambda \beta_r(7,7)
		.
	\end{align}

	{\bf Second moment.}
	Similar to before,
	\begin{align}
		\E\bigl((\Arrepi{k})^2 \mid \bar{r}_k = r, \calD_k\bigr)      & =\E_{(0,0)}\bigl((\ArrSED{r})^2\bigr),
		\label{eq:A2_epi_D}
		\\
		\E\bigl((\Arrepi{k})^2 \mid \bar{r}_k = r, (\calD_k)^c \bigr) & \leq \E_{(7,7)}\bigl((14+\ArrSED{r})^2\bigr)  \notag \\
		                                                              & =
		196 + 28 \E_{(7,7)}(\ArrSED{r}) + \E_{(7,7)}\bigl((\ArrSED{r})^2\bigr)
		\label{eq:A2_epi_Dc}
		.
	\end{align}
	Recall that $M(t)$ is the arrival counting process and that $\E_{\bfx}((\ArrSED{r})^2) = \E_{\bfx}(M(T_{\mathbf{0}})^2)$.
	Since $M(t)$ is a Poisson process with rate $\lambda$, $(M(t)-\lambda t)^2 - \lambda t$ is a martingale (see, e.g., Definition~1.5.3 and Example~1.5.4 in \citep{Karatzas1998}).
	By the \gls{OST} \citep[Thm.~3.22]{Karatzas1998},
	\begin{align}
		\E_{\bfx}\bigl(\bigl(M(T_{\mathbf{0}} \wedge n)-\lambda (T_{\mathbf{0}} \wedge n)\bigr)^2\bigr) =  \lambda \E_{\bfx} (T_{\mathbf{0}} \wedge n)
		,
		\qquad n\in\N, \ \bfx \in\N_{\geq 0}^2
		.
	\end{align}
	Hence, using also that $(a+b)^2 \leq 2a^2 + 2b^2$,
	\begin{align}
		\E_{\bfx}\bigl(M(T_{\mathbf{0}} \wedge n)^2\bigr)
		 & \leq
		2\E_{\bfx}\bigl(\bigl(M(T_{\mathbf{0}} \wedge n)-\lambda (T_{\mathbf{0}} \wedge n)\bigr)^2\bigr)
		+
		2 \lambda^2 \E_{\bfx}\bigl((T_{\mathbf{0}} \wedge n)^2\bigr) \\
		 & =
		2 \lambda \E_{\bfx} (T_{\mathbf{0}} \wedge n)
		+
		2 \lambda^2 \E_{\bfx}\bigl((T_{\mathbf{0}} \wedge n)^2\bigr)
		,
		\qquad n\in\N, \ \bfx \in\N_{\geq 0}^2
		.
	\end{align}
	Letting $n\to\infty$ and using the monotone convergence theorem, we find
	\begin{align}
		\E_{\bfx}((\ArrSED{r})^2)
		 & =
		\E_{\bfx}(M(T_{\mathbf{0}})^2)
		\leq
		2 \lambda \E_{\bfx} (T_{\mathbf{0}})
		+
		2 \lambda^2 \E_{\bfx}\bigl(T_{\mathbf{0}}^2\bigr)
		.
		\label{eq:poi_marting_sq}
	\end{align}
	Using~\eqref{eq:poi_E_A_B} and~\eqref{eq:poi_marting_sq} in~\eqref{eq:A2_epi_D} and~\eqref{eq:A2_epi_Dc} gives
	\begin{align}
		\E\bigl((\Arrepi{k})^2 \mid \bar{r}_k = r, \calD_k\bigr)
		 & \leq
		2 \lambda \E_{(0,0)}\bigl(T_{\mathbf{0}}\bigr) + 2 \lambda^2 \E_{(0,0)}\bigl(T_{\mathbf{0}}^2\bigr), \\
		\E\bigl((\Arrepi{k})^2 \mid \bar{r}_k = r, (\calD_k)^c \bigr)
		 & \leq
		196 + 30 \lambda \E_{(7,7)}(T_{\mathbf{0}})  + 2 \lambda^2 \E_{(7,7)}\bigl(T_{\mathbf{0}}^2\bigr)
		.
	\end{align}
	Applying \Cref{lem:E_B_bnd} and using that $\gamma_r(\bfq)$ in~\eqref{eq:bnd_B_squared_final} is non-decreasing in $q_1$ and $q_2$ concludes the proof of~\eqref{eq:bnd_E2_arr_epi_k},
	\begin{align}
		\E\bigl((\Arrepi{k})^2 \mid \bar{r}_k = r\bigr)
		 & \leq
		\max
		\bigl(
		\E\bigl((\Arrepi{k})^2 \mid \bar{r}_k = r, \calD_k\bigr), \E\bigl((\Arrepi{k})^2 \mid \bar{r}_k = r, \calD_k^c \bigr)
		\bigr)  \\
		 & \leq
		196 + 30 \lambda \beta_r(7,7)  + 2\lambda^2 \gamma_r(7,7)
		.
	\end{align}

\end{proof}

\end{document}